\documentclass{article}

\pdfoutput=1

\usepackage[utf8]{inputenc} % allow utf-8 input
\usepackage[T1]{fontenc}    % use 8-bit T1 fonts
\PassOptionsToPackage{hyphens}{url} 
\usepackage[hyperfootnotes=false]{hyperref}
\usepackage{booktabs}       % professional-quality tables
\usepackage{amsfonts}       % blackboard math symbols
\usepackage{nicefrac}       % compact symbols for 1/2, etc.
\usepackage{microtype}      % microtypography
\usepackage{xcolor}         % colors
\usepackage[hang,flushmargin]{footmisc}
\usepackage{algorithm}
\usepackage{algpseudocode}
\usepackage{graphicx}
\usepackage{enumerate}
\usepackage{caption}
\usepackage{subcaption}
\usepackage{comment}
\usepackage[export]{adjustbox}
\usepackage{mathtools}
\usepackage{amsthm}
\usepackage{authblk}
\usepackage{fullpage}
\usepackage{bbm}
\usepackage{enumitem}

\usepackage[sort&compress]{natbib}

\usepackage{tikz}
\usetikzlibrary{positioning}
\usepackage[HTML]{xcolor}
\usepackage{tcolorbox}

\usepackage{Definitions}
\usepackage{dsfont}

\definecolor{mycommentcolor}{HTML}{4f9739}
\algrenewcommand\algorithmiccomment[1]{\hfill\textcolor{mycommentcolor}{\(\triangleright\) #1}}

\newlength{\inlineheight}
\title{Learning the Cost of Reliable Inference}

\author{Dimitrios Rontogiannis}
\author{Ander Artola Velasco}
\author{Manuel~Gomez~Rodriguez}

\affil{Max Planck Institute for Software Systems \\ Kaiserslautern, Germany \\
\{drontogi, avelasco, manuel\}@mpi-sws.org}

\date{}

\begin{document}

\maketitle

\begin{abstract}
% 1. Benchmarking and routing platforms
%
Benchmarking and routing platforms increasingly act as intermediaries connecting large language model providers with end-users.
%
% 2. Single price per token, not the most competitive price per task.
%
However, providers on these platforms typically use a fixed price per token, preventing users from achieving the most competitive price for their tasks. % workloads. 
%
% 3. Our platform: competitive price for guaranteed quality level
%
In this work, we design a procurement platform where prices for each task are driven by provider competition, enabling users to secure competitive pricing for guaranteed quality levels.
%
% 4. Reserve second-price auction: truthfulness
%
To this end, the platform sequentially routes queries via a reverse second-price auction that incentivizes model providers to truthfully bid their best estimate of the average cost to serve a user's query.
%
% 5. Learning mechanism under truthfulness
%
As it routes queries, the platform learns the quality offered by each provider and progressively routes queries to the most cost-competitive provider among those meeting a desired quality threshold.
%
% 6. Experiments
%
To validate our design, we conduct experiments with multiple LLMs from the \texttt{Llama} and \texttt{Qwen} families on popular mathematical reasoning and question-answering benchmarks. 
Our results show that, depending on the task and quality threshold, the most cost-competitive price may closely track the winning provider's underlying cost or include a significant premium. When providers compete closely on cost, users may pay less than under fixed pricing; however, when a single provider holds a substantial cost advantage, users may pay more.
\end{abstract}

\section{Introduction}
\label{sec:intro}
The rapid proliferation of large language models (LLMs) has led to a massive market of online service providers offering models of varying capabilities at different prices.
To help users navigate their choice of models, benchmarking and routing platforms have become popular, letting users compare model performance on well-established tasks and even automatically select a provider to serve their queries.\footnote{See, \eg, \url{artificialanalysis.ai}, \url{openrouter.ai}, and \url{thegrid.ai}.\label{fn:platforms}}

However, existing platforms do not allow users to achieve the most competitive price for their tasks. This is because providers on these platforms typically charge users a fixed price per token,\footnote{\url{ai.google.dev/gemini-api/docs/pricing}, \url{openai.com/api/pricing/}, \url{claude.com/pricing}.} which fails to reflect that a provider's cost to serve a model and the response quality the user obtains in return depend heavily on the specific task.
As a concrete illustration, consider a user who wants to use an LLM to assist with medical diagnosis and can choose between two providers, each offering a different model. Suppose the providers charge $\$0.06$ and $\$0.07$ per token, while generating a token costs them $\$0.01$ and $\$0.04$, respectively. Then, if both models produce responses of similar quality and number of tokens for the user's task, the user---acting rationally---would select the cheapest model and pay $\$0.06$ per token. 
However, if the providers competed on price for the user's task, the user could in principle pay as little as $\$0.04$ per token: the first provider would lower its per-token price down to this point, since below this the second provider---which cannot profitably price below its own cost---would no longer be a competitive alternative~\citep{MILGROM1982280, mahmood2024pricing, bergemann2025menu}.
%
%  Contributions

In this work, we design a procurement platform where competition among providers determines the price users pay for their tasks at guaranteed quality levels.
To this end, the platform sequentially routes queries via a reverse second-price auction that incentivizes providers to truthfully bid their best estimate of the average cost to serve a user's query. 
As it routes queries, the platform learns the quality offered by each provider and progressively routes queries to the most cost-competitive provider among those meeting a desired quality threshold.
More formally, we show that, with high probability, the platform routes all but a vanishing fraction of the user's queries to the most cost-competitive qualified provider, whose (average) profit per query converges to the difference between the second-lowest average cost among qualified providers and its own average cost.

% Experiments
%
To validate our platform, we conduct experiments with multiple LLMs from the \texttt{Llama} and \texttt{Qwen} families on popular mathematical reasoning and question-answering benchmarks. 
We find that the pricing margin of the most cost-competitive qualified provider varies significantly depending on the task and quality threshold.
Moreover, we also find that, when qualified providers compete closely on cost, users may pay less on our platform than what they currently pay under fixed pricing and, when the most cost-competitive qualified provider has a substantial cost advantage over its nearest qualified competitor, they may pay more.

%
%  Further related work
%
\xhdr{Further related work}
%
% 1. Economics of generative AI
Our work contributes to a rapidly growing body of literature on the economic aspects of LLMs-as-a-service~\citep{raghavan2024competition, mahmood2024pricing, la2024language, LauferFine, Bergemann, bergemann2025menu, olmedo2026computational, velasco2025auditingpaypertokenlargelanguage, velasco2026overcharging, velasco2026ttcgames}. Within this literature, the works most closely related to ours characterize how LLM providers price their models when serving users in a competitive market. 
Specifically,~\citet{mahmood2024pricing} study the price competition between providers serving models of varying quality across different tasks, while \citet{bergemann2025menu} characterize the menu of token prices that maximizes a provider's revenue across heterogeneous users.
%1.1 Economic evaluation of models
Our work also relates to a recent line that evaluates LLMs not only by their capabilities but also by the price users pay~\citep{wang2024reasoning, zellinger2025economicevaluationllms, erol2026costofpass, chen2026price}.
%
%1.2 Routing platforms and compute as a commodity
This alternative view has very recently materialized in already-deployed platforms: some platforms 
publish provider rankings that explicitly factor in their prices, while others act as routers that lower the price by letting providers compete on their rates.\footref{fn:platforms}
In contrast to these solutions, our platform explicitly ensures that the user asymptotically secures competitive pricing for their specific task at guaranteed quality levels.

% 2. Dynamic/sequential auctions
Our work also builds on a large body of literature on mechanism and auction design~\citep{McAfee1986-cx, ledyard1987incentive, Che1993, bar2002incentive, roughgarden2017price, Akbarpour2020, AWAYA2025103083, CAI2026130083}. In this context, our sequential setting connects to a strand on dynamic auctions, where bidders sequentially learn private information, such as their type, that determines their utility, and the objective is to design mechanisms that ensure truthful bidding over multiple time steps~\citep{Hossain08, PavanDynamic2012, pmlr-v49-weed16, FengLearning, Bergemann2019, DirkData}. To the best of our knowledge, however, we are the first to consider a setting where LLM providers act as bidders who learn about their operational costs and serving qualities over time.

% 3. Multiarmed-bandits with strategic arms
Lastly, at a technical level, our work relates to a strand of the multi-armed bandit literature in which the learner must choose among alternatives that are themselves strategic agents seeking to maximize their utility, requiring learning algorithms that explicitly account for these incentives~\citep{Auer2002-xv, evendar06a, Babaioff2009-tj, Bubeck2012-mz, jain2015incentivecompatiblemultiarmed, pmlr-v99-braverman19b, Lattimore2020-gm, pmlr-v119-freeman20a, Padmanabhan2022-xv}. 
Therein, the closest works to ours are those that study procurement auctions through a bandit lens~\citep{Singla2013-mp, chen2026contextual, Patra20264}. Specifically, both~\citet{chen2026contextual} and~\citet{Patra20264} formulate procurement as a contextual bandit problem, and the latter studies the selection of LLM providers as we do. 
% where providers report their generation costs and the user learns query-dependent model quality. 
%
However, in contrast to both of them, providers in our work do not know their model's cost on the user's task a priori, and must learn it sequentially through experience. This means their underlying valuations---and consequently their bids---are non-stationary. 
In a similar LLM routing setting,~\citet{Cao26} instead design an allocation and payment mechanism that disincentivizes providers from misreporting their capabilities or manipulating the token bills charged to the user.

\section{A Procurement Platform for Reliable Inference}
\label{sec:model}
In this section, we introduce a procurement platform that acts as an intermediary between a user who seeks to solve a specific task (\eg, coding, mathematical reasoning, or question answering) and a set of $N$ LLM providers.
By incentivizing providers to compete on price, the platform ensures that, with high probability, all but a vanishing fraction of queries are routed to qualified providers, and the user's average payment converges to the most competitive price for the task.\footnote{By ``most competitive price,'' we mean the price a user would pay in a reverse second-price auction where providers and the platform have oracle knowledge of their average costs and the providers' quality, respectively.}
Our platform is best suited to workloads with a large volume of queries for a recurring task and reliable response evaluation, such as coding tasks with automated tests.
These workloads provide the repeated observations needed for the platform to learn provider qualities and for providers to learn their task-specific costs.

The platform operates sequentially over time steps $t = 1, 2, \dots, T$, where the total number of queries $T$ is unknown to the providers but assumed to be at least the number of providers $N$. 
At each time step $t$, before a new query is submitted by the user, all providers simultaneously submit their bids $b_{i,t} \in [0, c_{\max}]$, where $c_{\max} > 0$ is an upper bound on the cost of serving a single query.
The platform then receives the user's query $x_t$, an independent draw from a fixed distribution $P^x$ over the space $\Xcal$ of possible queries for the user's task.
Given the current bid vector $\mathbf{b}_t = (b_{1,t}, \dots, b_{N,t})$, the platform selects the provider $I_t \in [N]$ that serves the query. 
Provider $I_t$ generates a response $y_t \in \Ycal$ to the query using a given model, where $\Ycal$ is the space of possible responses,\footnote{In practice, the provider generates $y_t$ by sampling a sequence of tokens from its model, and so $\Ycal$ is the set of finite-length token sequences.} and, by doing so, it incurs a generation cost $c_{I_t,t} \in [0, c_{\max}]$, which includes the cost of energy, compute, or hardware.
The platform then grades the response using a score $r(x_t, y_t) \in [0, 1]$, ranging from unsatisfactory $(0)$ to fully satisfactory $(1)$, and sets the payment the user makes to provider $I_t$. 
In the remainder, we denote the joint distribution of the provider $i$'s cost--score pairs as $\theta_i$, and refer to provider $i$'s mean score $q_i \in [0,1]$ as its (average) quality.

Next, we describe the rules used by the platform to allocate queries and pay the providers. Our key idea is to combine an allocation rule based on confidence bounds---which progressively narrows the search to the most cost-competitive provider meeting a desired quality threshold---with a second-price payment rule that incentivizes providers to bid as low as possible without risking a payment below their cost estimate. Specifically, the platform begins with an initialization phase spanning the first $N$ time steps, during which it assigns one query to each provider in turn and pays them $c_{\max}$. This ensures every provider serves at least one query, rendering the estimates below well-defined from round $N + 1$ onward.
Thereafter, at each round $t > N$, the platform maintains a running estimate of each provider's
quality,
\begin{equation}\label{eq:quality-estimate}
    \hat{q}_{i,t} = \frac{1}{m_{i,t}} \sum_{\tau=1}^{t-1} r(x_\tau, y_\tau)\, \mathbbm{1}\{I_\tau = i\},
\end{equation}
where $m_{i,t} = \sum_{\tau=1}^{t-1} \mathbbm{1}\{I_\tau = i\}$ counts how many times provider $i$ has been selected before round $t$. The goal of these estimates is to identify the set of qualified providers $\Qcal = \{i \in [N]: q_i \geq q_{\min}\}$, meeting a user-defined quality threshold $q_{\min}$, which we assume contains at least two providers\footnote{This assumption is needed only for the analysis. If a single provider is estimated to be qualified, the minimum in Eq.~\ref{eq:critical-payment} is over an empty set and the provider is paid $c_{\max}$.}. Since the true qualities $q_i$ are unknown, the platform instead forms an optimistic estimate of $\Qcal$,
\begin{equation}\label{eq:qualified-set}
    \hat{\Qcal}_t = \bigl\{ i \in [N] : \hat{q}_{i,t} + \beta(m_{i,t}) \geq q_{\min} \bigr\},
    \quad \text{where} \quad
    \beta(m_{i,t}) = \sqrt{\frac{1}{2 m_{i,t}} \ln\!\left(\frac{2\pi^2 N m_{i,t}^2}{3\delta}\right)}.
\end{equation}
Here, $\beta(m_{i,t})$ is a confidence radius at a level $\delta \in (0,1)$ fixed by the platform that accounts for the error in the empirical estimate $\hat{q}_{i,t}$, so that $\hat{\Qcal}_t$ contains every provider whose upper confidence bound on quality exceeds the threshold $q_{\min}$. In Appendices~\ref{app:good-event} and~\ref{app:eligibility-exploration}, we show that, with probability at least $1 - \delta$, this optimistic estimate never excludes a qualified provider, \ie, $\Qcal \subseteq \hat{\Qcal}_t$ holds simultaneously across all time steps $t > N$.

Using the set $\hat{\Qcal}_t$, the platform allocates the query by prioritizing providers with lower bids while giving an explicit advantage to those that have not yet been selected often.
More concretely, let $g(t)$ be an increasing function specified by the platform that grows sublinearly in $t$, and $\Scal_t = \{\, i \in \hat{\Qcal}_t: m_{i,t} < g(t) \,\}$ denote the set of estimated-qualified providers selected fewer than $g(t)$ times so far. Then, at each round $t > N$, the platform selects:\footnote{If $\hat{\Qcal}_t = \emptyset$, then $I_t$ is chosen uniformly at random. However, this occurs with probability at most $\delta$ (see Appendices~\ref{app:good-event} and~\ref{app:eligibility-exploration}).}
\begin{equation}\label{eq:allocation}
    I_t =
    \begin{dcases}
        \text{Uniform}\left( \Scal_t \right)
        & \text{if } \Scal_t \neq \emptyset, \\[10pt]
        \argmin_{i \in \hat{\Qcal}_t}\, \bigl( b_{i,t} - c_{\max}\, \beta(m_{i,t}) \bigr)
        & \text{otherwise,}
    \end{dcases}
\end{equation}
with ties broken at random independently of the current query $x_t$. That is, if $\Scal_t \neq \emptyset$, which we refer to as an exploration step, the platform selects at random an estimated-qualified provider that has not been sampled sufficiently according to the function $g(t)$. 
Otherwise, the rule selects the provider minimizing an optimistic bidding index, where the term $c_{\max}\,
\beta(m_{i,t})$ gives an additional advantage to providers whose confidence radius remains large. 
Altogether, this allocation rule ensures that qualified providers are selected often enough for the platform to reliably estimate their quality and, as we will
see in Section~\ref{sec:provider}, for the providers to learn their own costs.

To design the payment rule, we build on the classical second-price rule from the auction and mechanism design literature~\citep{Vickrey, Nisan2011-op}. Specifically, for any provider $i$, we define the critical payment $P_{i, t}$ as  
\begin{equation}\label{eq:critical-payment}
     P_{i,t} = c_{\max}\, \beta(m_{i,t}) + \min_{j \in \hat{\Qcal}_t \setminus \{i\}}
    \bigl( b_{j,t} - c_{\max}\, \beta(m_{j,t}) \bigr),
\end{equation}
where the minimum over an empty set is taken to be $+\infty$ by convention. The actual payment made to the selected provider $I_t$ is then set to
\begin{equation}\label{eq:payment}
    \pi_{t} =
    \begin{dcases}
        c_{\max}& \text{if } \Scal_t \neq \emptyset, \\
         \min \left\{P_{I_t, t}, c_{\max} \right\}& \text{otherwise}.
    \end{dcases}
\end{equation}
In words, in an exploration step, the provider selected uniformly at random from the set $\Scal_t$ is paid $c_{\max}$. Otherwise, the selected provider $I_t$ receives a payment that depends on how often the providers have been selected in the past and on the other providers' bids $\mathbf{b}_{-I_t,t} = (b_{j,t})_{j \neq I_t}$, but not on its own bid $b_{I_t, t}$. 
Overall, the complete operation of our procurement platform is summarized in Algorithm~\ref{alg:platform}.\footnote{In Appendix~\ref{app:multidimensional-quality}, we extend (the theoretical analysis of) our platform to multiple quality scores, each with its own mean threshold.
In Appendix~\ref{app:soft-objectives}, we extend (the theoretical analysis of) our platform to allocation and payment rules accounting for tradeoffs between cost and other attributes, such as latency.}

\setlength{\textfloatsep}{10pt}\label{app:algorithm}
\begin{algorithm}[!!h]
\caption{It progressively routes queries to the most cost-competitive qualified provider.}
\label{alg:platform}
\begin{algorithmic}[1]

\State \textbf{Input:} providers $[N]$, quality threshold $q_{\min}$, confidence level $\delta$,
allocation parameters $k, \alpha$, user-defined score $r$, cost bound $c_{\max}$
\vspace{2pt}

\State \textbf{Initialization:} assign one query to each provider $i \in [N]$ and pay them $c_{\max}$
% and receive their first standing bid $b_i$
\State Set $m_i \gets 1$ and $\hat{q}_i \gets r(x_i, y_i)$ for all $i \in [N]$
\vspace{2pt}

\For{$t = N+1, N+2, \dots$}
    \State All providers simultaneously submit their bids
    \State Receive user query $x_t \sim P^{x}$
    \State $\beta_i \gets \sqrt{\tfrac{1}{2 m_i}\ln\!\big(\tfrac{2\pi^2 N m_i^2}{3\delta}\big)}$
    for all $i \in [N]$ \Comment{confidence radius}
    \State $\hat{\Qcal} \gets \{\, i \in [N] : \hat{q}_i + \beta_i \geq q_{\min} \,\}$
    \Comment{optimistic qualified set}
    \State $g \gets k\,(t/N)^{\alpha}$
    \Comment{target selection count}
    \If{$\hat{\Qcal} = \emptyset$}
        \State $I_t \gets$ a provider chosen uniformly at random from $[N]$;\quad
        $\pi_t \gets c_{\max}$
    \ElsIf{$m_i < g$ for some $i \in \hat{\Qcal}$}
        \State $I_t \gets$ a provider chosen uniformly at random from $\{\, i \in \hat{\Qcal}
        : m_i < g \,\}$;\quad $\pi_{t} \gets c_{\max}$
        
    \Else
        \State $I_t \gets \argmin_{i \in \hat{\Qcal}}\, \big( b_i - c_{\max}\, \beta_i \big)$
        \Comment{allocation rule, ties broken at random independently of $x_t$}
        \State $P \gets c_{\max}\, \beta_{I_t} + \min_{j \in \hat{\Qcal} \setminus \{I_t\}}
        \big( b_j - c_{\max}\, \beta_j \big)$
        \State $\pi_{t} \gets \min\{ P, c_{\max} \}$
        \Comment{second-price payment}
    \EndIf
    \State The user pays provider $I_t$ the amount $\pi_{t}$
    \State Provider $I_t$ serves the query, returning a response $y_t$
    \State Observe the score $r(x_t, y_t)$ and set
    $\hat{q}_{I_t} \gets \tfrac{m_{I_t}\, \hat{q}_{I_t} + r(x_t, y_t)}{m_{I_t} + 1}$ and
    $m_{I_t} \gets m_{I_t} + 1$
    % \State Provider $I_t$ may revise its standing bid $b_{I_t}$; all other bids are unchanged
\EndFor

\end{algorithmic}
\end{algorithm}

\section{Pricing Inference Through Competition}
\label{sec:provider}
In this section, we study the incentives that providers face when participating in our procurement platform and characterize their optimal bidding strategy. 
As a first step, we analyze the conditions under which a provider is selected to serve a query as a function of their bid. 
Concretely, the following proposition shows that, when $\mathcal{S}_t = \emptyset$, a provider's assignment is determined by whether their bid falls below the critical payment $P_{i,t}$.\footnote{All proofs are deferred to Appendix~\ref{app:proofs}.}
\begin{proposition}\label{prop:selection-threshold}
    Fix a step $t > N$ with $\mathcal{S}_t = \emptyset$, a provider $i \in \hat{\mathcal{Q}}_t$, and the other providers’ bids $\mathbf{b}_{-i, t}$. Then, provider $i$ is assigned the query $x_t$ if they bid $b_{i,t} < P_{i,t}$, and is not assigned the query if $b_{i,t} > P_{i,t}$.
\end{proposition}

The above proposition does not fully characterize how a provider \emph{should} bid; however, it already offers valuable intuition. When $\mathcal{S}_t \neq \emptyset$, the query is assigned independently of the bids $\mathbf{b}_t$, meaning a provider's bid only influences the allocation when $\Scal_t = \emptyset$.
In this case, a provider is assigned the query only if it bids below the threshold $P_{i,t}$. 
Because this threshold depends on the other providers' bids, it is unknown to provider $i$ in advance, which incentivizes it to bid as low as possible.
Conversely, once a provider bids low enough to secure the query, due to the payment rule in Eq.~\ref{eq:payment}, their final payout remains independent of their own bid.
Together, these properties suggest that providers are incentivized to bid truthfully according to the generation costs they incur. In the remainder of this section, we formalize this intuition and characterize the resulting equilibrium.

We cast the interaction among the $N$ providers on the platform as a multi-stage game~\citep{fudenberg1991game}. 
In this game, each provider $i$ is characterized by its joint cost--score distribution $\theta_i$, which need not be known even to the provider itself; we write $\thetab = (\theta_1, \dots, \theta_N)$.\footnote{In the game theory literature, this corresponds to an incomplete-information setting in which $\theta_i$ is the (private) \emph{type} of each provider~\citep{Harsanyi, Hossain08, FengLearning}.} Moreover, its action in the game corresponds to choosing a bidding strategy $\sigma_{i} = (\sigma_{i,t})_{t\geq1}$. At each time step $t$, $\sigma_{i, t}$ determines its bid $b_{i,t}$ based on its observed history $H_{i,t}$, which collects the costs $c_{i,\tau}$ and scores $r(x_{\tau}, y_{\tau})$ of every query that provider $i$ has served before $t$, 
\begin{equation}\label{eq:provider-history}
    b_{i,t} = \sigma_{i, t}(H_{i,t}) \quad \text{where} \quad H_{i,t} = \bigl(c_{i,\tau}, r(x_{\tau}, y_{\tau}) \bigr)_{\tau < t \,:\, I_\tau = i}.
\end{equation}
In what follows, we denote the providers' overall history as $H_t$ and say that provider $i$'s history $H_{i,t}$ is compatible with $H_t$ if it is contained in $H_t$.

Crucially, although a provider does not necessarily know its true mean cost $\bar{c}_i$ of serving a user query, it can use its observed history to compute an estimate of the conditional mean cost of serving the next query, $\mathbb{E}\bigl[c_{i,t} \mid H_{i,t}\bigr]$.
For ease of exposition, we focus on providers that compute this estimate as the empirical average of past costs, given by 
\begin{equation*}
\hat{c}_{i,t} = \frac{1}{m_{i,t}} \sum_{\tau = 1}^{t-1} c_{i,\tau}\,\mathbbm{1}\{ I_{\tau} = i\}.    
\end{equation*}
In Appendix~\ref{app:bayesian}, we show that Theorem~\ref{thm:robust-dominant} holds for any cost estimator and, if the platform widens the cost radius $c_{\max}\beta(\cdot)$ to $(1 + \gamma)c_{\max}\beta(\cdot)$, 
Theorems~\ref{thm:sample-complexity} and~\ref{thm:regret} hold for providers that compute $\hat{c}_{i,t}$ through Bayesian estimators whose estimates deviate from the empirical mean by at most a constant multiple $\gamma$ of the confidence radius.

At each time step $t > N$, provider $i$ estimates its profit from serving a query by subtracting its current cost estimate $\hat{c}_{i,t}$ from the payment $\pi_t$ it would receive if selected. Then, given any history $H_t$ compatible with its observed history $H_{i,t}$, the provider's estimated utility is the expected sum of these estimated profits over the queries it serves from $t$ to $T$, \ie,
\begin{equation}\label{eq:utilities}
    U_{i \mid H_t}(\sigma_i, \boldsymbol{\sigma}_{-i}; \thetab_{-i}, T)
    = \mathbb{E}\left[  \sum_{s=t}^{T} \left( \pi_{s}
      - \hat{c}_{i,s} \right) \cdot \mathbbm{1}\{I_s = i\} \given H_{t} \right],
\end{equation}
where $\boldsymbol{\sigma}_{-i} = (\sigma_j)_{j \neq i}$ denotes the strategies of all providers other than $i$, and the expectation is over the randomness in query draws, the providers' generations, and the platform's selections, conditional on the history $H_{t}$, which depends implicitly on $\thetab$.

Having formalized the providers' interaction as a game, we now turn to characterizing their equilibrium bidding strategies. Specifically, since a provider $i$ knows neither the joint cost--score distributions $\thetab_{-i}$ characterizing the other providers nor the total number of user queries $T$, we focus on strategies that maximize the provider's estimated utility under this uncertainty using the framework of robust game theory~\citep{AghassiRobust}.
Formally, the next theorem shows that the strategy $\sigma_i^*$ that bids the provider's cost estimate for the next query, \ie,\footnote{Under this strategy, if provider $i$ does not serve the query at time $t$, it does not observe a new cost and therefore submits the same bid at $t + 1$.}
\begin{equation}\label{eq:truthful-strategy}
    \sigma_{i, t}^*(H_{i,t}) = \hat{c}_{i,t},
\end{equation}
maximizes each provider's worst-case estimated utility over its uncertainty about $\thetab_{-i}$ and $T$ for any given history $H_t$ that is compatible with its observed history $H_{i,t}$.
\begin{theorem}\label{thm:robust-dominant}
    For any provider $i \in [N]$, 
    any strategy profile $\boldsymbol{\sigma}_{-i}$ of all other providers, 
    and any history $H_t$ with $t > N$ that is compatible with provider $i$'s observed history $H_{i, t}$, 
    the strategy $\sigma_i^*$ satisfies:
    \begin{enumerate}[label=\roman*), nosep, leftmargin=7mm]
        \item for every $T \geq t$ and $\thetab_{-i}$, $U_{i \mid H_t}(\sigma^*_i, \boldsymbol{\sigma}_{-i}; \thetab_{-i}, T) \geq 0$;
        \item for every $\sigma_i$, $\inf_{T \geq t,\, \thetab_{-i}} U_{i \mid H_t}(\sigma^*_i, \boldsymbol{\sigma}_{-i}; \thetab_{-i}, T) \geq \inf_{T \geq t,\, \thetab_{-i}} U_{i \mid H_t}(\sigma_i, \boldsymbol{\sigma}_{-i}; \thetab_{-i}, T)$;
        \item for every $\sigma_i$, if $\Scal_t = \emptyset$, $i \in \hat{\Qcal}_t$ and $(\sigma^*_{i, t}(H_{i,t}) - P_{i,t}) (\sigma_{i,t}(H_{i,t}) - P_{i,t}) < 0$, the inequality in ii) is strict.
    \end{enumerate}
\end{theorem}
In the above theorem, claim $i)$ holds because a provider bidding its cost estimate serves a query only when the payment covers that estimate. 
Claim $ii)$ holds because providers do not know $T$ and thus cannot rule out that the current query is the last one; consequently, under any strategy, a provider's worst-case estimated utility is bounded above by its profit from the current query.
Moreover, under $\sigma_i^*$, this upper bound is attained since, by claim i), every later query yields nonnegative estimated profit. 
Since this worst-case scenario effectively isolates the current query, and because the payment does not depend on the provider's own bid, the optimal bid is the lowest possible value under which the provider earns nonnegative estimated profit whenever it is selected---which is precisely the provider's cost estimate. 
Claim $iii)$ holds because, if the current query happens to be the last one, any deviation changing whether a provider is assigned the query either causes the provider to forgo a query with positive estimated profit or forces it to serve one with negative estimated profit.

Furthermore, as an immediate consequence of Theorem~\ref{thm:robust-dominant}, the choice of bidding strategies $\boldsymbol{\sigma}^* = (\sigma^*_1, \dots, \sigma^*_N)$ is a robust equilibrium, \ie, for every $\sigma_i$
\begin{equation}
    \inf_{T \geq t,\, \thetab_{-i}} U_{i \mid H_t}(\sigma^*_i, \boldsymbol{\sigma}^*_{-i}; \thetab_{-i}, T)
    \;\geq\;
    \inf_{T \geq t,\, \thetab_{-i}} U_{i \mid H_t}(\sigma_i, \boldsymbol{\sigma}^*_{-i}; \thetab_{-i}, T).
\end{equation}
In words, at every step $t > N$ and after every history, no provider $i$ can improve its worst-case estimated utility by deviating from $\sigma^*_i$ when the other providers follow $\boldsymbol{\sigma}^*_{-i}$.

Remarkably, in Appendix~\ref{app:dynamic-deviations}, we show that even when providers know the horizon $T$, any unilateral deviation from the cost-estimate bidding profile $\boldsymbol{\sigma}^*$ yields at most a vanishing gain per query with high probability, that is, $\boldsymbol{\sigma}^*$ forms an approximate Nash equilibrium. 
Therein, we also show that the same conclusion holds when providers bid to maximize their profit computed using their realized costs $c_{i,t}$, instead of the estimates $\hat{c}_{i,t}$.
%
% In this context, it is also worth noting that, for fixed $\delta$, the per-query equilibrium error bounds retain a term of order $\delta$. However, when the platform also knows $T$ and sets $\delta=1/T$, the bounds vanish.
%
Further, in Appendix~\ref{app:robustness}, we show that, when providers deviate above or below $\sigma^{*}$ by bounded amounts, the theoretical guarantees offered by our platform degrade gracefully.

\section{Finding the Lowest Competitive Inference Price}
% for a Task}
\label{sec:user}
In this section, we show that, when all providers choose their equilibrium strategy $\sigma_i^*$, with high probability, the platform routes all but a vanishing fraction of queries to the most cost-competitive qualified provider,
%the user is served by a provider meeting their quality threshold $q_{\min}$ on all but a vanishing fraction of the queries, 
while the average price they pay converges to the most competitive price for their task.
For concreteness, in the remainder of the section, we specify the function $g(t)$ to be of the form $g(t) = k(t/N)^{\alpha}$, where $\alpha \in (0, 1)$ and $k > 0$ are given parameters.

We begin by identifying the provider that serves most of the user's queries. More precisely, we show that, on all but a vanishing fraction of queries, the platform selects the qualified provider $i^*$ with the lowest cost, \ie,
\begin{equation}\label{eq:optimal-provider}
    i^* = \argmin_{i \in \Qcal} \bar{c}_i,
\end{equation}
which we refer to as the optimal provider and assume to be unique.
The number of queries not served by the optimal provider depends on the specific parameters $\alpha$ and $k$ chosen by the platform as well as on the difficulty of distinguishing each provider from $i^*$, which we characterize using two quantities for each provider.
First, for an unqualified provider $i \notin \Qcal$, we let $\varepsilon_i = q_{\min} - q_i > 0$ measure how far its quality falls below the quality threshold. Second, for a qualified provider $i \in \Qcal \setminus \{i^*\}$, we let $\Delta_i = \bar{c}_i - \bar{c}_{i^*} > 0$ measure how far its mean cost exceeds that of $i^*$. 
Intuitively, the smaller $\varepsilon_i$ or $\Delta_i$, the harder the corresponding provider is to distinguish from $i^*$, and the more queries the platform may allocate to them, as formalized in the next theorem.
\begin{theorem}\label{thm:sample-complexity}
    When all providers choose their equilibrium strategy $\boldsymbol{\sigma}^*$, with probability at least $1 - \delta$, the number of time steps in which the platform does not assign the query to the lowest-cost qualified provider, \ie, $I_t \neq i^*$, is at most $N - 1 + B_\mathrm{id}(T)$, where
    \begin{equation}
        B_\mathrm{id}(T) \coloneqq \sum_{i \notin \Qcal} M(\varepsilon_i)
        + \sum_{i \in \Qcal \setminus \{i^*\}} \max\left\{ M\!\left(\frac{\Delta_i}{c_{\max}}\right),\;
        k\cdot\left(\frac{T}{N}\right)^{\!\alpha}\right\},
    \end{equation}
    where, for any $x > 0$,
    \begin{equation*}
        M(x) = \left\lceil \frac{4}{x^2}\, \ln\!\left(\frac{57N}{\delta x^4}\right) \right\rceil.
    \end{equation*}
\end{theorem}
In the above theorem, the bound consists of the $N - 1$ initialization steps not served by $i^*$ alongside a sum of per-provider terms. 
For an unqualified provider, the term grows as $1/\varepsilon_i^2$ up to logarithmic factors. 
Conversely, for a qualified provider other than $i^*$, the term is determined by the maximum of two terms: 
a term growing as $1/\Delta_i^2$ that bounds the steps in which it outbids $i^*$ and a term $k(T/N)^{\alpha}$ that bounds the exploration steps in which it is selected.
Crucially, this bound depends sublinearly on the total number of queries $T$. 
Hence, the platform allocates the query to the optimal provider on all but a vanishing fraction of the user's queries.

We now turn to the price the user pays on the platform for a query, \ie, the payment $\pi_t$ made to the provider serving it. More concretely, we focus on comparing this payment with the second-lowest mean cost among qualified providers, $\bar{c}_{(2)} = \min_{i \in \Qcal \setminus \{i^*\}} \bar{c}_i$.
This is the lowest price the optimal provider $i^*$ would need to offer to remain cheaper than every other qualified provider, since the next-cheapest cannot profitably offer a price below its own mean cost $\bar{c}_{(2)}$.
The next theorem shows that the average payment converges to the second-lowest mean cost among qualified providers.
\begin{theorem}\label{thm:regret}
    When all providers choose their equilibrium strategy $\boldsymbol{\sigma}^*$ and the user submits a minimum number $T_0$ of queries, with probability at least $1 - \delta$ the payments made by the user satisfy
    \begin{align}
        \frac{1}{T} \sum_{t=1}^{T} \bigl| \pi_t - \bar{c}_{(2)} \bigr|
        \;\leq\;& \frac{T_0}{T}\, c_{\max}
        + k \left(\frac{N}{T}\right)^{\!1 - \alpha} \bigl( c_{\max} - \bar{c}_{(2)} \bigr) \nonumber\\
        &+ \frac{6\, c_{\max}}{2 - \alpha} \sqrt{ \frac{1}{2k} \left(\frac{N}{T}\right)^{\!\alpha}
        \ln\!\left( \frac{2\pi^2 N k^2}{3\delta} \left(\frac{T}{N}\right)^{\!2\alpha} \right) },
        \label{eq:average-absolute-payment-user}
    \end{align}
where $T_0 = N \max\bigl\{1,\, (\bar{M}/k)^{1/\alpha}\bigr\}$ and $\bar{M} = \max\bigl\{1, \max_{i \notin \Qcal} M(\varepsilon_i)\bigr\}$. Hence, $\frac{1}{T} \sum_{t=1}^{T} \pi_t  \to \bar c_{(2)}$ as $T \to \infty$.
\end{theorem}
In the above theorem, the first term accounts for the first $T_0$ time steps, which include the $N$ initialization steps and the non-exploration steps during which an unqualified provider may still belong to $\hat{\Qcal}_t$; the second term accounts for the exploration steps after $T_0$, in which the selected provider is paid $c_{\max}$; and the third term accounts for the remaining steps. As $T$ grows, the average difference between the payment and the second-lowest mean cost $\bar{c}_{(2)}$ of a qualified provider vanishes at a rate $\max\{ T^{-(1-\alpha)},\, T^{-\alpha/2} \}$, up to logarithmic factors. 
In Corollary~\ref{cor:optimal-alpha} in Appendix~\ref{app:regret}, we show that this rate is maximized when the platform sets $\alpha = 2/3$, giving a rate of $T^{-1/3}$.

In Appendix~\ref{app:user-guarantees}, we derive additional theoretical results that complement our main analysis above. In particular, we show that, with high probability, the total number of queries assigned to unqualified providers is bounded independently of $T$ and, eventually, every non-exploration step selects the most cost-competitive qualified provider $i^*$. Further, we bound the cumulative amounts by which selected providers' mean qualities fall below the user's threshold, their mean generation costs exceed that of $i^*$, and the user's payments exceed $\bar c_{(2)}$.

\section{Experiments}
\label{sec:experiments}
In this section, we evaluate our platform on simulated providers serving real LLMs from the \texttt{Llama} and \texttt{Qwen} families.
Specifically, we evaluate whether the platform learns the set of qualified providers and progressively routes queries to the optimal provider $i^*$. 
We then compare the price paid on our platform against the lowest mean query price among qualified providers, computed from the per-token prices listed by inference providers on Hugging Face\footnote{\url{huggingface.co/docs/inference-providers/index}, consulted on December 30, 2025.}.
Appendix~\ref{app:experimental-details} and Appendices~\ref{app:gsm8k}--\ref{app:margin-sensitivity} contain further details on our experimental setup and additional experimental results, respectively.

\begin{table}[t]
    \centering
    \footnotesize
    \setlength{\tabcolsep}{4pt}
    \begin{tabular}{lccccccc@{\hspace{12pt}}ccc}
        \toprule
        Benchmark & $N$ & $q_{\min}$ & $|\Qcal|$ & $i^*$ & $\bar{c}_{i^*}$ & $\bar{v}_{i^*}$ & $\bar{c}_{(2)}$ & Payment $\pi_t$ & Margin & Share of $i^*$ \\
        \midrule
        \texttt{GSM8K} &  9 & 0.764 & 3 & \texttt{Qwen2.5-3B}   & 16.1 & 20.1 & 34.4 & 30.3 (27.5) & $71\%$ & 0.96  \\
        \texttt{GPQA}  &  4 & 0.145 & 2 & \texttt{Llama-3.1-8B} & 43.0 & 53.8 & 50.7 & 48.3 (47.4) & $10\%$ & 0.97 \\
        \bottomrule
    \end{tabular}
    \caption{{\bf Our procurement platform on queries from \texttt{GSM8K} and \texttt{GPQA}.} 
    The table reports the number of simulated providers $N$, 
    the quality threshold $q_{\min}$, 
    the number of qualified providers $|\Qcal|$, 
    the most cost-competitive qualified provider $i^{*}$ and its mean cost $\bar{c}_{i^{*}}$ and price $\bar{v}_{i^{*}}$, 
    the mean cost of the second most cost-competitive qualified provider $\bar{c}_{(2)}$, 
    the average payment $\pi_t$ over all steps and, in parentheses, over non-exploration steps, 
    the pricing margin on non-exploration steps relative to $\bar c_{i^*}$, 
    and the share of queries routed to $i^{*}$.  
    The last three quantities are computed over the last $K$ queries, averaged across $30$ independent runs.
    % , where $K$ is the size of the benchmark.
    %
    Monetary quantities are in units of $10^{-6}$ USD.}
    \label{tab:main-environments}
    \vspace{-4mm}
\end{table}

\xhdr{Experimental setup}
% Provider
Each of the $N$ simulated providers serves one LLM, which we use to identify them. We consider two settings: (i) $N=9$ providers, each serving a model from the \texttt{Llama} family (\texttt{Llama-\{3-8B, 3.1-8B, 3.2-1B, 3.2-3B\}-Instruct}) or the \texttt{Qwen} family (\texttt{Qwen-\{2-0.5B, 2-1.5B, 2-7B, 2.5-3B, 2.5-7B\}-Instruct}), and (ii) $N=4$ providers serving the newest of the largest models and the smallest model of each family, \ie, \texttt{Llama-3.1-8B}, \texttt{Llama-3.2-1B}, \texttt{Qwen2.5-7B} and \texttt{Qwen2-0.5B}.

% Queries
In both settings, the user's queries are drawn from one of two widely used benchmarks. 
In the setting with $N=9$ providers, we draw queries from the \texttt{GSM8K} dataset~\citep{cobbe2021trainingverifierssolvemath}, which consists of $K = 1{,}319$ grade-school mathematical reasoning problems. 
In the setting with $N=4$ providers, we draw queries from the \texttt{GPQA}~\citep{rein2023gpqagraduatelevelgoogleproofqa} dataset, which consists of $K = 546$ graduate-level science questions. 
Both benchmarks feature verifiable ground-truth answers, and the score $r(x_t, y_t)$ represents whether the response is correct (1) or incorrect (0). Consequently, the quality $q_i$ of a provider corresponds to the probability that the model generates a correct answer. 
We rely on a public dataset of recorded generations,\footnote{\url{https://huggingface.co/datasets/Human-Centric-Machine-Learning/strategic-ttc-data}} which contains $128$ generations per model for every question across both benchmarks.
In each experiment, we simulate $T = 70{,}000$ queries by taking $53.07$ and $128.21$ passes over all queries in \texttt{GSM8K} and \texttt{GPQA}, respectively, sampling without replacement.
At each step, once a provider is selected, its response is sampled uniformly at random from its $128$ recorded generations for that query.
We repeat each experiment across $30$ independent runs to report averaged results.

% Qualities, costs, and prices
Since the public dataset we use lacks generation costs, we construct them using real-world market signals.
In particular, the cost of each recorded generation is determined by its token count and the model's average per-token price on the Hugging Face inference-provider list discounted by a fixed $25\%$ margin (\ie, divided by $1.25$), common to all providers and unknown to the platform.
We define $\bar{v}_i = 1.25\,\bar{c}_i$ as the mean query price of provider $i$, representing the average amount a user pays per query under publicly listed prices. 
Further, we place the quality threshold $q_{\min}$ at the midpoint between two adjacent provider qualities, so that three providers qualify on \texttt{GSM8K} and two on \texttt{GPQA}.
In Appendix~\ref{app:margin-sensitivity}, we run additional experiments with varying margins across providers showing qualitatively similar results.

We purposely choose these two settings to vary the intensity of competition among qualified providers, which fundamentally determines the user's final payment. 
Specifically, since the average payment converges with high probability to the second-lowest mean cost $\bar{c}_{(2)}$ among qualified providers by Theorem~\ref{thm:regret}, we call competition strong when $\bar{c}_{(2)}$ is smaller than the listed query price of the optimal provider $\bar{v}_{i^*}$, and weak otherwise.\footnote{Under the common margin, $i^{*}$ also offers the lowest listed mean price per query among qualified providers.} 
On \texttt{GSM8K}, the qualified competitors of $i^*$ are considerably more expensive ($\bar{c}_{(2)} > \bar{v}_{i^*}$, weak), whereas on \texttt{GPQA}, the sole competitor is only slightly more expensive ($\bar{c}_{(2)} < \bar{v}_{i^*}$, strong). 
Table~\ref{tab:main-environments} summarizes both settings.

%Platform and baselines
The platform runs Algorithm~\ref{alg:platform} with confidence level $\delta = 0.05$ and exploration function $g(t) = k(t/N)^{\alpha}$, with $k = 2$ and $\alpha = 3/4$.
We choose this value over the rate-optimal $\alpha = 2/3$ from Corollary~\ref{cor:optimal-alpha} to induce greater early exploration at the expense of a slower asymptotic rate. 
Under this setup, the simulated providers follow their equilibrium strategy $\boldsymbol{\sigma}^*$.
In Appendix~\ref{app:robustness-exp}, we run additional experiments showing that, when providers deviate above or below $\boldsymbol{\sigma}^*$ by bounded amounts, we obtain qualitatively similar results.

\begin{figure}[t]
    \centering
    \includegraphics[width=\linewidth]{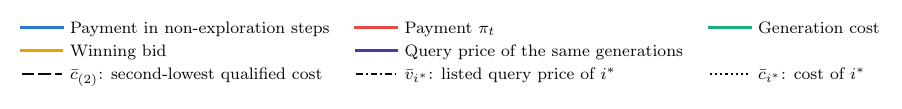}\\[-3pt]
    \begin{subfigure}[t]{0.48\textwidth}
        \centering
        \includegraphics[width=\linewidth]{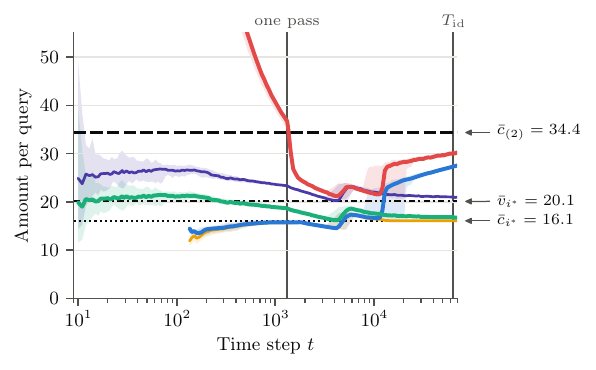}
        \caption{\texttt{GSM8K}, weak competition}
    \end{subfigure}\hfill
    \begin{subfigure}[t]{0.48\textwidth}
        \centering
        \includegraphics[width=\linewidth]{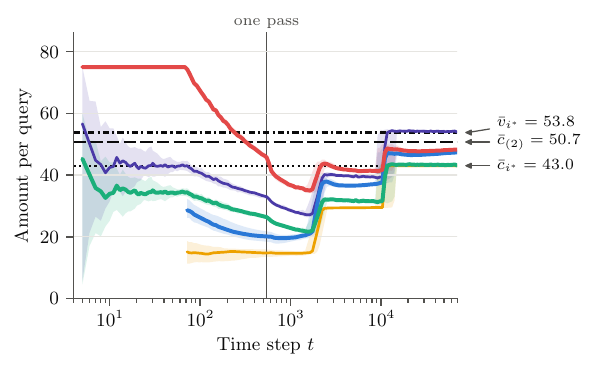}
        \caption{\texttt{GPQA}, strong competition}
    \end{subfigure}
    \caption{{\bf Temporal evolution of the winning bids, payments, and costs.} 
    For each reported quantity, every point corresponds to the median over $30$ independent runs of a moving average whose window size equals the benchmark size $K$, and the shaded regions span the $10$th to $90$th percentiles across runs. 
    The vertical lines mark one complete pass over the benchmark queries and, where applicable within the time horizon, the time step $T_{\mathrm{id}}$ beyond which Theorem~\ref{thm:sample-complexity} guarantees that $i^*$ becomes the unique most selected provider. Amount per query is in units of $10^{-6}$ USD.}
    \label{fig:main-payments}
    \vspace{-3mm}
\end{figure}

\xhdr{Results}
Table~\ref{tab:main-environments} summarizes our main results. 
We find that the platform successfully identifies the optimal provider: 
over the final pass of the benchmark, it routes at least $96\%$ of all queries to $i^*$, and $i^{*}$ becomes the most selected provider after roughly $10{,}000$ and $30{,}000$ queries on \texttt{GSM8K} and \texttt{GPQA} respectively. 
In both cases, this occurs well before the time step $T_{\mathrm{id}} = \min \{T \,:\, T > N + 2B_\mathrm{id}(T)\}$, \ie, the time step beyond which Theorem~\ref{thm:sample-complexity} guarantees, with probability at least $1-\delta$, that $i^*$ is the unique most-selected provider ($63{,}740$ and $153{,}697$ queries, respectively; see Appendices~\ref{app:gsm8k-env} and~\ref{app:strong-env}). 
This gap demonstrates that our sample complexity bounds are highly conservative in practice.
The delayed convergence on \texttt{GPQA} stems from the presence of the unqualified \texttt{Llama-3.2-1B} model; because its quality falls a mere $0.037$ below $q_{\min}$, the platform requires approximately $11{,}000$ queries to statistically exclude it from $\hat{\Qcal}_t$.

Figure~\ref{fig:main-payments} provides additional insight into the temporal evolution of the winning bids, payments and costs. 
As routing concentrates on $i^*$, the average winning bid (yellow), based on providers' accumulated cost estimates, approaches $\bar{c}_{i^*}$, while realized generation costs (green), averaged within each plotted window, settle near the same value.
The query price (purple) is exactly $1.25$ times the generation cost (green), since both curves use the same generations and averaging.
Early payments on non-exploration steps (blue) can fall below $\bar{c}_{i^*}$ because cheaper, unqualified providers still influence the critical payment before being excluded from $\hat{\Qcal}_t$.
The red curve displays a moving average of payments; Theorem~\ref{thm:regret} shows that, with high probability, their long-run average converges to $\bar{c}_{(2)}$.
Importantly, we also find that the user's payment for a guaranteed quality level is fundamentally determined by the intensity of competition among qualified providers on each benchmark. 
On the final pass, under strong competition on \texttt{GPQA}, the average payment per query on non-exploration steps lies $10\%$ above the generation cost $\bar{c}_{i^*}$ and $12\%$ below the listed price $\bar{v}_{i^*}$ that the user would pay in the current fixed-price market.
Including exploration, the average payment over the same window is $10\%$ below $\bar{v}_{i^*}$, while the limiting payment $\bar{c}_{(2)}$ is about $6\%$ below it.
Conversely, under weak competition on \texttt{GSM8K}, the average payment lies $71\%$ above $\bar{c}_{i^*}$ and $37\%$ above $\bar{v}_{i^*}$, due to the higher cost of the nearest qualified competitor.
Including exploration, the average payment over the final window is $51\%$ above $\bar{v}_{i^*}$, while the limiting payment $\bar{c}_{(2)}$ is about $71\%$ above it.
Altogether, our platform targets task-specific competitive pricing, with average payments approaching $\bar{c}_{(2)}$.
The competitive price may lie above or below the lowest listed price among qualified providers, so savings relative to routing by listed prices are not guaranteed.

\section{Discussion and Limitations}
\label{sec:discussion}
In this section, we highlight several limitations of our work and discuss avenues for future research.

\xhdr{Methodology}
We assumed each provider maximizes its worst-case estimated utility, since it knows neither the total number of user queries nor the joint cost--score distributions of its competitors. Under this assumption, we showed that bidding its current cost estimate is optimal.
That said, in Appendix~\ref{app:dynamic-deviations}, we showed that, when providers know the total number of user queries, cost-estimate bidding forms an approximate Nash equilibrium, and this is also true in the setting in which providers maximize their true utility based on realized profit. Moreover, when the platform also knows the horizon $T$ and sets $\delta=1/T$, the equilibrium error per query vanishes.
A natural open question is to investigate which bidding profiles, if any, form an exact Nash equilibrium whenever providers know the total number of user queries, and derive similar guarantees on provider selection and average payment under such profiles.

Further, whenever the cost difference between the two cheapest qualified providers is small, or an unqualified provider's quality sits just below the threshold, our platform---and more generally, any routing mechanism seeking to learn unknown provider qualities and costs---may require a relatively large number of queries to secure the best price for the user.
In such cases, future work could lower query requirements by leveraging quality and cost estimates across tasks.
Finally, in our platform, users carry the cost of exploration. However, because the most cost-competitive qualified provider directly benefits from this exploration, one may argue that this burden should be shifted to, or at least shared with, the providers, who typically possess greater financial resources.
Exploring alternative payment rules that redistribute such a burden is an interesting direction for future work.

\xhdr{Evaluation}
We evaluated our platform using simulated providers who compete to serve queries from benchmarking datasets with verifiable ground truth.
An important next step would be to evaluate our platform in settings where providers compete to serve open-ended queries and quality is measured automatically without relying on predefined correct outputs. 
Further, in our experiments, we estimated the generation costs by discounting publicly listed token prices by an arbitrary profit margin. 
Future research should leverage specialized profiling tools to measure these true costs precisely.
Ultimately, and despite the resource demands involved, deploying and evaluating the proposed platform in a real-world production environment remains crucial.

\xhdr{Practical considerations}
We assumed that our platform can accurately estimate the quality of the responses generated by the providers. In many domains where LLMs are routinely used, such as coding or mathematical reasoning, automated evaluation procedures make this feasible with minimal overhead. 
Moving forward, the viability of our platform in settings with less well-established quality metrics and benchmarks will depend on continued progress in the LLM evaluation literature. 
Lastly, we note that, while second-price auctions possess many desirable properties, they also make a provider's payment dependent on the bids of its competitors. This dependence leaves the mechanism potentially vulnerable to manipulation or collusion~\citep{Hendricks1989, MAILATH1991467, CHE2018398}.

\section{Conclusions}
\label{sec:conclusions}
In this work, we introduced a procurement platform connecting large language model providers with end users in which reverse second-price auctions drive token prices through provider competition.
More concretely, this mechanism incentivizes providers to truthfully bid their best estimate of the average cost to serve a user's query.
Further, we demonstrated that, with high probability, the platform routes all but a vanishing fraction of the user's queries to the most cost-competitive provider meeting a desired quality threshold. 
Moreover, this provider's profit per query served also converges with high probability to the difference between the second-lowest average cost among such providers and its own average cost.
Empirically, we demonstrated that the pricing margin on our platform varies significantly by task and quality threshold, and users may pay more or less than what they currently pay under fixed pricing depending on market conditions.
Ultimately, we hope that our work inspires users, providers, and routing platforms to explore alternative, more efficient economic paradigms for the market of LLM-as-a-service.

\vspace{2mm}
\xhdr{Acknowledgements} 
We thank Stratis Tsirtsis for feedback and fruitful discussions. Gomez-Rodriguez acknowledges support from the European Research Council (ERC) under the European Union'{}s Horizon 2020 research and innovation programme (grant agreement No. 101169607).

{ 
\small
\bibliographystyle{unsrtnat}
\bibliography{learning-cost-reliable-inference}
}

\clearpage
\newpage

\appendix

\section{Proofs}
\label{app:proofs}

\subsection{Proof of Proposition~\ref{prop:selection-threshold}.} \label{app:prop-proof}
    Fix a time step $t > N$ with $\mathcal{S}_t = \emptyset$, a provider $i \in \hat{\mathcal Q}_t$, and the other providers' bids $\mathbf{b}_{-i, t}$. Since $t$ is not an exploration step, the allocation rule (Eq.~\ref{eq:allocation}) selects a provider minimizing the adjusted bid
    \begin{equation}
        b_{j, t} - c_{\max}\beta(m_{j, t})
    \end{equation}
    over $j \in \hat{\mathcal Q}_t$.

    Provider $i$ is therefore selected whenever
    \begin{equation*}
        b_{i,t} - c_{\max}\beta(m_{i, t})< \min_{j \in \hat{\mathcal Q}_t \setminus\{i\}} \left(b_{j, t} - c_{\max}\beta(m_{j, t})\right).
    \end{equation*}
    Rearranging gives
    \begin{equation*}
        b_{i,t} < c_{\max}\beta(m_{i, t}) + \min_{j \in \hat{\mathcal Q}_t \setminus\{i\}} \left(b_{j, t} - c_{\max}\beta(m_{j, t})\right) = P_{i, t},
    \end{equation*}
    where the equality follows from the definition of $P_{i, t}$. Thus $b_{i, t} < P_{i, t}$ implies that provider $i$ is selected.

    Similarly, if
    \begin{equation*}
        b_{i, t} > P_{i, t},
    \end{equation*}
    then
    \begin{equation*}
        b_{i,t} - c_{\max}\beta(m_{i, t}) > \min_{j \in \hat{\mathcal Q}_t \setminus\{i\}} \left(b_{j, t} - c_{\max}\beta(m_{j, t})\right).
    \end{equation*}
    so provider $i$ does not minimize the adjusted bid and is therefore not selected. This proves the claim.

\subsection{Proof of Theorem~\ref{thm:robust-dominant}.}\label{app:robust-proof}\label{app:robust-dominant}
Throughout, we fix a provider $i \in [N]$, the strategies $\boldsymbol{\sigma}_{-i}$ of the other providers, and a history $H_t$ of the platform before step $t > N$, containing the observed history $H_{i, t}$ of provider $i$. For every step $s \geq t$, we write
\begin{equation*}
    G_{i,s} = \bigl( \pi_s - \hat c_{i,s} \bigr)\, \mathbbm{1}\{I_s = i\}
\end{equation*}
for the net payment of provider $i$ at step $s$, so that $U_{i \mid H_t}(\sigma_i, \boldsymbol{\sigma}_{-i}; \thetab_{-i}, T) = \mathbb{E}[\sum_{s=t}^{T} G_{i,s} \given H_t]$ by Eq.~\ref{eq:utilities}.

The proof rests on a single observation: once the history before a step is fixed, the net payment of provider $i$ at that step depends on its own bid only through whether it is assigned the query, and bidding the cost estimate is the best choice.

We begin with the following lemma.
\begin{lemma}\label{lem:one-step}
    Fix a step $s > N$ and a history $H_s$ of the platform before step $s$, containing the observed history $H_{i, s}$. For $b \in [0, c_{\max}]$, let $G_{i,s}(b)$ denote the net payment of provider $i$ at step $s$ if it bids $b_{i,s} = b$. Then, for every $b \in [0, c_{\max}]$,
    \begin{enumerate}[label=\alph*)]
        \item $G_{i, s}(\hat c_{i, s}) \geq 0$;
        \item $G_{i, s}(\hat c_{i, s}) \geq G_{i, s}(b)$;
        \item if $\mathcal{S}_s = \emptyset$, $i \in \hat{\Qcal}_s$, and $(\hat c_{i,s} - P_{i,s})(b - P_{i,s}) < 0$, then $G_{i,s}(b) < G_{i,s}(\hat c_{i,s})$.
    \end{enumerate}
\end{lemma}
\begin{proof}
    If $\hat{\mathcal{Q}}_s = \emptyset$ or $\mathcal{S}_s \neq \emptyset$, the allocation does not depend on the bids and the selected provider is paid $c_{\max}$. Hence $G_{i, s}(b) = (c_{\max} - \hat c_{i, s})\;\mathbbm{1}\{I_s = i\} \geq 0$ for every $b$, which proves a) and b), while c) does not apply. If $\Scal_s = \emptyset$ and $i \notin \hat{\Qcal}$, provider $i$ is not selected whatever it bids, so $G_{i, s}(b) = 0$ for every $b$ and all three claims hold.

    It remains to consider $\Scal_s = \emptyset$ and $i \in \hat{\Qcal}_s$. By Proposition~\ref{prop:selection-threshold}, provider $i$ is selected if $b < P_{i, s}$. Whenever it is selected, the payment rule (Eq.~\ref{eq:payment}) pays it $\pi_s = \min\{P_{i, s}, c_{\max}\}$, which does not depend on $b$. Therefore,
    \begin{equation}\label{eq:one-step-gain}
        G_{i,s}(b) = \bigl( \min\{P_{i,s}, c_{\max}\} - \hat c_{i,s} \bigr)\, \mathbbm{1}\{I_s = i\} \in \bigl\{ 0,\; \min\{P_{i,s}, c_{\max}\} - \hat c_{i,s} \bigr\}.
    \end{equation}
    We first compute $G_{i,s}(\hat c_{i,s})$. If $\hat c_{i,s} < P_{i,s}$, provider $i$ is selected and, since $\hat c_{i,s} \leq c_{\max}$, $\min\{P_{i,s}, c_{\max}\} \geq \min\{\hat c_{i,s}, c_{\max}\} = \hat c_{i,s}$, so $G_{i,s}(\hat c_{i,s}) = \min\{P_{i,s}, c_{\max}\} - \hat c_{i,s} \geq 0$. If $\hat c_{i,s} > P_{i,s}$, provider $i$ is not selected and $G_{i,s}(\hat c_{i,s}) = 0$, while $\min\{P_{i,s}, c_{\max}\} - \hat c_{i,s} \leq P_{i,s} - \hat c_{i,s} < 0$. If $\hat c_{i,s} = P_{i,s}$, the factor $\min\{P_{i,s}, c_{\max}\} - \hat c_{i,s}$ in Eq.~\ref{eq:one-step-gain} equals $\min\{\hat c_{i,s}, c_{\max}\} - \hat c_{i,s} = 0$, so $G_{i,s}(\hat c_{i,s}) = 0$ whether or not $i$ is selected. In all three cases,
    \begin{equation}\label{eq:truthful-gain}
        G_{i, s}(\hat{c}_{i, s}\bigr) = \bigl( \min\{P_{i, s}, c_{\max}\} - \hat{c}_{i, s}\bigr)^+ \geq 0
    \end{equation}
    which proves a). Claim b) follows since, by Eq.~\ref{eq:one-step-gain}, $G_{i, s}(b)$ is either $0$ or $\min\{P_{i, s}, c_{\max}\} - \hat{c}_{i, s}$, both of which are at most the right-hand side of Eq.~\ref{eq:truthful-gain}.

    For c), suppose first that $\hat c_{i,s} < P_{i,s} < b$. Then bidding $\hat c_{i,s}$ gets the query while bidding $b$ does not, so $G_{i,s}(b) = 0$. Moreover, $\hat c_{i,s} < b \leq c_{\max}$ and $\hat c_{i,s} < P_{i,s}$ give $\min\{P_{i,s}, c_{\max}\} > \hat c_{i,s}$, so $G_{i,s}(\hat c_{i,s}) > 0 = G_{i,s}(b)$ by Eq.~\ref{eq:truthful-gain}. Suppose instead that $b < P_{i,s} < \hat c_{i,s}$. Then bidding $b$ gets the query while bidding $\hat c_{i,s}$ does not, so $G_{i,s}(\hat c_{i,s}) = 0$, whereas $G_{i,s}(b) = \min\{P_{i,s}, c_{\max}\} - \hat c_{i,s} \leq P_{i,s} - \hat c_{i,s} < 0$ by Eq.~\ref{eq:one-step-gain}.
\end{proof}

    To lighten the notation, we omit the fixed arguments $\boldsymbol{\sigma}_{-i}$ and $\thetab_{-i}$, and write $U_{i \mid H_t}(\sigma_i; T)$. By Eq.~\ref{eq:truthful-strategy}, $\sigma_{i,s}^*(H_{i, s}) = \hat c_{i, s}$ at every step $s$.

    \emph{Claim i).}
    Under $\sigma_i^*$, provider $i$ bids $\hat c_{i,s}$ at every step $s \geq t$, so Lemma~\ref{lem:one-step}a), applied at step $s$ to the realized history before it, gives $G_{i,s} \geq 0$. Summing over $s = t, \dots, T$ and taking the expectation conditional on $H_t$ yields $U_{i \mid H_t}(\sigma_i^*; T) \geq 0$ for every $T \geq t$ and every $\thetab_{-i}$.

    \emph{Claim ii).}
    By Lemma~\ref{lem:one-step}a), every term $G_{i, s}$ is non-negative under $\sigma_i^*$, so $U_{i \mid H_t}(\sigma_i^*; T)$ is non-decreasing in $T$, and its infimum over $T \geq t$ is attained at $T = t$. Given $H_t$ and the fixed strategies $\sigma_{-i}$, the other providers' round-$t$ bids $b_{j,t}=\sigma_{j,t}(H_{j,t})$, the selection counts, and the estimates entering the allocation and payment rules are fixed, so the net payment at step $t$ does not depend further on $\thetab_{-i}$, and $U_{i \mid H_t}(\sigma_i^*; t) = \mathbb{E}[G_{i,t}(\hat c_{i,t}) \given H_t]$ for every $\thetab_{-i}$. Hence,
    \begin{equation}\label{eq:truthful-worst-case}
        \inf_{T \geq t,\, \thetab_{-i}} U_{i \mid H_t}(\sigma_i^*; T) = \mathbb{E}\bigl[G_{i,t}(\hat c_{i,t}) \given H_t\bigr].
    \end{equation}
    Now fix any strategy $\sigma_i$ and let $b = \sigma_{i, t}(H_{i,t})$ be its bid at step $t$. Taking $T = t$ in the infimum and then applying Lemma~\ref{lem:one-step}b) for every realization of the platform's draw at step $t$,
    \begin{equation}\label{eq:deviation-worst-case}
        \inf_{T \geq t,\, \thetab_{-i}} U_{i \mid H_t}(\sigma_i; T)
        \leq U_{i \mid H_t}(\sigma_i; t)
        = \mathbb{E}\bigl[G_{i,t}(b) \given H_t\bigr]
        \leq \mathbb{E}\bigl[G_{i,t}(\hat c_{i,t}) \given H_t\bigr].
    \end{equation}
    Combining Eqs.~\ref{eq:truthful-worst-case} and~\ref{eq:deviation-worst-case} proves claim ii).
    
    \emph{Claim iii).}
    Under the assumptions of claim iii), step $t$ is not an exploration step, so $G_{i,t}(b)$ and $G_{i,t}(\hat c_{i,t})$ are determined by $H_t$, and Lemma~\ref{lem:one-step}c) gives $G_{i,t}(b) < G_{i,t}(\hat c_{i,t})$. The last inequality in Eq.~\ref{eq:deviation-worst-case} is therefore strict, and so is the inequality in claim ii).

% \begin{theorem}\label{thm:sample-complexity}
%     When all providers choose their equilibrium strategy $\boldsymbol{\sigma}^*$, with probability
%     at least $1 - \delta$, for every horizon $T \geq N$, the number of time steps in which the platform does not assign the query
%     to the lowest-cost qualified provider, $I_t \neq i^*$, is at most
%     %
%     \begin{equation}
%         N - 1 + \sum_{i \notin \mathcal{Q}} \left\lceil \frac{4}{\varepsilon_i^2}
%         \ln\!\left(\frac{57 N}{\delta\, \varepsilon_i^4}\right) \right\rceil
%         + \sum_{i \in \mathcal{Q} \setminus \{i^*\}} \max\left\{\left\lceil \frac{4 c_{\max}^2}{\Delta_i^2}
%         \ln\!\left(\frac{57 N c_{\max}^4}{\delta\, \Delta_i^4}\right) \right\rceil, k\left(\frac{T}{N}\right)^{\!\alpha}\right\}.
%     \end{equation}
% \end{theorem}

\subsection{Proof of Theorem~\ref{thm:sample-complexity}.}\label{app:sample-complexity-proof}
    We work under the event $\mathcal E$, which by Lemma~\ref{lem:good-event} occurs with probability at least $1-\delta$. In the first $N$ initialization steps, each provider is selected exactly once. Hence, at most $N-1$ initialization steps are not assigned to $i^*$.

    We next consider the post-initialization steps. Every provider $i\neq i^*$ is either unqualified or qualified and suboptimal. By Lemma~\ref{lem:unqualified-elimination}, each provider $i\notin\mathcal Q$ is selected at most
    \begin{equation*}
        M_i = \left\lceil \frac{4}{\varepsilon_i^2} \ln\!\left( \frac{57N}{\delta\varepsilon_i^4} \right) \right\rceil = M(\varepsilon_i)
    \end{equation*}
    times after initialization.

    For each provider $i\in\mathcal Q\setminus\{i^*\}$, Lemma~\ref{lem:qualified-suboptimal-selection} gives at most
    \begin{equation*}
        \max\left\{L_i,\, \left\lceil g(T)\right\rceil-1\right\}
    \end{equation*}
    post-initialization selections, where $L_i = M(\Delta_i/c_{\max})$ by Lemma~\ref{lem:ordinary-selection-threshold} and, for
    \begin{equation*}
        g(t)=k\left(\frac{t}{N}\right)^{\!\alpha},
    \end{equation*}
    we have $\lceil g(T)\rceil - 1 = \lceil k(T/N)^{\alpha}\rceil - 1$.

    Since every step with $I_t \neq i^*$ selects exactly one provider other than $i^*$, summing these bounds and adding the $N-1$ initialization steps gives
    \begin{align*}
        \sum_{t=1}^T \mathbbm{1}\{I_t \neq i^*\}
        &\leq N - 1 + \sum_{i\notin\mathcal Q} M_i + \sum_{i\in\mathcal Q\setminus\{i^*\}} \max\left\{L_i,\, \left\lceil g(T)\right\rceil-1\right\} \\
        &= N - 1 + \sum_{i \notin \mathcal{Q}} M(\varepsilon_i)
        + \sum_{i \in \mathcal{Q} \setminus \{i^*\}} \max\left\{ M\!\left(\frac{\Delta_i}{c_{\max}}\right),\;
        \left\lceil k\left(\frac{T}{N}\right)^{\!\alpha}\right\rceil - 1 \right\} \\
        &\leq N - 1 + B_{\mathrm{id}}(T),
    \end{align*}
    which proves the claimed bound.

\begin{corollary}\label{cor:vanishing-nonoptimal-selections}
Suppose that all providers use their equilibrium strategy $\boldsymbol{\sigma}^*$. With probability at least $1-\delta$,
\begin{equation}
    \frac{1}{T} \sum_{t=1}^T \mathbbm{1}\{I_t\neq i^*\} \longrightarrow 0 \qquad\text{as } T \to \infty.
    \label{eq:vanishing-nonoptimal-selections}
\end{equation}
In particular, the platform assigns a fraction tending to one of the user's queries to the lowest-cost qualified provider $i^*$.
\end{corollary}
\begin{proof}
    On the event $\mathcal E$, Theorem~\ref{thm:sample-complexity} gives
    \begin{equation*}
        \sum_{t=1}^T \mathbbm{1}\{I_t\neq i^*\} \leq N - 1 + \sum_{i \notin \mathcal Q} M_i + \sum_{i \in \mathcal Q \setminus\{i^*\}} \max\left\{L_i,\,k\left(\frac{T}{N}\right)^{\!\alpha}\right\}.
    \end{equation*}
    Dividing by $T$ yields
    \begin{equation*}
       \frac{1}{T} \sum_{t=1}^T \mathbbm{1}\{I_t\neq i^*\} \leq \frac{N - 1 + \sum_{i \notin \mathcal Q} M_i}{T} + \frac{1}{T} \sum_{i \in \mathcal Q \setminus\{i^*\}} \max\left\{L_i,\,k\left(\frac{T}{N}\right)^{\!\alpha}\right\}.
    \end{equation*}
    The quantities $M_i$ and $L_i$ are independent of $T$. Moreover, since $\alpha\in(0,1)$,
    \begin{equation*}
        \frac{k(T/N)^\alpha}{T} = \frac{k}{N^\alpha}T^{\alpha-1} \xrightarrow{T \to \infty} 0.
    \end{equation*}
    Hence every term on the right-hand side converges to zero, which proves \eqref{eq:vanishing-nonoptimal-selections}. Since $\Pr(\mathcal E) \geq 1 - \delta$ by Lemma~\ref{lem:good-event}, the result holds with probability at least $1 - \delta$.
\end{proof}

\subsection{Proof of Theorem~\ref{thm:regret}}
\label{app:regret-theorem-proof}
\label{app:regret}

We show that, with high probability, the payment on non-exploration steps converges to $\bar c_{(2)}$. Throughout, we write
\begin{equation*}
    \bar{M} = \max\bigl(\{1\} \cup \{M_i : i \notin \Qcal\}\bigr), \qquad T_0 = N\max\bigl\{1, (\bar M/k)^{1/\alpha}\bigr\}
\end{equation*}
so that every time step $t > T_0$ comes after initialization and satisfies $g(t) > \bar M$.

The second claim of Theorem~\ref{thm:payment-convergence} gives Theorem~\ref{thm:regret}.

\begin{theorem}\label{thm:payment-convergence}
    When all providers choose their equilibrium strategy $\boldsymbol{\sigma}^*$, with probability at least $1 - \delta$, for every time step $t > T_0$ that is not an exploration step, \ie, $\mathcal{S}_t = \emptyset$, the payment satisfies
    \begin{equation}\label{eq:two-sided-payment}
        \bar{c}_{(2)} - 3\, c_{\max}\, \beta(g(t)) \;\leq\; \pi_t \;\leq\; \bar{c}_{(2)} + c_{\max}\, \beta(g(t)).
    \end{equation}
    Hence, $\pi_t \to \bar{c}_{(2)}$ as $t \to \infty$ along these time steps. Moreover, for every $T \geq T_0$,
    \begin{align}
        \frac{1}{T} \sum_{t=1}^{T} \bigl| \pi_t - \bar{c}_{(2)} \bigr| \leq\;& \frac{T_0}{T} c_{\max} + k \left(\frac{N}{T}\right)^{\!1 - \alpha} \bigl( c_{\max} - \bar{c}_{(2)} \bigr) \nonumber\\
        &+ \frac{6\, c_{\max}}{2 - \alpha} \sqrt{ \frac{1}{2k} \left(\frac{N}{T}\right)^{\!\alpha} \ln\!\left( \frac{2\pi^2 N k^2}{3\delta} \left(\frac{T}{N}\right)^{\!2\alpha} \right) }.
        \label{eq:average-absolute-payment}
    \end{align}
\end{theorem}
\begin{proof}
    We work under the event $\mathcal{E}$, which by Lemma~\ref{lem:good-event} occurs with probability at least $1 - \delta$.

    The first claim, Eq.\ref{eq:two-sided-payment}, is Lemma~\ref{lem:two-sided-payment}. Since $g(t) \to\infty$ and $\beta(x) \to 0$ as $x \to \infty$, we have $\beta(g(t)) \to 0$, and hence $\pi_t \to \bar c_{(2)}$ in non-exploration steps.

    For the second claim, fix $T \geq T_0$. In $\mathcal{E}$, every post-initialization step is either an exploration step or a non-exploration step, so we partition the steps $t \leq T$ into the three sets
    \begin{equation*}
        \mathcal R_1 = \{t \leq T_0\}, \qquad \mathcal R_2 = \{T_0 < t \leq T : \mathcal S_t \neq \emptyset\}, \qquad \mathcal R_3 = \{T_0 < t \leq T : \mathcal S_t = \emptyset\}.
    \end{equation*}

    The set $\mathcal R_1$ contains at most $T_0$ steps, each of which contributes at most $c_{\max}$ to the sum in the left-hand side of Eq.~\ref{eq:average-absolute-payment} by Lemma~\ref{lem:payment-range}. Hence,
    \begin{equation*}
        \sum_{t \in \mathcal R_1} \bigl| \pi_t - \bar c_{(2)} \bigr| \leq T_0\, c_{\max}.
    \end{equation*}

    On every step in $\mathcal R_2$, the selected provider is paid $c_{\max}$, so the step contributes $c_{\max} - \bar c_{(2)}$. By Lemma~\ref{lem:forced-exploration-bound}, there are at most $N(\lceil g(T) \rceil - 1) \leq Ng(T) = kN(T/N)^{\alpha}$ exploration steps. Hence,
    \begin{equation*}
        \sum_{t \in \mathcal R_2} \bigl| \pi_t - \bar c_{(2)} \bigr| \leq kN\left(\frac{T}{N}\right)^{\!\alpha} \bigl( c_{\max} - \bar c_{(2)} \bigr).
    \end{equation*}

    Finally, for every $t \in \mathcal R_3$, Lemma~\ref{lem:two-sided-payment} gives
    \begin{equation*}
        \bigl| \pi_t - \bar c_{(2)} \bigr| \leq 3c_{\max}\beta(g(t)).
    \end{equation*}
    Since $1 \leq \bar M < g(t) \leq g(T)$ and $1/g(t) = (N/t)^{\alpha}/k$,
    \begin{equation*}
        \beta(g(t)) = \sqrt{\frac{1}{2g(t)}\ln\!\left(\frac{2\pi^2 N g(t)^2}{3\delta}\right)} \leq \sqrt{\frac{1}{2k}\ln\!\left(\frac{2\pi^2 N k^2}{3\delta}\left(\frac{T}{N}\right)^{\!2\alpha}\right)} \left(\frac{N}{t}\right)^{\!\alpha/2}.
    \end{equation*}
    Moreover,
    \begin{equation*}
        \sum_{t=1}^{T} t^{-\alpha/2} \leq \int_
        0^T x^{-\alpha/2}\, dx = \frac{2}{2 - \alpha}\, T^{1 - \alpha/2}.
    \end{equation*}
    Combining the two preceding inequalities gives
    \begin{equation*}
        \sum_{t \in \mathcal R_3} \bigl| \pi_t - \bar c_{(2)} \bigr| \leq \frac{6c_{\max}}{2 - \alpha} \sqrt{\frac{N^{\alpha}}{2k}\ln\!\left(\frac{2\pi^2 N k^2}{3\delta}\left(\frac{T}{N}\right)^{\!2\alpha}\right)}\; T^{1 - \alpha/2}.
    \end{equation*}
    Summing the contributions of $\mathcal R_1$, $\mathcal R_2$, and $\mathcal R_3$ and dividing by $T$, together with
    \begin{equation*}
        \frac{kN(T/N)^{\alpha}}{T} = k\left(\frac{N}{T}\right)^{\!1 - \alpha}, \qquad \frac{N^{\alpha/2}\, T^{1 - \alpha/2}}{T} = \left(\frac{N}{T}\right)^{\!\alpha/2},
    \end{equation*}
    proves \eqref{eq:average-absolute-payment}.
\end{proof}

\begin{corollary}\label{cor:optimal-alpha}
    Under the assumptions of Theorem~\ref{thm:payment-convergence}, for $N$, $k$, $\delta$ fixed and $T \to \infty$, the right-hand side of Eq.~\ref{eq:average-absolute-payment} is
    \begin{equation*}
        O\!\left(T^{-(1-\alpha)}\right) + O\!\left(T^{-\alpha/2}\sqrt{\ln T}\right)
        = O\!\left(T^{-\min\{1-\alpha,\,\alpha/2\}}\sqrt{\ln T}\right).
    \end{equation*}
    The exponent $\min\{1-\alpha, \alpha/2\}$ is maximized over $\alpha \in (0,1)$ at $\alpha = 2/3$, for which the bound is $O\!\left(T^{-1/3}\sqrt{\ln T}\right)$.
\end{corollary}
\begin{proof}
    The first term of Eq.~\ref{eq:average-absolute-payment} is $O(1/T)$ and the second is $O(T^{-(1-\alpha)})$; since $\alpha > 0$, the first is dominated by the second. In the third term, $\ln\!\big(\tfrac{2\pi^2 N k^2}{3\delta}(T/N)^{2\alpha}\big) = O(\ln T)$, so the term is $O(T^{-\alpha/2}\sqrt{\ln T})$. Finally, $1-\alpha$ is decreasing and $\alpha/2$ is increasing in $\alpha$, so $\min\{1-\alpha,\alpha/2\}$ is maximized where they coincide, $1-\alpha = \alpha/2$, i.e., at $\alpha = 2/3$, where both equal $1/3$.
\end{proof}

\clearpage
\newpage

\section{Additional theoretical guarantees}
\label{app:user-guarantees}
In this section, we complement Theorems~\ref{thm:sample-complexity} and~\ref{thm:regret} with additional guarantees on selection, quality, generation costs, and user payments.

\subsection{Eventual selection on non-exploration steps}
\begin{corollary}\label{cor:eventual-ordinary-optimality}
Suppose that all providers use their equilibrium strategy $\boldsymbol{\sigma}^*$. On $\mathcal E$, if at some step $t>N$,
\begin{equation}
    g(t) > M_i \quad\text{for every }i \notin \mathcal Q, \qquad\text{and} \qquad
    g(t) > L_i \quad\text{for every }i \in \mathcal Q\setminus\{i^*\},
    \label{eq:ordinary-identification-condition}
\end{equation}
then every non-exploration step $\tau\geq t$ is assigned to the optimal provider $i^*$.  
\end{corollary}
\begin{proof}
    Fix a non-exploration step $\tau\geq t$ and let $i=I_\tau$. On $\mathcal E$, the selected provider is eligible. Since $\mathcal S_\tau=\emptyset$ and $g$ is non-decreasing, $m_{i,\tau}\geq g(\tau)\geq g(t)$. If $i\notin\mathcal Q$, Eq.~\ref{eq:ordinary-identification-condition} gives $m_{i,\tau}>M_i$, contradicting Lemma~\ref{lem:unqualified-elimination}. If $i\in\mathcal Q\setminus\{i^*\}$, it gives $m_{i,\tau}>L_i$, contradicting Lemma~\ref{lem:ordinary-selection-threshold}. Hence $I_\tau=i^*$.
\end{proof}

\subsection{Quality and generation-cost guarantees}
We give cumulative bounds for quality and generation cost.

\begin{theorem}\label{thm:quality-guarantee}
On the event $\mathcal{E}$, the sum
\begin{equation*}
    R_T^{\mathrm{qual}} = \sum_{t=1}^{T} \bigl(q_{\min} - q_{I_t}\bigr)^+,
\end{equation*}
which measures the cumulative quality under $q_{\min}$ obtained by the user, satisfies,
\begin{equation}\label{eq:quality-regret-bound}
    R_T^{\mathrm{qual}}  \leq \min\left\{N + 2\sqrt{2N(T-N) \ln\!\left( \frac{2\pi^2 N T^2}{3\delta}\right)},\;\sum_{i\notin\mathcal{Q}}\bigl(M_i+1\bigr)\varepsilon_i\right\}.
\end{equation}
Moreover, the number of queries assigned to unqualified providers satisfies,
\begin{equation} \label{eq:unqualified-selection-bound}
    \sum_{t=1}^{T} \mathbbm{1}\{I_t\notin\mathcal{Q}\} \leq \sum_{i\notin\mathcal{Q}}(M_i+1).
\end{equation}
Consequently,
\begin{equation}\label{eq:vanishing-quality-guarantees}
    \frac{R_T^{\mathrm{qual}}}{T} = O\!\left(\frac{1}{T}\right) \longrightarrow 0, \qquad \frac{1}{T}\sum_{t=1}^{T} \mathbbm{1}\{I_t\notin\mathcal{Q}\} = O\!\left(\frac{1}{T}\right) \longrightarrow 0
\end{equation}
\end{theorem}
\begin{proof}
    In the first $N$ initialization steps, every provider is selected once. Since $(q_{\min} - q_i)^+\leq 1$, the contribution of the first $N$ steps to the sum $R_T^{\mathrm{qual}}$ is at most $N$.

    Consider now a post-initialization step $t > N$. If $I_t\in\mathcal{Q}$, then $q_{I_t}\geq q_{\min}$ and the step does not contribute to the sum $R_T^{\mathrm{qual}}$. If $I_t \notin \mathcal{Q}$, then the selected provider must belong to $\hat{\mathcal{Q}}_t$. In the proof of Lemma~\ref{lem:unqualified-elimination} (Eq.~\ref{eq:unqualified-radius-condition}), we showed that this implies that:
    \begin{equation*}
        \varepsilon_{I_t} = \bigl(q_{\min}-q_{I_t}\bigr)^+  \leq 2\beta(m_{I_t,t}).
    \end{equation*}
    Applying Lemma~\ref{lem:radius-sum} with $\kappa=2$ therefore gives
    \begin{equation*}
        \sum_{t=N+1}^{T} \bigl(q_{\min}-q_{I_t}\bigr)^+ \leq 2\sqrt{2N(T-N)\ln\!\left(\frac{2\pi^2NT^2}{3\delta}\right)}.
    \end{equation*}
    Adding the initialization contribution proves the first bound in \eqref{eq:quality-regret-bound}.

    Alternatively, Lemma~\ref{lem:unqualified-elimination}, the elimination bound for unqualified providers, shows that every $i\notin\mathcal{Q}$ is selected at most $M_i$ times after initialization. Including its initialization selection, provider $i$ therefore contributes at most
    \begin{equation*}
        (M_i+1)\varepsilon_i
    \end{equation*}
    to the cumulative quality regret. Summing over all $i\notin\mathcal{Q}$ proves the second bound in \eqref{eq:quality-regret-bound}.

    The same selection bound immediately gives
    \begin{equation*}
        \sum_{t=1}^{T} \mathbbm{1}\{I_t\notin\mathcal{Q}\} \leq \sum_{i\notin\mathcal{Q}}(M_i+1).
    \end{equation*}
    proving \eqref{eq:unqualified-selection-bound}.

    Finally, the right-hand sides of the second bound in \eqref{eq:quality-regret-bound} and of \eqref{eq:unqualified-selection-bound} are independent of $T$. Dividing both inequalities by $T$ therefore gives \eqref{eq:vanishing-quality-guarantees}.
\end{proof}

We bound the cumulative excess mean generation cost of the selected providers relative to the optimal provider $i^*$, i.e., how much more the user's queries cost to serve than they would if every query were served by $i^*$. We refer to this quantity as the generation-cost regret.

\begin{theorem}\label{thm:generation-regret}
Suppose that all providers use their equilibrium strategy $\boldsymbol{\sigma}^*$. On the event $\mathcal{E}$, the sum
\begin{equation*}
    R_T^{\mathrm{gen}} = \sum_{t=1}^{T} \bigl(\bar c_{I_t} - \bar c_{i^*}\bigr)^+
\end{equation*}
satisfies
\begin{align}
    R_T^{\mathrm{gen}} \leq \min\Biggl\{ & (N-1)\lceil g(T)\rceil c_{\max} + 2c_{\max}\sqrt{2N(T-N)\ln\!\left(\frac{2\pi^2 N T^2}{3\delta}\right)}, \nonumber\\
    & \sum_{i\notin\mathcal{Q}} (M_i+1)\bigl(\bar c_i - \bar c_{i^*}\bigr)^+ + \sum_{i\in\mathcal{Q}\setminus\{i^*\}} \bigl(\max\{L_i,\lceil g(T)\rceil-1\}+1\bigr)\Delta_i \Biggr\}.
    \label{eq:generation-regret-bound}
\end{align}
\end{theorem}
\begin{proof}
    Since every mean cost lies in $[0, c_{\max}]$, each term of $R_T^{\mathrm{gen}}$ is at most $c_{\max}$, and it vanishes whenever $I_t = i^*$.

    For the first bound, we partition the time steps up to $T$ into initialization steps, exploration steps, and non-exploration steps. During initialization, at most $N-1$ steps are not assigned to $i^*$, which contribute at most $(N-1)c_{\max}$. By Lemma~\ref{lem:forced-exploration-bound}, every provider is selected in at most $\lceil g(T)\rceil - 1$ exploration steps, so the exploration steps assigned to a provider other than $i^*$ contribute at most $(N-1)(\lceil g(T)\rceil - 1)c_{\max}$. Finally, consider a non-exploration step $t > N$ with $I_t = i \neq i^*$. On $\mathcal{E}$, Lemma~\ref{lem:qualified-eligibility} gives $i^* \in \hat{\mathcal{Q}}_t$, so the allocation rule (Eq.~\ref{eq:allocation}) implies
    \begin{equation*}
        b_{i,t} - c_{\max}\beta(m_{i,t}) \leq b_{i^*,t} - c_{\max}\beta(m_{i^*,t}),
    \end{equation*}
    and Lemma~\ref{lem:adjusted-bid-bounds}, applied to both providers, yields $\bar c_i - 2c_{\max}\beta(m_{i,t}) \leq \bar c_{i^*}$. Hence, letting $a_t = (\bar c_{I_t} - \bar c_{i^*})^+$ on non-exploration steps and $a_t = 0$ on exploration steps, we have $a_t \leq 2c_{\max}\beta(m_{I_t,t})$ for every $t = N+1, \ldots, T$, and Lemma~\ref{lem:radius-sum} with $\kappa = 2c_{\max}$ gives
    \begin{equation*}
        \sum_{t=N+1}^{T} a_t \leq 2c_{\max}\sqrt{2N(T-N)\ln\!\left(\frac{2\pi^2 N T^2}{3\delta}\right)}.
    \end{equation*}
    Summing the three contributions proves the first bound.

    For the second bound, Lemma~\ref{lem:unqualified-elimination} shows that every $i \notin \mathcal{Q}$ is selected at most $M_i$ times after initialization, and Lemma~\ref{lem:qualified-suboptimal-selection} shows that every $i \in \mathcal{Q}\setminus\{i^*\}$ is selected at most $\max\{L_i, \lceil g(T)\rceil - 1\}$ times after initialization. Including its initialization selection, each provider $i \neq i^*$ therefore contributes at most its number of selections times $(\bar c_i - \bar c_{i^*})^+$ to $R_T^{\mathrm{gen}}$, where $(\bar c_i - \bar c_{i^*})^+ = \Delta_i$ for $i \in \mathcal{Q}\setminus\{i^*\}$. Summing over $i \neq i^*$ proves the second bound.
\end{proof}

\subsection{Cumulative excess payments}
\label{app:payment-guarantees}

\begin{theorem}\label{thm:excess-payment-general}
    On $\mathcal E$, when all providers use their equilibrium strategy $\boldsymbol{\sigma}^*$,
    \begin{align}
        \sum_{t=1}^{T} \bigl(\pi_t-\bar c_{(2)}\bigr)^+ \leq\;& \left[N + N\bigl(\lceil g(T)\rceil-1\bigr)\right]\bigl(c_{\max}-\bar c_{(2)}\bigr) \nonumber\\
        &\quad + c_{\max} \sqrt{2N(T-N)\ln\!\left(\frac{2\pi^2NT^2}{3\delta}\right)}.
        \label{eq:excess-payment-general}
    \end{align}
\end{theorem}
\begin{proof}
    Partition the time steps up to $T$ into initialization steps, exploration steps, and non-exploration steps.

    During initialization, each of the $N$ providers is paid $c_{\max}$, so the total excess payment contributed by these steps is
    \begin{equation*}
        N\bigl(c_{\max}-\bar c_{(2)}\bigr).
    \end{equation*}

    On every exploration step, the selected provider is also paid $c_{\max}$. By Lemma~\ref{lem:forced-exploration-bound}, there are at most $N(\lceil g(T)\rceil-1)$ such steps. Their total contribution is therefore at most
    \begin{equation*}
        N\bigl(\lceil g(T)\rceil-1\bigr) \bigl(c_{\max}-\bar c_{(2)}\bigr).
    \end{equation*}

    Finally, on every non-exploration step $t>N$, Lemma~\ref{lem:ordinary-payment} gives
    \begin{equation*}
        \bigl(\pi_t-\bar c_{(2)}\bigr)^+ \leq c_{\max}\beta(m_{I_t,t}).
    \end{equation*}

    Hence, for the non-exploration steps, applying Lemma~\ref{lem:radius-sum},
    \begin{align*}
        \sum_{\substack{t=N+1\\ t\text{non-exploration}}}^{T} \bigl(\pi_t-\bar c_{(2)}\bigr)^+ &\leq c_{\max} \sum_{t=N+1}^{T} \beta(m_{I_t,t})\\
        &\leq c_{\max} \sqrt{2N(T-N)\ln\!\left(\frac{2\pi^2NT^2}{3\delta}\right)}.
    \end{align*}
    Summing the three contributions proves~\eqref{eq:excess-payment-general}.
\end{proof}

\begin{corollary}\label{cor:excess-payment-polynomial}
    Suppose
    \begin{equation*}
        g(t)=k\left(\frac{t}{N}\right)^{\!\alpha}, \qquad k>0,\quad \alpha\in(0,1).
    \end{equation*}
    On $\mathcal E$, when all providers use their equilibrium strategy $\boldsymbol{\sigma}^*$,
    \begin{align}
        \frac{1}{T}\sum_{t=1}^{T} \bigl(\pi_t-\bar c_{(2)}\bigr)^+ \leq\;& \left(\frac{N}{T} + k\left(\frac{N}{T}\right)^{1-\alpha}\right) \bigl(c_{\max}-\bar c_{(2)}\bigr) \nonumber\\
        &\quad + c_{\max}\sqrt{\frac{2N}{T}\ln\!\left(\frac{2\pi^2NT^2}{3\delta}\right)}.
        \label{eq:average-excess-payment}
    \end{align}
\end{corollary}
\begin{proof}
    Since
    \begin{equation*}
        g(T) = k\left(\frac{T}{N}\right)^{\!\alpha},
    \end{equation*}
    the forced exploration term in Theorem~\ref{thm:excess-payment-general} satisfies
    \begin{equation}
        N\bigl(\lceil g(T) \rceil - 1\bigr) \leq Ng(T) = kN\left(\frac{T}{N}\right)^{\!\alpha}
    \end{equation}
    Substituting this into \eqref{eq:excess-payment-general} and dividing by $T$ gives
    \begin{equation*}
        \frac{Ng(T)}{T} = k\left(\frac{N}{T}\right)^{\!1 - \alpha}
    \end{equation*}
    which proves the result.
\end{proof}

\clearpage
\newpage

\section{Auxiliary lemmas}\label{app:auxiliary-lemmas} 
In this section, we discuss, state and prove several auxiliary lemmas needed for the proofs in Appen\-di\-ces~\ref{app:proofs} and~\ref{app:user-guarantees}.

\subsection{Concentration and confidence radii}
\label{app:good-event}
\label{app:generic-counting-bounds}

We define a high-probability event under which, for every provider and every selection count, the empirical quality and cost averages are simultaneously close to their corresponding population averages. For each provider $i \in [N]$ and $m \geq 1$, let $\hat q_{i,m}$ and $\hat c_{i,m}$ denote the empirical means of the scores and generation costs observed over the first $m$ queries served by provider $i$, respectively. Recall that
\begin{equation*}
    \beta(m) = \sqrt{\frac{1}{2m}\ln\!\left(\frac{2\pi^2 N m^2}{3\delta}\right)}.
\end{equation*}

\begin{lemma}
\label{lem:good-event}
Let $\mathcal E$ denote the event that, simultaneously for every provider $i \in [N]$ and every $m \geq 1$,
\begin{equation} \label{eq:good-event}
    \left|\hat q_{i,m}-q_i\right| \leq \beta(m), \qquad \left|\hat c_{i,m}-\bar c_i\right| \leq c_{\max}\beta(m).
\end{equation}
Then, $\Pr(\mathcal E) \geq 1-\delta$. 
\end{lemma}
\begin{proof}
    Fix a provider $i \in [N]$, and let $\tau_i(1)<\tau_i(2)<\cdots$ denote the steps in which provider $i$ is selected. Given the history and bids, the platform selects a provider independently of the current query, including when breaking ties, and before the score and generation cost are realized. The cost--score pairs observed at provider $i$'s successive selections are therefore i.i.d.\ draws from $\theta_i$. The mean cost is $\bar c_i$, and the mean score is $q_i$. Thus, for every fixed selection count $m\geq1$, Hoeffding's inequality and the fact that scores lie in $[0,1]$ give
    \begin{equation*}
        \Pr \! \left(\left| \hat q_{i,m} - q_i \right| > \beta(m) \right) \leq 2\exp\!\left(-2m \beta(m)^2\right).
    \end{equation*}
    Similarly, since generation costs lie in $[0,c_{\max}]$,
    \begin{equation*}
        \Pr \! \left(\left| \hat c_{i,m} - \bar c_i \right| > c_{\max}\beta(m) \right) \leq 2\exp\!\left(-2m \beta(m)^2\right).
    \end{equation*}
    By the definition of $\beta(m)$,
    \begin{equation*}
        2\exp\!\left(-2m \beta(m)^2\right) = \frac{3\delta}{\pi^2 N m^2}.
    \end{equation*}
    A union bound over the two estimates, all providers, and all selection counts therefore gives
    \begin{align*}
        \Pr(\mathcal E^c) &\leq \sum_{i=1}^N \sum_{m=1}^{\infty} 2\frac{3\delta}{\pi^2 N m^2} \\
        &= \frac{6\delta}{\pi^2} \sum_{m=1}^{\infty}\frac{1}{m^2} = \delta,
    \end{align*}
    where $\sum_{m=1}^{\infty}m^{-2} = \pi^2/6$. Hence,
    \begin{equation*}
        \Pr(\mathcal E) \geq 1 - \delta.
    \end{equation*}
\end{proof}

We prove a technical result showing that a lower bound on the confidence radius $\beta(m)$ implies an upper bound on the selection count $m$. We use this result in the proof of Lemma~\ref{lem:ordinary-selection-threshold} and Lemma~\ref{lem:unqualified-elimination}.

\begin{lemma}\label{lem:radius-inversion}
For any $a > 0$ and $m \geq 1$, if
\begin{equation} \label{eq:radius-inversion-condition}
    a \leq 2\beta(m),
\end{equation}
then
\begin{equation}\label{eq:radius-inversion}
    m \leq \frac{4}{a^2} \ln\!\left(\frac{57N}{\delta a^4}\right).
\end{equation}
\end{lemma}
\begin{proof}
    By \eqref{eq:radius-inversion-condition} and the definition of $\beta(m)$,
    \begin{equation*}
        \frac{1}{2m} \ln\!\left(\frac{2\pi^2 N m^2}{3\delta}\right) \geq \frac{a^2}{4}.
    \end{equation*}
    Equivalently,
    \begin{equation*}
        m \leq \frac{2}{a^2} \ln\!\left(\frac{2\pi^2 N m^2}{3\delta} \right) = \frac{2}{a^2} \ln\!\left(\frac{2\pi^2 N}{3\delta} \right) + \frac{4}{a^2}\ln m.
    \end{equation*}
    Set
    \begin{equation*}
        A = \frac{2}{a^2} \ln\!\left(\frac{2\pi^2 N}{3\delta}\right), \qquad B = \frac{4}{a^2},
    \end{equation*}
    so that $m \leq A + B\ln m$. For every $r>0$, the inequality $\ln x \leq x-1$ gives
    \begin{equation*}
        \ln m = \ln\!\left(\frac{m}{r}\right)+\ln r \leq \frac{m}{r}+\ln r-1.
    \end{equation*}
    Hence,
    \begin{equation*}
        \left(1-\frac{B}{r}\right)m \leq A+B(\ln r-1).
    \end{equation*}
    Choosing $r = 2B = 8/a^2$ yields
    \begin{equation*}
        \frac{m}{2} \leq A+B\left(\ln\!\left(\frac{8}{a^2}\right) - 1 \right).
    \end{equation*}
    Therefore,
    \begin{align*}
        m &\leq 2A + 2B\left(\ln\!\left(\frac{8}{a^2}\right) - 1\right) \\
        &= \frac{4}{a^2}\left[\ln\!\left(\frac{2\pi^2 N}{3\delta}\right) + 2\ln\!\left(\frac{8}{a^2}\right) - 2\right] \\
        &= \frac{4}{a^2} \ln\!\left(\frac{128\pi^2 N}{3e^2\delta a^4} \right) \\
        &\leq \frac{4}{a^2} \ln\!\left(\frac{57N}{\delta a^4}\right),
    \end{align*}
    where the final inequality follows from $128\pi^2/(3e^2)<57$.
\end{proof}

\begin{lemma}\label{lem:radius-monotone}
For every real $x \geq 1$, let
\begin{equation}
    \beta(x) = \sqrt{\frac{1}{2x}\ln\left(\frac{2\pi^2Nx^2}{3\delta}\right)}.   
\end{equation}
Then $\beta$ is decreasing on $[1, \infty)$. Consequently, $\beta(m) \leq \beta(x)$ for every $m \geq x \geq 1$.
\end{lemma}
\begin{proof}
    Let $C = 2\pi^2 N/(3\delta)$ and $h(x) = \beta(x)^2 = \ln(Cx^2)/(2x) \geq 0$. Since $\mathcal Q \subseteq [N]$ contains at least two providers, $N \geq 2$. Since moreover $\delta < 1$, for every $x \geq 1$,
    \begin{equation*}
        Cx^2 \geq C > \frac{4\pi^2}{3} > e^2.
    \end{equation*}
    Differentiating gives
    \begin{equation*}
        h'(x) = \frac{2 - \ln(Cx^2)}{2x^2} < 0,
    \end{equation*}
    so $h$ is decreasing on $[1, \infty)$, and so is $\beta = \sqrt{h}$.
\end{proof}

We establish a bound on sums of confidence radii along the sequence of selected providers. This counting argument is independent of whether a step is exploratory or non-exploratory and is used in the proof of Theorem~\ref{thm:quality-guarantee}.

\begin{lemma}[Radius sum]\label{lem:radius-sum}
Fix $T \geq N$ and let $(a_t)_{t=N+1}^{T}$ be nonnegative numbers satisfying
\begin{equation*}
    a_t \leq \kappa\,\beta(m_{I_t,t}), \qquad t = N + 1,\ldots,T,
\end{equation*}
for some $\kappa \geq 0$. Then,
\begin{equation}\label{eq:radius-sum}
    \sum_{t=N+1}^{T} a_t \leq \kappa \sqrt{2N(T-N)\ln\!\left(\frac{2\pi^2 N T^2}{3\delta}\right)}.
\end{equation}
\end{lemma}
\begin{proof}
    For each provider $i \in [N]$, let
    \begin{equation*}
        D_i = m_{i,T+1} - 1
    \end{equation*}
    denote the number of times provider $i$ is selected after initialization up to and including time step $T$. Since exactly one provider is selected at every post-initialization step,
    \begin{equation*}
        \sum_{i=1}^{N} D_i = T-N.
    \end{equation*}
    Therefore,
    \begin{align*}
        \sum_{t=N+1}^{T} a_t &\leq \kappa\sum_{i=1}^{N}\sum_{\ell=1}^{D_i} \beta(\ell) \\
        &= \kappa \sum_{i=1}^{N} \sum_{\ell=1}^{D_i}\sqrt{\frac{1}{2\ell}\ln\!\left(\frac{2\pi^2 N \ell^2}{3\delta}\right)}.
    \end{align*}
    Since $\ell \leq D_i \leq T$,
    \begin{equation*}
        \beta(\ell) \leq \sqrt{\frac{1}{2\ell}\ln\!\left(\frac{2\pi^2 N T^2}{3\delta}\right)}.
    \end{equation*}
    Hence,
    \begin{align*}
        \sum_{t=N+1}^{T} a_t &\leq \kappa \sqrt{\frac{1}{2}\ln\!\left(\frac{2\pi^2 N T^2}{3\delta}\right)} \sum_{i=1}^{N} \sum_{\ell=1}^{D_i} \frac{1}{\sqrt{\ell}} \\
        &\leq \kappa \sqrt{2 \ln\!\left(\frac{2\pi^2 N T^2}{3\delta} \right)} \sum_{i=1}^{N}\sqrt{D_i},
    \end{align*}
    where we used $\sum_{\ell=1}^{D}\ell^{-1/2}\leq 2\sqrt{D}$. Finally, using the Cauchy--Schwarz inequality,
    \begin{equation*}
        \sum_{i=1}^{N}\sqrt{D_i} \leq \sqrt{ N\sum_{i=1}^{N}D_i} = \sqrt{N(T-N)}.
    \end{equation*}
    Combining the preceding inequalities gives
    \begin{equation*}
        \sum_{t=N+1}^{T} a_t \leq \kappa \sqrt{2N(T - N)\ln\!\left(\frac{2\pi^2 N T^2}{3\delta}\right)},
    \end{equation*}
    as claimed.
\end{proof}

\subsection{Eligibility and elimination}
\label{app:unqualified-elimination}
\label{app:eligibility-exploration}

We characterize the relationship between the estimated confidence sets $\hat{\mathcal{Q}}_t$, introduced in Eq.~\ref{eq:qualified-set}, and the set $\mathcal{Q}$ of qualified providers.

\begin{lemma}\label{lem:qualified-eligibility}
On $\mathcal E$, every qualified provider remains in the estimated-qualified set at every post-initialization step:
\begin{equation}
    \mathcal Q \subseteq \hat{\mathcal Q}_t, \qquad t>N.
\end{equation}
In particular, $\hat{\mathcal Q}_t\neq\emptyset$ for every $t>N$.
\end{lemma}
\begin{proof}
Fix $i\in\mathcal Q$ and $t > N$. Since $i$ is qualified,
\begin{equation*}
    q_i \geq q_{\min}.
\end{equation*}
On $\mathcal E$, Lemma~\ref{lem:good-event} gives, for any provider $i$ and time step $t>N$,
\begin{equation*}
    \hat q_{i,t} \geq q_i - \beta(m_{i,t}).
\end{equation*}
Therefore,
\begin{equation*}
    \hat q_{i,t} + \beta(m_{i,t}) \geq q_i \geq q_{\min},
\end{equation*}
 and hence $i\in\hat{\mathcal{Q}}_t$ by the definition of the estimated-qualified set. Since this holds for every $i \in \mathcal{Q}$, we obtain $\mathcal{Q} \subseteq \hat{\mathcal{Q}}_t$. Finally, $\mathcal{Q}$ contains at least two providers by assumption, so $\hat{\mathcal{Q}}_t$ is nonempty under the event $\mathcal{E}$.
\end{proof}

We bound the number of times an unqualified provider can be selected on the platform, including both non-exploration and exploration steps.

\begin{lemma}\label{lem:unqualified-elimination}
    On $\mathcal E$, every unqualified provider $i \notin \mathcal Q$, with quality gap
    \begin{equation*}
        \varepsilon_i \coloneqq q_{\min} - q_i > 0,
    \end{equation*}
    is selected at most
    \begin{equation}
        M_i \coloneqq \left\lceil \frac{4}{\varepsilon_i^2} \ln\!\left(\frac{57N}{\delta\varepsilon_i^4}\right) \right\rceil
        \label{eq:unqualified-selection-threshold}
    \end{equation}
    times in post-initialization steps.
\end{lemma}
\begin{proof}
    Fix an unqualified provider $i \notin \mathcal Q$ and a post-initialization step $t>N$ in which $I_t=i$. Whether $t$ is a non-exploration or exploration step, provider $i$ can be selected only if $i\in\hat{\mathcal Q}_t$. Hence,
    \begin{equation*}
        \hat q_{i,t} + \beta(m_{i,t}) \geq q_{\min}.
    \end{equation*}
    On $\mathcal E$, Lemma~\ref{lem:good-event} gives
    \begin{equation*}
        \hat q_{i,t} \leq q_i + \beta(m_{i,t}).
    \end{equation*}
    Combining the two inequalities yields
    \begin{equation*}
        q_{\min} \leq q_i+2\beta(m_{i,t}),
    \end{equation*}
    and therefore
    \begin{equation}
        \varepsilon_i \leq 2\beta(m_{i,t}). 
        \label{eq:unqualified-radius-condition}
    \end{equation}
    Applying Lemma~\ref{lem:radius-inversion} with $a = \varepsilon_i$ gives
    \begin{equation*}
        m_{i,t} \leq \frac{4}{\varepsilon_i^2} \ln\!\left( \frac{57N}{\delta\varepsilon_i^4} \right)\leq M_i.
    \end{equation*}
    Since this holds whenever provider $i$ is selected in a post-initialization step, provider $i$ is selected at most $M_i$ times after initialization. Consequently, once $m_{i,t} > M_i$, provider $i$ cannot be selected at any step $t > N$.
\end{proof}

\begin{lemma}\label{lem:unqualified-exclusion}
    On $\mathcal{E}$, for every unqualified provider $i \notin \mathcal{Q}$ and every step $t > N$, if $i \in \hat{\mathcal{Q}}_t$, then
    \begin{equation*}
        m_{i, t} \leq M_i.
    \end{equation*}
\end{lemma}
\begin{proof}
    Fix $i \notin \mathcal{Q}$ and a step $t > N$ with $i \in \hat{\mathcal{Q}}_t$. By the definition of estimated-qualified set and Lemma~\ref{lem:good-event},
    \begin{equation*}
        q_{\min} \leq \hat q_{i, t} + \beta(m_{i, t}) \leq q_i + 2\beta(m_{i, t}),
    \end{equation*}
    and therefore $\varepsilon_i \leq 2\beta(m_{i, t})$. Applying Lemma~\ref{lem:radius-inversion} with $a = \varepsilon_i$ gives
    \begin{equation*}
        m_{i, t} \leq \frac{4}{\varepsilon_i^2}\ln\!\left(\frac{57N}{\delta\varepsilon_i^4}\right) \leq M_i,
    \end{equation*}
    which proves the claim.
\end{proof}

\begin{lemma}\label{lem:ordinary-qualified-set}
    For every non-exploration step $t > N$,
    \begin{equation}\label{eq:ordinary-counts}
        m_{j,t} \geq g(t) \qquad \text{for every } j \in \hat{\mathcal Q}_t.
    \end{equation}
    If, moreover, $\mathcal E$ holds and $t > T_0$, then $\hat{\mathcal Q}_t = \mathcal Q$.
\end{lemma}
\begin{proof}
    Fix an ordinary step $t > N$, so that $\mathcal S_t = \emptyset$. By the definition of $\mathcal S_t$, no provider $j \in \hat{\mathcal Q}_t$ satisfies $m_{j,t} < g(t)$, which proves \eqref{eq:ordinary-counts}.

    Now suppose that $\mathcal E$ holds and $t > T_0$, so that $g(t) > \bar M$. By Lemma~\ref{lem:qualified-eligibility}, $\mathcal Q \subseteq \hat{\mathcal Q}_t$. Conversely, suppose that some unqualified provider $i \notin \mathcal Q$ belongs to $\hat{\mathcal Q}_t$. Then \eqref{eq:ordinary-counts} gives
    \begin{equation*}
        m_{i,t} \geq g(t) > \bar M \geq M_i,
    \end{equation*}
    whereas Lemma~\ref{lem:unqualified-exclusion} gives $m_{i,t} \leq M_i$, a contradiction. Hence, $\hat{\mathcal Q}_t \subseteq \mathcal Q$, and therefore $\hat{\mathcal Q}_t = \mathcal Q$.
\end{proof}

\subsection{Exploration and selection counts}
\label{app:ordinary-selection}
\label{app:selection-identification}

\begin{lemma}\label{lem:forced-exploration-bound}
Suppose that $g \colon \mathbb{N} \to (0,\infty)$ is non-decreasing, and fix a total number of user queries $T>N$. If provider $i$ is selected in an exploration step $t\leq T$, then
\begin{equation} \label{eq:forced-selection-bound}
    m_{i,t} < g(t) \leq g(T).
\end{equation}
Consequently, provider $i$ can be selected in an exploration step at most
\begin{equation} \label{eq:provider-forced-bound}
    \left\lceil g(T) \right\rceil-1
\end{equation}
times, and the total number of post-initialization exploration steps up to $T$ queries is at most
\begin{equation} \label{eq:total-forced-bound}
    N\left(\left\lceil g(T) \right\rceil-1\right).
\end{equation}
\end{lemma}
\begin{proof}
    If provider $i$ is selected in an exploration step $t$, then $i\in\mathcal S_t$. By definition of $\mathcal S_t$,
    \begin{equation*}
        m_{i,t}<g(t).
    \end{equation*}
    Since $g$ is non-decreasing and $t\leq T$,
    \begin{equation*}
        m_{i,t} < g(t) \leq g(T),
    \end{equation*}
    which proves \eqref{eq:forced-selection-bound}.

    Now suppose that $t$ is an exploration step in which provider $i$ is selected. Since $m_{i,t}$ is a positive integer after initialization, \eqref{eq:forced-selection-bound} implies
    \begin{equation*}
        m_{i,t} \leq \left\lceil g(T) \right\rceil - 1.
    \end{equation*}
    Thus, provider $i$ can therefore be selected in an exploration step at most $\lceil g(T) \rceil - 1$ when the total number of user queries is $T$. Summing this bound over all $N$ providers gives \eqref{eq:total-forced-bound}.
\end{proof}

We establish properties of the allocation rule on non-exploration steps, that is, when $\mathcal{S}_t= \emptyset$ and the selection of provider $I_t$ depends on the bids $\mathbf{b}_t$. 

\begin{lemma}\label{lem:adjusted-bid-bounds}
    Suppose that all providers use their equilibrium strategy $\boldsymbol{\sigma}^*$. On $\mathcal{E}$, for every provider $i \in [N]$ and every step $t > N$,
    \begin{equation}\label{eq:adjusted-bid-bounds}
        \bar c_i - 2c_{\max}\beta(m_{i,t}) \leq b_{i,t} - c_{\max}\beta(m_{i, t}) \leq \bar c_i.
    \end{equation}
\end{lemma}
\begin{proof}
    Fix a provider $i \in [N]$ and a step $t > N$. Under the equilibrium strategy $\sigma^*$,
    \begin{equation*}
        b_{i, t} = \hat c_{i, t}.
    \end{equation*}
    Thus, on $\mathcal{E}$,
    \begin{equation}
        \bar c_i - c_{\max}\beta(m_{i, t}) \leq \hat c_{i, t} \leq \bar c_i + c_{\max}\beta(m_{i, t}).
    \end{equation}
    Substituting $b_{i, t} = \hat c_{i, t}$ and subtracting $c_{\max}\beta(m_{i, t})$ yields
    \begin{equation*}
        \bar c_i - 2c_{\max}\beta(m_{i,t}) \leq b_{i,t} - c_{\max}\beta(m_{i, t}) \leq \bar c_i.
    \end{equation*}
    which proves the claim.
\end{proof}

\begin{lemma} \label{lem:ordinary-selection-gap}
    Suppose that all providers use their equilibrium strategy $\boldsymbol{\sigma}^*$. On $\mathcal E$, if a provider $i \in \mathcal{Q} \setminus\{i^*\}$ is selected at a non-exploration step $t > N$, then
    \begin{equation} \label{eq:ordinary-selection-gap}
        \Delta_i \leq 2c_{\max}\beta(m_{i, t}),
    \end{equation}
    where $\Delta_i = \bar c_i - \bar c_{i^*} > 0$.
\end{lemma}
\begin{proof}
    Fix a non-exploration step $t > N$ in which $i \in \mathcal{Q} \setminus \{i^*\}$ is selected. Under the event $\mathcal{E}$, Lemma~\ref{lem:qualified-eligibility} implies that $i^* \in \hat{\mathcal{Q}}_t$. Since $t$ is a non-exploration step, the allocation rule therefore implies
    \begin{equation*}
        b_{i, t} - c_{\max}\beta(m_{i, t}) \leq b_{i^*, t} - c_{\max}\beta(m_{i^*, t}).
    \end{equation*}
    Applying Lemma~\ref{lem:adjusted-bid-bounds} to both providers gives
    \begin{equation*}
        \bar c_i-2c_{\max}\beta(m_{i,t}) \leq b_{i,t} - c_{\max}\beta(m_{i,t}) \leq b_{i^*,t}-c_{\max}\beta(m_{i^*,t}) \leq \bar c_{i^*}.
    \end{equation*}
    Hence,
    \begin{equation*}
        \bar c_i-\bar c_{i^*} \leq 2c_{\max}\beta(m_{i,t}).
    \end{equation*}
    Since $\Delta_i=\bar c_i-\bar c_{i^*}$, this proves \eqref{eq:ordinary-selection-gap}.
\end{proof}

\begin{lemma}\label{lem:ordinary-selection-threshold}
    Suppose that all providers use their equilibrium strategy $\boldsymbol{\sigma}^*$. On $\mathcal E$, if a provider $i \in \mathcal Q\setminus\{i^*\}$ is selected at a non-exploration step  $t > N$, then
    \begin{equation} \label{eq:ordinary-selection-threshold}
        m_{i,t} \leq L_i \coloneqq \left\lceil \frac{4c_{\max}^2}{\Delta_i^2} \ln\!\left(\frac{57N c_{\max}^4} {\delta\Delta_i^4} \right) \right\rceil.
    \end{equation}
    Consequently, once $m_{i,t} > L_i$, provider $i$ cannot be selected at a non-exploration step 
\end{lemma}
\begin{proof}
    By Lemma~\ref{lem:ordinary-selection-gap}, if provider $i \in \mathcal Q \setminus\{i^*\}$ is selected at a non-exploration step  $t > N$, then
    \begin{equation*}
        \Delta_i \leq 2c_{\max}\beta(m_{i,t}).
    \end{equation*}
    Dividing by $c_{\max}$ gives
    \begin{equation*}
        \frac{\Delta_i}{c_{\max}} \leq 2 \beta(m_{i,t}).
    \end{equation*}
    Applying Lemma~\ref{lem:radius-inversion} with
    \begin{equation*}
        a = \frac{\Delta_i}{c_{\max}}
    \end{equation*}
    therefore yields
    \begin{equation*}
        m_{i,t} \leq \frac{4c_{\max}^2}{\Delta_i^2} \ln\!\left(\frac{57N c_{\max}^4}{\delta\Delta_i^4}\right) \leq L_i.
    \end{equation*}
    Thus, if $m_{i,t} > L_i$, provider $i$ cannot be selected by the allocation rule at a non-exploration step.
\end{proof}

We bound the total number of times a qualified suboptimal provider can be selected.

\begin{lemma} \label{lem:qualified-suboptimal-selection}
Suppose that all providers use their equilibrium strategy $\boldsymbol{\sigma}^*$. On $\mathcal E$, for every provider $i\in\mathcal Q\setminus\{i^*\}$, the number of post-initialization steps $t\leq T$ in which $I_t=i$ is at most
\begin{equation}
    \max\left\{L_i,\, \left\lceil g(T)\right\rceil - 1 \right\},
    \label{eq:qualified-suboptimal-selection}
\end{equation}
where $L_i$ is defined in Lemma~\ref{lem:ordinary-selection-threshold}.
\end{lemma}
\begin{proof}
    Fix $T>N$ and a provider $i\in\mathcal Q\setminus\{i^*\}$, and consider any post-initialization step $t\leq T$ in which $I_t=i$. If $t$ is a non-exploration step, Lemma~\ref{lem:ordinary-selection-threshold} gives
    \begin{equation*}
        m_{i,t}\leq L_i.
    \end{equation*}
    If instead $t$ is an exploration step, Lemma~\ref{lem:forced-exploration-bound} gives
    \begin{equation*}
        m_{i,t}<g(t)\leq g(T).
    \end{equation*}
    Since $m_{i,t}$ is an integer,
    \begin{equation*}
        m_{i,t}\leq \left\lceil g(T)\right\rceil-1.
    \end{equation*}
    Thus, in either case,
    \begin{equation*}
        m_{i,t} \leq \max\left\{L_i,\, \left\lceil g(T)\right\rceil - 1 \right\}.
    \end{equation*}
    Since provider $i$ is selected once during initialization and $m_{i,t}$ increases by one at each of its selections, its $\ell$-th post-initialization selection occurs at a step $t$ with $m_{i,t} = \ell$. Hence, provider $i$ is selected at most $\max\{L_i, \lceil g(T)\rceil - 1\}$ times in post-initialization steps $t \leq T$.
\end{proof}

\subsection{Payment bounds}\label{app:payment-bounds}
We establish per-round bounds on the payment $\pi_t$ made by the user.

\begin{lemma}\label{lem:payment-range}
    For every step $t \geq 1$, $0 \leq \pi_t \leq c_{\max}$. Consequently,
    \begin{equation}\label{eq:payment-range}
        \bigl| \pi_t - \bar c_{(2)} \bigr| \leq c_{\max}, \qquad t \geq 1.
    \end{equation}
\end{lemma}
\begin{proof}
    By the payment rule, $\pi_t \leq c_{\max}$, with equality during initialization, on exploration steps, and whenever $\hat{\mathcal Q}_t = \emptyset$. On an ordinary round, Proposition~\ref{prop:selection-threshold} implies that the selected provider bids $b_{I_t,t} \leq P_{I_t,t}$, and therefore
    \begin{equation*}
        \pi_t = \min\{P_{I_t,t}, c_{\max}\} \geq \min\{b_{I_t,t}, c_{\max}\} = b_{I_t,t} \geq 0.
    \end{equation*}
    Since also $\bar c_{(2)} \in [0, c_{\max}]$, \eqref{eq:payment-range} follows.
\end{proof}

\begin{lemma}\label{lem:ordinary-payment}
    On $\mathcal E$, suppose that all providers use their equilibrium strategy $\boldsymbol{\sigma}^*$. For every non-exploration step $t>N$ with $I_t=i$,
    \begin{equation}
        \pi_t \leq \min\left\{c_{\max}, \bar c_{(2)} + c_{\max}\beta(m_{i,t}) \right\}.
    \end{equation}
    Consequently,
    \begin{equation}
        \bigl(\pi_t-\bar c_{(2)}\bigr)^+ \leq c_{\max}\beta(m_{i,t}).
    \end{equation}
\end{lemma}
\begin{proof}
    Fix a non-exploration step $t>N$ with $I_t=i$. Then, the payment rule (Eq.~\ref{eq:payment}) gives
    \begin{equation*}
        \pi_t = \min\left\{c_{\max}\beta(m_{i,t}) + \min_{j\in\hat{\mathcal Q}_t\setminus\{i\}} \left(b_{j,t}-c_{\max}\beta(m_{j,t})\right),c_{\max}\right\}.
    \end{equation*}
    If $i=i^*$, let $j$ be a second-cheapest qualified provider, so that $\bar c_j=\bar c_{(2)}$. If $i\neq i^*$, let $j=i^*$, in which case $\bar c_j=\bar c_{i^*}\leq\bar c_{(2)}$. In either case, $j\in\mathcal Q$ and $j\neq i$ and $j\in\hat{\mathcal Q}_t$ on the event $\mathcal{E}$. Moreover, using Lemma~\ref{lem:adjusted-bid-bounds},
    \begin{equation*}
        b_{j,t}-c_{\max}\beta(m_{j,t}) \leq \bar c_j \leq \bar c_{(2)}.
    \end{equation*}
    Therefore,
    \begin{equation*}
        \min_{j'\in\hat{\mathcal Q}_t\setminus\{i\}} \left(b_{j',t}-c_{\max}\beta(m_{j',t})\right) \leq \bar c_{(2)},
    \end{equation*}
    and hence
    \begin{equation*}
        \pi_t \leq \min\left\{c_{\max},\bar c_{(2)}+c_{\max}\beta(m_{i,t})\right\}.
    \end{equation*}
    The second claim follows immediately.
\end{proof}

\begin{lemma}\label{lem:two-sided-payment}
    Suppose that all providers use their equilibrium strategy $\boldsymbol{\sigma}^*$. On $\mathcal E$, for every non-exploration step $t > T_0$,
    \begin{equation}\label{eq:two-sided-payment-app}
        \bar c_{(2)} - 3c_{\max}\beta(g(t)) \leq \pi_t \leq \bar c_{(2)} + c_{\max}\beta(g(t)).
    \end{equation}
\end{lemma}
\begin{proof}
    Fix a non-exploration step $t > T_0$ and let $i = I_t$. By Lemma~\ref{lem:ordinary-qualified-set}, $\hat{\mathcal Q}_t = \mathcal Q$ and $m_{j,t} \geq g(t)$ for every $j \in \mathcal Q$. Moreover, $g(t) > \bar M \geq 1$, so Lemma~\ref{lem:radius-monotone} gives
    \begin{equation}\label{eq:ordinary-radius-monotone}
        \beta(m_{j,t}) \leq \beta(g(t)), \qquad j \in \mathcal Q.
    \end{equation}

    We start with the upper bound. Lemma~\ref{lem:ordinary-payment} and Eq.~\ref{eq:ordinary-radius-monotone} give
    \begin{equation*}
        \pi_t \leq \bar c_{(2)} + c_{\max}\beta(m_{i,t}) \leq \bar c_{(2)} + c_{\max}\beta(g(t)).
    \end{equation*}

    We next turn to the lower bound. Since $\hat{\mathcal Q}_t = \mathcal Q$, the selected provider satisfies $i \in \mathcal Q$, and the payment rule gives
    \begin{equation*}
        \pi_t = \min\left\{ c_{\max}\beta(m_{i,t}) + \min_{j \in \mathcal Q \setminus \{i\}} \bigl( b_{j,t} - c_{\max}\beta(m_{j,t}) \bigr),\; c_{\max} \right\}.
    \end{equation*}
    For every $j \in \mathcal Q \setminus \{i\}$, Lemma~\ref{lem:adjusted-bid-bounds} and Eq.~\ref{eq:ordinary-radius-monotone} give
    \begin{equation}\label{eq:competitor-adjusted-bid}
        b_{j,t} - c_{\max}\beta(m_{j,t}) \geq \bar c_j - 2c_{\max}\beta(g(t)).
    \end{equation}
    If $i = i^*$, then $\bar c_j \geq \bar c_{(2)}$ for every $j \in \mathcal Q \setminus \{i^*\}$, and Eq.~\ref{eq:competitor-adjusted-bid} yields
    \begin{equation*}
        c_{\max}\beta(m_{i^*,t}) + \min_{j \in \mathcal Q \setminus \{i^*\}} \bigl( b_{j,t} - c_{\max}\beta(m_{j,t}) \bigr) \geq \bar c_{(2)} - 2c_{\max}\beta(g(t)) \geq \bar c_{(2)} - 3c_{\max}\beta(g(t)).
    \end{equation*}
    If instead $i \neq i^*$, then $i^* \in \mathcal Q \setminus \{i\}$ and $\bar c_j \geq \bar c_{i^*}$ for every $j \in \mathcal Q \setminus \{i\}$, so Eq.~\ref{eq:competitor-adjusted-bid} yields
    \begin{equation*}
        \min_{j \in \mathcal Q \setminus \{i\}} \bigl( b_{j,t} - c_{\max}\beta(m_{j,t}) \bigr) \geq \bar c_{i^*} - 2c_{\max}\beta(g(t)).
    \end{equation*}
    Further, Lemma~\ref{lem:ordinary-selection-gap} gives $\bar c_i - \bar c_{i^*} \leq 2c_{\max}\beta(m_{i,t})$, and since $\bar c_i \geq \bar c_{(2)}$,
    \begin{equation*}
        \bar c_{i^*} \geq \bar c_{(2)} - 2c_{\max}\beta(m_{i,t}).
    \end{equation*}
    Therefore,
    \begin{align*}
        c_{\max}\beta(m_{i,t}) + \min_{j \in \mathcal Q \setminus \{i\}} \bigl( b_{j,t} - c_{\max}\beta(m_{j,t}) \bigr) &\geq \bar c_{(2)} - c_{\max}\beta(m_{i,t}) - 2c_{\max}\beta(g(t)) \\
        &\geq \bar c_{(2)} - 3c_{\max}\beta(g(t)),
    \end{align*}
    where the last inequality uses Eq.~\ref{eq:ordinary-radius-monotone}. In both cases, the first argument of the minimum defining $\pi_t$ is at least $\bar c_{(2)} - 3c_{\max}\beta(g(t))$, and, since $\bar c_{(2)} \leq c_{\max}$, so is the second. Hence,
    \begin{equation*}
        \pi_t \geq \bar c_{(2)} - 3c_{\max}\beta(g(t)),
    \end{equation*}
    which completes the proof.
\end{proof}
\clearpage

\section{Bayesian cost estimation}
\label{app:bayesian}
The main model considers providers that bid the empirical mean $\hat c_{i,t}$ of their observed generation costs. In this section, we show that the robust optimality of bidding one's current cost estimate (Theorem~\ref{thm:robust-dominant}) holds verbatim for any estimator, and that the user-facing guarantees of Section~\ref{sec:user} extend to Bayesian estimators whose prior is eventually dominated by the data, with the cost-estimation radius widened by a factor $1+\gamma$ and unchanged asymptotic rates. We identify the properties of the bids used in the proofs in Appendices~\ref{app:proofs}--\ref{app:auxiliary-lemmas} and verify them for Bayesian estimates.

\paragraph{Estimation model.}
Each provider $i \in [N]$ forms its cost estimate through a map $e_i$ from histories to $[0, c_{\max}]$, and we write $e_{i,t} = e_i(H_{i,t})$ for its estimate before step $t$. Since $H_{i,t}$ consists of the first $m_{i,t}$ observations of provider $i$, we also write $e_{i,m}$ for the estimate after $m$ selections, so that $e_{i,t} = e_{i,m_{i,t}}$, in the same way that $\hat q_{i,t} = \hat q_{i,m_{i,t}}$ and $\hat c_{i,t} = \hat c_{i,m_{i,t}}$. For a Bayesian provider, $e_{i,m}$ is the posterior mean of the cost of serving its next query given a prior and its $m$ observed costs; it lies in $[0, c_{\max}]$ because the cost does. We require that the prior does not pull the estimate away from the empirical mean by more than a fixed multiple of the confidence radius.

\begin{assumption}
\label{asmp:beliefs}
    There exists a constant $\gamma \geq 0$ such that, for every provider $i \in [N]$ and every $m \geq 1$,
    \begin{equation}
        \bigl| e_{i,m} - \hat c_{i,m} \bigr| \leq \gamma\, c_{\max}\, \beta(m).
        \label{eq:belief-slack}
    \end{equation}
\end{assumption}
The constant $\gamma$ is a property of the estimation rule, and $\gamma = 0$ recovers the empirical mean.

\begin{lemma}
\label{lem:bayesian-cost-concentration}
    Under Assumption~\ref{asmp:beliefs}, on the event $\mathcal E$ of Lemma~\ref{lem:good-event}, for every provider $i \in [N]$ and every $m \geq 1$,
    \begin{equation}
        \bigl| e_{i,m} - \bar c_i \bigr| \leq (1+\gamma)\, c_{\max}\, \beta(m).
        \label{eq:bayesian-cost-concentration}
    \end{equation}
\end{lemma}
\begin{proof}
    On $\mathcal E$, $|\hat c_{i,m} - \bar c_i| \leq c_{\max}\beta(m)$ by Eq.~\ref{eq:good-event}. By the triangle inequality and Eq.~\ref{eq:belief-slack},
    \begin{equation*}
        \bigl| e_{i,m} - \bar c_i \bigr| \leq \bigl| e_{i,m} - \hat c_{i,m} \bigr| + \bigl| \hat c_{i,m} - \bar c_i \bigr| \leq \gamma c_{\max}\beta(m) + c_{\max}\beta(m).
        \qedhere
    \end{equation*}
\end{proof}

\paragraph{Bayesian variant of the mechanism.}
The platform uses a known upper bound $\gamma$ satisfying Assumption~\ref{asmp:beliefs}---any larger valid bound may be used, since Eq.~\ref{eq:belief-slack} is monotone in $\gamma$---and replaces the cost-estimation radius $c_{\max}\beta(m_{i,t})$ by $(1+\gamma)c_{\max}\beta(m_{i,t})$ in the adjusted bids of the allocation rule (Eq.~\ref{eq:allocation}) and in the critical payment (Eq.~\ref{eq:critical-payment}), \ie,
\begin{equation}
\begin{aligned}
    I_t &= \argmin_{i \in \hat{\Qcal}_t} \bigl( b_{i,t} - (1+\gamma)c_{\max}\beta(m_{i,t}) \bigr),
    && \text{if } \Scal_t = \emptyset, \\
    P_{i,t} &= (1+\gamma)c_{\max}\beta(m_{i,t}) + \min_{j \in \hat{\Qcal}_t \setminus \{i\}} \bigl( b_{j,t} - (1+\gamma)c_{\max}\beta(m_{j,t}) \bigr).
\end{aligned}
\label{eq:bayesian-mechanism}
\end{equation}
The quality estimates $\hat q_{i,t}$, the estimated-qualified set $\hat{\Qcal}_t$, the exploration set $\Scal_t$, the schedule $g$, the initialization and exploration payments $c_{\max}$, and the payment cap in Eq.~\ref{eq:payment} are unchanged, and $\gamma = 0$ recovers the mechanism of Section~\ref{sec:model}. Providers bid their current estimate, $b_{i,t} = e_{i,t}$, and the utility in Eq.~\ref{eq:utilities} is evaluated with $e_{i,s}$ in place of $\hat c_{i,s}$.

The widening of the radius is necessary: if the platform kept $c_{\max}\beta$, the adjusted bid of $i^*$ would only satisfy $b_{i^*,t} - c_{\max}\beta(m_{i^*,t}) \leq \bar c_{i^*} + \gamma c_{\max}\beta(m_{i^*,t})$, and the optimism property $b_{i,t} - c_{\max}\beta(m_{i,t}) \leq \bar c_i$ on which Lemma~\ref{lem:ordinary-selection-gap} rests would fail.

\begin{proposition}\label{prop:bayesian-substitution} 
    Consider the Bayesian variant of the mechanism in Eq.~\ref{eq:bayesian-mechanism} with all providers bidding $b_{i,t} = e_{i,t}$, and let $L_i^{(\gamma)} = M\bigl(\Delta_i/((1+\gamma)c_{\max})\bigr)$, where $M(\cdot)$ is defined in Theorem~\ref{thm:sample-complexity}.
    \begin{enumerate}[label=\roman*)]
        \item Without Assumption~\ref{asmp:beliefs}: Lemmas~\ref{lem:good-event}, \ref{lem:radius-inversion}, \ref{lem:qualified-eligibility}, \ref{lem:forced-exploration-bound}, \ref{lem:unqualified-elimination}, \ref{lem:radius-sum}, \ref{lem:radius-monotone}, \ref{lem:unqualified-exclusion}, ~\ref{lem:ordinary-qualified-set}, and ~\ref{lem:payment-range}, Theorem~\ref{thm:quality-guarantee}, and the quantities $M_i$, $\bar M$, and $T_0$ hold exactly as stated, under arbitrary bids; and Proposition~\ref{prop:selection-threshold}, Lemma~\ref{lem:one-step}, Theorem~\ref{thm:robust-dominant}, and the robust-equilibrium property that follows it hold with $e_{i,t}$ in place of $\hat c_{i,t}$, for any estimator $e_i$ with values in $[0, c_{\max}]$.

        \item Under Assumption~\ref{asmp:beliefs}: Lemmas~\ref{lem:adjusted-bid-bounds}, \ref{lem:ordinary-selection-gap}, \ref{lem:ordinary-selection-threshold}, \ref{lem:qualified-suboptimal-selection}, \ref{lem:ordinary-payment}, and~\ref{lem:two-sided-payment}, Theorems~\ref{thm:sample-complexity}, \ref{thm:generation-regret}, \ref{thm:excess-payment-general}, \ref{thm:payment-convergence}, and~\ref{thm:regret}, and Corollaries~\ref{cor:vanishing-nonoptimal-selections}, \ref{cor:eventual-ordinary-optimality}, \ref{cor:excess-payment-polynomial}, and~\ref{cor:optimal-alpha} hold after replacing each cost-estimation radius $c_{\max}\beta(\cdot)$ by $(1+\gamma)c_{\max}\beta(\cdot)$ in their statements, with the corresponding replacement $L_i \to L_i^{(\gamma)}$.
    \end{enumerate}
\end{proposition}
\begin{proof}
    \emph{Claim i).} The results in the first group concern the quality estimates, the estimated-qualified set, the exploration rule, or the confidence radius alone; their proofs do not involve the bids or the cost estimates, all of which the Bayesian variant leaves unchanged, so they hold as stated under arbitrary bids. For the second group, the proofs of Proposition~\ref{prop:selection-threshold} and Lemma~\ref{lem:one-step} use only that the allocation rule selects a minimizer of $b_{j,t} - \rho_{j,t}$ over $\hat{\Qcal}_t$ and pays $\min\{P_{i,t}, c_{\max}\}$ with $P_{i,t} = \rho_{i,t} + \min_{j \in \hat{\Qcal}_t \setminus \{i\}}(b_{j,t} - \rho_{j,t})$, for arbitrary radii $\rho_{j,t} \geq 0$ determined by the history, together with the fact that the bid compared to $P_{i,t}$ and the estimate entering the net payment $G_{i,s}$ coincide and lie in $[0, c_{\max}]$. Both hold for $\rho_{j,t} = (1+\gamma)c_{\max}\beta(m_{j,t})$ and $e_{i,t} \in [0, c_{\max}]$. The proof of Theorem~\ref{thm:robust-dominant} uses only Lemma~\ref{lem:one-step} and that the truthful strategy bids the estimate entering Eq.~\ref{eq:utilities} at every step. None of these proofs uses the event $\mathcal E$ or the form of the estimator.
    
    \emph{Claim ii).} The only property of the bids used in the listed proofs is the adjusted-bid bound of Lemma~\ref{lem:adjusted-bid-bounds}, which in the Bayesian variant reads, on $\mathcal E$ and by Lemma~\ref{lem:bayesian-cost-concentration},
    \begin{equation}
        \bar c_i - 2(1+\gamma)c_{\max}\beta(m_{i,t}) \leq b_{i,t} - (1+\gamma)c_{\max}\beta(m_{i,t}) \leq \bar c_i,
        \label{eq:bayesian-adjusted-bid-bounds}
    \end{equation}
    with the same one-line proof. Downstream of this bound, the radius enters only through the choice $a = \Delta_i/c_{\max}$ in the application of Lemma~\ref{lem:radius-inversion} (proof of Lemma~\ref{lem:ordinary-selection-threshold}), which becomes $a = \Delta_i/((1+\gamma)c_{\max})$ and yields $L_i^{(\gamma)}$, and the choices $\kappa = c_{\max}$ and $\kappa = 2c_{\max}$ in the applications of Lemma~\ref{lem:radius-sum} (proofs of Theorems~\ref{thm:excess-payment-general} and~\ref{thm:generation-regret}), which become $\kappa = (1+\gamma)c_{\max}$ and $\kappa = 2(1+\gamma)c_{\max}$; both lemmas hold for arbitrary $a > 0$ and $\kappa \geq 0$. Every other occurrence of $c_{\max}$ in these proofs---the bid range, the initialization and exploration payments, the payment cap, the bound $|\pi_t - \bar c_{(2)}| \leq c_{\max}$, and the factor $c_{\max} - \bar c_{(2)}$---refers to the range of costs and payments, which is unchanged. Hence each listed proof goes through verbatim under the stated replacement.
\end{proof}

\paragraph{Resulting guarantees.}
Two consequences of Proposition~\ref{prop:bayesian-substitution}ii) are worth recording. First, Theorem~\ref{thm:sample-complexity} holds with $M(\Delta_i/c_{\max})$ replaced by $L_i^{(\gamma)} = M\bigl(\Delta_i/((1+\gamma)c_{\max})\bigr)$: the ordinary-step selection threshold of a qualified suboptimal provider grows with $\gamma$, while the terms for unqualified providers, the exploration term, and the quality guarantee of Theorem~\ref{thm:quality-guarantee} are unchanged.
Second, Theorem~\ref{thm:regret} holds with the coefficient $\frac{6\,c_{\max}}{2-\alpha}$ of its last term replaced by $\frac{6(1+\gamma)\,c_{\max}}{2-\alpha}$, and its first two terms unchanged; likewise, Lemma~\ref{lem:two-sided-payment} gives $\bar c_{(2)} - 3(1+\gamma)c_{\max}\beta(g(t)) \leq \pi_t \leq \bar c_{(2)} + (1+\gamma)c_{\max}\beta(g(t))$ on every non-exploration step $t > T_0$. Hence, for fixed $\gamma$, the average absolute deviation of the payments from $\bar c_{(2)}$ is $O\bigl(T^{-(1-\alpha)}\bigr) + O\bigl((1+\gamma)\,T^{-\alpha/2}\sqrt{\ln T}\bigr)$, the exponent is still maximized at $\alpha = 2/3$ (Corollary~\ref{cor:optimal-alpha}), and $\frac{1}{T}\sum_{t=1}^{T}\pi_t \to \bar c_{(2)}$ as $T \to \infty$. The average excess-payment bound of Corollary~\ref{cor:excess-payment-polynomial} receives the same radius substitution in its last term only.

\paragraph{Provider incentives.}
By Proposition~\ref{prop:bayesian-substitution}i), bidding one's current cost estimate, $\sigma_{i,t}^*(H_{i,t}) = e_{i,t}$, is worst-case optimal in the sense of Theorem~\ref{thm:robust-dominant} for any estimator with values in $[0, c_{\max}]$, in the Bayesian variant of the mechanism as well as in the original one. The robust-optimality result is thus estimator-agnostic: once a provider's utility is evaluated relative to its current estimate, the threshold structure of the allocation rule makes bidding that estimate optimal. Assumption~\ref{asmp:beliefs} plays no role in this result; it is needed only to relate the providers' estimates to their true mean costs, and thereby obtain the user-facing guarantees above.

\clearpage

\section{Robustness to deviations from cost-estimate bidding} \label{app:robustness}
Theorem~\ref{thm:robust-dominant} shows that, for a provider, bidding its current cost estimate is worst-case optimal, but a provider may nevertheless deviate from it. In this section, we quantify how the user-facing guarantees of Section~\ref{sec:user} change when bids may fall below or above the cost estimate by bounded amounts. 

\begin{assumption}
\label{asmp:bounded-deviations}
    There exist constants $d_-, d_+ \geq 0$ such that every provider $i \in [N]$ follows a strategy $\sigma_i$ with
    \begin{equation}
        -d_- \;\leq\; \sigma_{i,t}(H_{i,t}) - \hat c_{i,t} \;\leq\; d_+
        \qquad \text{for every history } H_{i,t},\ t > N.
        \label{eq:bounded-deviations}
    \end{equation}
\end{assumption}
The constant $d_-$ bounds underbidding and $d_+$ overbidding relative to the empirical mean cost; the equilibrium strategy $\boldsymbol{\sigma}^*$ of Eq.~\ref{eq:truthful-strategy} is the case $d_- = d_+ = 0$.

The good-event, quality-estimation, and forced-exploration results of Appendix~\ref{app:auxiliary-lemmas} do not use the providers' bids. In particular, Lemmas~\ref{lem:good-event}, \ref{lem:radius-inversion}, \ref{lem:qualified-eligibility}, \ref{lem:forced-exploration-bound}, \ref{lem:unqualified-elimination}, \ref{lem:radius-sum}, \ref{lem:radius-monotone}, \ref{lem:unqualified-exclusion}, ~\ref{lem:ordinary-qualified-set}, and ~\ref{lem:payment-range}, Theorem~\ref{thm:quality-guarantee}, and the quantities $M_i$, $\bar M$, and $T_0$ hold under Assumption~\ref{asmp:bounded-deviations} exactly as stated; unqualified providers are therefore eliminated, and the quality guarantee holds, unchanged. The remaining results depend on the bids only through Lemma~\ref{lem:adjusted-bid-bounds}, whose counterpart under deviations is the following:

\begin{lemma}
\label{lem:deviation-adjusted-bid-bounds}
    Under Assumption~\ref{asmp:bounded-deviations}, on the event $\mathcal E$, for every provider $i \in [N]$ and step $t > N$,
    \begin{equation}
        \bar c_i - 2c_{\max}\beta(m_{i,t}) - d_- \;\leq\; b_{i,t} - c_{\max}\beta(m_{i,t}) \;\leq\; \bar c_i + d_+.
        \label{eq:deviation-adjusted-bid-bounds}
    \end{equation}
\end{lemma}
\begin{proof}
    On $\mathcal E$, $|\hat c_{i,t} - \bar c_i| \leq c_{\max}\beta(m_{i,t})$ by Eq.~\ref{eq:good-event}; combine with $\hat c_{i,t} - d_- \leq b_{i,t} \leq \hat c_{i,t} + d_+$.
\end{proof}
The adjusted bid is no longer optimistic, since its upper bound exceeds $\bar c_i$ by $d_+$, and its lower bound is lowered by $d_-$. Tracing these two shifts through the proofs of Appendices~\ref{app:proofs}--\ref{app:auxiliary-lemmas} gives the results below. 

\paragraph{Selection and identification.}
Underbidding by a competitor and overbidding by $i^*$ both help the competitor win a non-exploration round, so the two deviations enter the selection analysis.

\begin{proposition}
    \label{prop:deviation-selection}
    Under Assumption~\ref{asmp:bounded-deviations}, on the event $\mathcal E$:
    \begin{enumerate}[label=\roman*)]
        \item On every non-exploration step $t > N$, the selected provider satisfies
        \begin{equation}
            \bar c_{I_t} - \bar c_{i^*} \;\leq\; 2c_{\max}\beta(m_{I_t,t}) + d_- + d_+.
            \label{eq:deviation-selection-gap}
        \end{equation}
        \item Every qualified provider $i \in \Qcal \setminus \{i^*\}$ with $\Delta_i > d_- + d_+$ is selected after initialization at most $\max\{L_i^{(d)}, \lceil g(T) \rceil - 1\}$ times, where
        \begin{equation}
            L_i^{(d)} \;\coloneqq\; M\!\left(\frac{\Delta_i - d_- - d_+}{c_{\max}}\right)
            \label{eq:deviation-selection-threshold}
        \end{equation}
        and $M(\cdot)$ is defined in Theorem~\ref{thm:sample-complexity}.
        \item If $\Delta_i > d_- + d_+$ for every $i \in \Qcal \setminus \{i^*\}$, then Theorem~\ref{thm:sample-complexity} and Corollaries~\ref{cor:vanishing-nonoptimal-selections} and~\ref{cor:eventual-ordinary-optimality} hold with $M(\Delta_i/c_{\max})$ replaced by $L_i^{(d)}$.
    \end{enumerate}
\end{proposition}
\begin{proof}
    \emph{Claim i).} As in the proof of Lemma~\ref{lem:ordinary-selection-gap}: on $\mathcal E$, $i^* \in \hat{\Qcal}_t$ by Lemma~\ref{lem:qualified-eligibility}, so the selected provider $I_t = i$ satisfies $b_{i,t} - c_{\max}\beta(m_{i,t}) \leq b_{i^*,t} - c_{\max}\beta(m_{i^*,t})$. Bounding the left side from below and the right side from above with Eq.~\ref{eq:deviation-adjusted-bid-bounds} gives $\bar c_i - 2c_{\max}\beta(m_{i,t}) - d_- \leq \bar c_{i^*} + d_+$. This holds whether the selected provider is qualified or not, since Eq.~\ref{eq:deviation-adjusted-bid-bounds} holds for every provider.

    \emph{Claim ii).} At every non-exploration selection of $i \in \Qcal \setminus \{i^*\}$, claim i) gives $(\Delta_i - d_- - d_+)/c_{\max} \leq 2\beta(m_{i,t})$, and Lemma~\ref{lem:radius-inversion} with $a = (\Delta_i - d_- - d_+)/c_{\max} > 0$ gives $m_{i,t} \leq L_i^{(d)}$. The rest is the proof of Lemma~\ref{lem:qualified-suboptimal-selection} with $L_i$ replaced by $L_i^{(d)}$.

    \emph{Claim iii).} The proofs of Theorem~\ref{thm:sample-complexity} and its corollaries use the bids only through Lemma~\ref{lem:qualified-suboptimal-selection}, which claim ii) replaces.
\end{proof}
The condition $\Delta_i > d_- + d_+$ cannot be dropped. Consider two providers with deterministic costs $\bar c_{i^*}<\bar c_i$ and identical deterministic scores at least $q_{\min}$. If $\Delta_i \le d_-+d_+$, their admissible bid intervals overlap between the two costs, so both can submit the same bid after initialization. Whenever their selection counts differ, an ordinary round selects the less-sampled provider because $\beta$ is strictly decreasing. The counts therefore differ by at most one plus the number of exploration rounds. Since there are only $O(g(T))=o(T)$ such rounds, each provider serves $T/2+o(T)$ queries.
When some qualified competitor is within $d_- + d_+$ of $i^*$, the platform therefore need not identify $i^*$; claim i) shows what it guarantees instead: on every ordinary round it selects a provider whose mean cost exceeds $\bar c_{i^*}$ by at most $d_- + d_+$ plus a vanishing term, and that provider is qualified on all but the rounds bounded by Lemma~\ref{lem:unqualified-elimination}.

\begin{theorem}
\label{thm:deviation-generation-regret}
    Under Assumption~\ref{asmp:bounded-deviations}, on the event $\mathcal E$, the generation-cost regret of Theorem~\ref{thm:generation-regret} satisfies
    \begin{equation}
        R_T^{\mathrm{gen}} \;\leq\; (N-1)\lceil g(T)\rceil c_{\max} + 2c_{\max}\sqrt{2N(T-N)\ln\!\left(\frac{2\pi^2 N T^2}{3\delta}\right)} + (d_- + d_+)(T - N).
        \label{eq:deviation-generation-regret}
    \end{equation}
    If moreover $\Delta_i > d_- + d_+$ for every $i \in \Qcal \setminus \{i^*\}$, the second bound of Theorem~\ref{thm:generation-regret} holds with $L_i$ replaced by $L_i^{(d)}$.
\end{theorem}
\begin{proof}
     In the proof of the first bound of Theorem~\ref{thm:generation-regret}, the inequality $\bar c_i - 2c_{\max}\beta(m_{i,t}) \leq \bar c_{i^*}$ on non-exploration steps is replaced by Eq.~\ref{eq:deviation-selection-gap}. Splitting $a_t \leq 2c_{\max}\beta(m_{I_t,t}) + (d_- + d_+)$, Lemma~\ref{lem:radius-sum} with $\kappa = 2c_{\max}$ bounds the sum of the first terms as before, and the second terms contribute at most $(d_- + d_+)(T-N)$. The second bound uses the bids only through Lemma~\ref{lem:qualified-suboptimal-selection}, which Proposition~\ref{prop:deviation-selection}ii) replaces.
\end{proof}
Dividing by $T$, the average generation-cost regret is at most $d_- + d_+$ up to a term of order $T^{-(1-\alpha)} + \sqrt{\ln T / T}$, whatever the cost gaps.

\paragraph{Payments.}
Overbidding and underbidding affect the two sides of the payment bounds separately: the payment is set by the adjusted bid of the cheapest competing provider, which overbidding raises by at most $d_+$ and underbidding lowers by at most $d_-$.

\begin{lemma}
\label{lem:deviation-payment}
    Under Assumption~\ref{asmp:bounded-deviations}, on the event $\mathcal E$, on every non-exploration step $t > N$ with $I_t = i$,
    \begin{equation}
        \pi_t \;\leq\; \min\bigl\{c_{\max},\; \bar c_{(2)} + c_{\max}\beta(m_{i,t}) + d_+\bigr\},
        \label{eq:deviation-payment-upper}
    \end{equation}
    and, if moreover $t > T_0$,
    \begin{equation}
        \pi_t \;\geq\; \bar c_{(2)} - 3c_{\max}\beta(g(t)) - d_-.
        \label{eq:deviation-payment-lower}
    \end{equation}
\end{lemma}
\begin{proof}
    \emph{Upper bound.} As in the proof of Lemma~\ref{lem:ordinary-payment}, with $j = i^*$ if $i \neq i^*$ and $j$ the second-cheapest qualified provider if $i = i^*$: by Eq.~\ref{eq:deviation-adjusted-bid-bounds}, $P_{i,t} \leq c_{\max}\beta(m_{i,t}) + b_{j,t} - c_{\max}\beta(m_{j,t}) \leq c_{\max}\beta(m_{i,t}) + \bar c_j + d_+ \leq c_{\max}\beta(m_{i,t}) + \bar c_{(2)} + d_+$.

    \emph{Lower bound.} For $t > T_0$, Lemma~\ref{lem:ordinary-qualified-set} gives $\hat{\Qcal}_t = \Qcal$ and $m_{j,t} \geq g(t)$ for every $j \in \Qcal$, so $\beta(m_{j,t}) \leq \beta(g(t))$ by Lemma~\ref{lem:radius-monotone}. Write $A_{j,t} = b_{j,t} - c_{\max}\beta(m_{j,t})$, so that $P_{i,t} = c_{\max}\beta(m_{i,t}) + \min_{j \in \Qcal \setminus \{i\}} A_{j,t}$. For every $j \in \Qcal \setminus \{i, i^*\}$, Eq.~\ref{eq:deviation-adjusted-bid-bounds} and $\bar c_j \geq \bar c_{(2)}$ give $A_{j,t} \geq \bar c_{(2)} - 2c_{\max}\beta(g(t)) - d_-$. If $i = i^*$, this covers every $j \in \Qcal \setminus \{i\}$ and $P_{i,t} \geq \bar c_{(2)} - 2c_{\max}\beta(g(t)) - d_-$. If $i \neq i^*$, the remaining competitor is $i^*$, and since $i$ was selected, $A_{i^*,t} \geq A_{i,t}$, so that
    \begin{align*}
        c_{\max}\beta(m_{i,t}) + A_{i^*,t} 
        &\geq c_{\max}\beta(m_{i,t}) + A_{i,t} = b_{i,t} \geq \hat c_{i,t} - d_- \\
        &\geq \bar c_i - c_{\max}\beta(g(t)) - d_- \geq \bar c_{(2)} - c_{\max}\beta(g(t)) - d_-
    \end{align*}
    using $i \in \Qcal \setminus \{i^*\}$ in the last step. Hence $P_{i,t} \geq \bar c_{(2)} - 2c_{\max}\beta(g(t)) - d_-$ in both cases, and $\pi_t = \min\{P_{i,t}, c_{\max}\} \geq \bar c_{(2)} - 2c_{\max}\beta(g(t)) - d_- \geq \bar c_{(2)} - 3c_{\max}\beta(g(t)) - d_-$, since $\bar c_{(2)} \leq c_{\max}$.
\end{proof}
At $d_- = d_+ = 0$, Eqs.~\ref{eq:deviation-payment-upper} and~\ref{eq:deviation-payment-lower} are Lemmas~\ref{lem:ordinary-payment} and~\ref{lem:two-sided-payment}.

\begin{theorem}\label{thm:deviation-payment-convergence}
    Under Assumption~\ref{asmp:bounded-deviations}, with probability at least $1-\delta$, for every $T \geq T_0$,
    \begin{align}
        \frac{1}{T}\sum_{t=1}^{T} \bigl|\pi_t - \bar c_{(2)}\bigr|
        &\leq \frac{T_0}{T}\,c_{\max} + k\left(\frac{N}{T}\right)^{\!1-\alpha}\bigl(c_{\max} - \bar c_{(2)}\bigr) + \max\{d_-, d_+\} \nonumber \\
        &\quad + \frac{6\,c_{\max}}{2-\alpha}\sqrt{\frac{1}{2k}\left(\frac{N}{T}\right)^{\!\alpha}\ln\!\left(\frac{2\pi^2 N k^2}{3\delta}\left(\frac{T}{N}\right)^{\!2\alpha}\right)}
        \label{eq:deviation-average-absolute-payment}
    \end{align}
    and consequently
    \begin{equation}
        \bar c_{(2)} - d_- \;\leq\; \liminf_{T\to\infty} \frac{1}{T}\sum_{t=1}^{T}\pi_t \;\leq\; \limsup_{T\to\infty} \frac{1}{T}\sum_{t=1}^{T}\pi_t \;\leq\; \bar c_{(2)} + d_+.
        \label{eq:deviation-payment-limits}
    \end{equation}
    Moreover, the average excess payment of Corollary~\ref{cor:excess-payment-polynomial} is bounded by its right-hand side plus $d_+$.
\end{theorem}
\begin{proof}
    In the proof of Theorem~\ref{thm:payment-convergence}, the only change is on the non-exploration steps $t > T_0$, where Lemma~\ref{lem:deviation-payment} gives $|\pi_t - \bar c_{(2)}| \leq 3c_{\max}\beta(g(t)) + \max\{d_-, d_+\}$ in place of $3c_{\max}\beta(g(t))$; summing the additional term over at most $T$ steps and dividing by $T$ gives Eq.~\ref{eq:deviation-average-absolute-payment}. For Eq.~\ref{eq:deviation-payment-limits}, the same decomposition applied to $\pi_t - \bar c_{(2)}$ without absolute values gives, for the non-exploration steps $t > T_0$, $-3c_{\max}\beta(g(t)) - d_- \leq \pi_t - \bar c_{(2)} \leq c_{\max}\beta(g(t)) + d_+$, while the first $T_0$ steps and the exploration steps contribute $O(T_0/T + T^{-(1-\alpha)})$ to the average; the vanishing terms are those of Eq.~\ref{eq:deviation-average-absolute-payment}. The excess-payment claim follows from Eq.~\ref{eq:deviation-payment-upper} in the proof of Theorem~\ref{thm:excess-payment-general}, whose $\sqrt{\cdot}$ term is unchanged and which acquires the additional term $d_+(T-N)$.
\end{proof}

\paragraph{Summary.}
On the good event, under bounded deviations, the platform still serves qualified providers on all but a bounded number of rounds; it identifies $i^*$ whenever every qualified competitor is more than $d_- + d_+$ costlier, and otherwise selects a qualified provider within $d_- + d_+$ of $\bar c_{i^*}$ in mean cost; and the average payment converges to within $[\bar c_{(2)} - d_-,\ \bar c_{(2)} + d_+]$ at the rate of Theorem~\ref{thm:regret}. Underbidding thus affects the lower side of the payment guarantee by at most $d_-$, whereas overbidding affects the upper side by at most $d_+$; neither affects the quality guarantee, and both slow identification through the reduced gaps $\Delta_i - d_- - d_+$. Setting $d_- = d_+ = 0$ recovers Theorems~\ref{thm:sample-complexity}, \ref{thm:generation-regret}, and~\ref{thm:regret} verbatim. Finally, for the Bayesian variant of Appendix~\ref{app:bayesian}, with Assumption~\ref{asmp:bounded-deviations} stated relative to $e_{i,t}$ in place of $\hat c_{i,t}$, the same statements hold with $c_{\max}\beta(\cdot)$ replaced by $(1+\gamma)c_{\max}\beta(\cdot)$ throughout and $L_i^{(d)}$ replaced by $M\bigl((\Delta_i - d_- - d_+)/((1+\gamma)c_{\max})\bigr)$, since the proofs above use the empirical mean only through Eq.~\ref{eq:good-event}, which Lemma~\ref{lem:bayesian-cost-concentration} replaces.

\clearpage

\section{Multidimensional quality thresholds} \label{app:multidimensional-quality}

The quality requirement can impose separate thresholds on several criteria, such as correctness, latency, and availability. The platform then compares providers by cost among those meeting every requirement. This extension changes the quality estimates and their confidence bounds. The selection and payment rules remain unchanged, using the eligible set defined by the multidimensional constraints.

\paragraph{Setting.} Fix a finite number of criteria $d\geq1$ and thresholds $\mathbf q_{\min}=(q_{\min,1},\ldots,q_{\min,d})\in[0,1]^d$, specified before the process begins. Retain $c_{\max}>0$, $\delta\in(0,1)$, and bids $b_{i,t}\in[0,c_{\max}]$. Each assignment produces a score vector $\mathbf r_t=(r_{1,t},\ldots,r_{d,t})\in[0,1]^d$. Each coordinate is a fixed scoring rule applied to the query, response, and any required service observations.

The $m$-th query assignment to provider $i$ reveals the $m$-th observation of its cost and quality scores:
\begin{equation}
     (\mathbf R_{i,m},C_{i,m})\in[0,1]^d\times[0,c_{\max}],
     \qquad
     \mathbb E[R_{i,a,m}]=q_{i,a},\quad
     \mathbb E[C_{i,m}]=\bar c_i.
\end{equation}

Each scalar sequence $(R_{i, a, m})_{m \geq 1}$ and $(C_{i, m})_{m \geq 1}$ is i.i.d. 

Write
\begin{equation} \label{eq:md-qualified}
    \mathcal Q=\{i\in[N]:q_{i,a} \geq q_{\min, a} \text{ for every } a \in [d]\}.
\end{equation}
Here qualification means satisfying all thresholds jointly. As in the main analysis, assume $|\mathcal Q| \geq 2$ and let $i^*$ be its unique cheapest provider. Set $\bar c_{(2)}=\min_{i\in\mathcal Q\setminus\{i^*\}}\bar c_i$ and $\Delta_i=\bar c_i-\bar c_{i^*}>0$ for $i\in\mathcal Q\setminus\{i^*\}$.

\paragraph{Quality estimates and mechanism.} Keep $m_{i,t}=\sum_{s<t}\mathbbm{1}\{I_s=i\}$ and the initialization in Algorithm~\ref{alg:platform}. For $t>N$, replace Eqs.~\ref{eq:quality-estimate} and~\ref{eq:qualified-set} by
\begin{align}
    \hat q_{i,a,t}
    &=\frac{1}{m_{i,t}}\sum_{s<t}r_{a,s}\mathbbm{1}\{I_s=i\}, \qquad a\in[d], \label{eq:md-estimates}\\
    \hat{\mathcal Q}_t
    &=\{i\in[N]:\hat q_{i,a,t}+\beta_d(m_{i,t})\geq q_{\min,a} \text{ for every }a\in[d]\}, \label{eq:md-eligible}
\end{align}
where, for $m\geq1$,
\begin{equation}
    \beta_d(m)=\sqrt{\frac{1}{2m}\ln\!\left(\frac{2\pi^2Nd m^2}{3\delta}\right)},
    \label{eq:md-radii}
\end{equation}
Each quality coordinate uses the confidence radius $\beta_d(m)$, which reduces to the original radius $\beta(m)$ when $d = 1$. The cost confidence radius remains $c_{\max}\beta(m)$.

The selection and payment rules in Eqs.~\ref{eq:allocation}--\ref{eq:payment} remain unchanged, using the eligible set in Eq.~\ref{eq:md-eligible}. All other steps of Algorithm~\ref{alg:platform}, including initialization, exploration, and bid submissions, are unchanged. Retain $g(t)=k(t/N)^\alpha$, with $k>0$ and $\alpha\in(0,1)$. The history in Eq.~\ref{eq:provider-history} now records the score vector in place of the scalar score.

\paragraph{Simultaneous confidence and quality gaps.} Let $\hat q_{i,a,m}$ and $\hat c_{i,m}$ denote the empirical means of the first $m$ observations of provider $i$. Define the event
\begin{equation}
    \mathcal E_d=\left\{
    \begin{array}{ll}
    |\hat q_{i,a,m}-q_{i,a}|\leq\beta_d(m),&i\in[N],\ a\in[d],\ m\geq1,\\
    |\hat c_{i,m}-\bar c_i|\leq c_{\max}\beta(m),&i\in[N],\ m\geq1
    \end{array}\right\}.
    \label{eq:md-event}
\end{equation}

For every provider, write $\varepsilon_{i,a}=(q_{\min,a}-q_{i,a})^+$ and $\varepsilon_i = \max_{a\in[d]}\varepsilon_{i,a}$. Thus $\varepsilon_i > 0$ exactly when $i \not\in \mathcal{Q}$. For $0 < x \leq 1$, set
\begin{equation}
    \begin{aligned}
    M_d(x)&=\left\lceil\frac{4}{x^2}\ln\!\left(\frac{57Nd}{\delta x^4}\right)\right\rceil,\\
    M_i^{(d)}&=M_d(\varepsilon_i)\quad(i\notin\mathcal Q),\qquad
    L_i=M_1(\Delta_i/c_{\max})\quad(i\in\mathcal Q\setminus\{i^*\}).
    \end{aligned}
    \label{eq:md-counts}
\end{equation}
The quantity $\varepsilon_i$ is used only to bound learning; qualification still tests each coordinate separately.

\begin{lemma}\label{lem:md-confidence}
The event $\mathcal E_d$ has probability at least $1-\delta$. On this event, for every $t>N$, all qualified providers belong to $\hat{\mathcal Q}_t$, and
\begin{equation}
    i\in\hat{\mathcal Q}_t\setminus\mathcal Q
    \quad\Longrightarrow\quad
    \varepsilon_i\leq2\beta_d(m_{i,t}),\qquad m_{i,t}\leq M_i^{(d)}.
    \label{eq:md-rejection}
\end{equation}
Consequently, each unqualified provider is selected at most $M_i^{(d)}$ times after initialization. These statements hold for arbitrary bidding strategies.
\end{lemma}
\begin{proof}
    For fixed $i,a,m$, Hoeffding's inequality gives
    \begin{equation}
        \Pr\!\left(|\hat q_{i,a,m}-q_{i,a}|>\beta_d(m)\right)
        \leq2e^{-2m\beta_d(m)^2}=\frac{3\delta}{\pi^2Nd m^2}.
    \end{equation}
    Similarly, for each provider $i$ and sample size $m\geq1$,
    \begin{equation}
        \Pr\!\left(|\hat c_{i,m}-\bar c_i|>c_{\max}\beta(m)\right)
        \leq \frac{3\delta}{\pi^2Nm^2}.
    \end{equation}

    A union bound over providers, selection counts, and quality coordinates yields
    \begin{equation}
        \Pr(\mathcal E_d^c)
        \leq\sum_{i=1}^N\sum_{m=1}^\infty
        \left(\frac{3\delta}{\pi^2Nm^2}
            +\sum_{a=1}^d\frac{3\delta}{\pi^2Nd m^2}\right)
        =\delta.
    \end{equation}

    On $\mathcal E_d$, if $i\in\mathcal Q$, then for every coordinate
    \begin{equation}
         \hat q_{i,a,t}+\beta_d(m_{i,t})\geq q_{i,a}\geq q_{\min,a}.
    \end{equation}
    Thus every qualified provider remains eligible, including those whose mean quality equals the threshold for some criteria. Since $|\mathcal Q|\geq2$, the eligible set is nonempty on $\mathcal E_d$. Conversely, if $i\in\hat{\mathcal Q}_t$, then
    \begin{equation}
        q_{\min,a}\leq\hat q_{i,a,t}+\beta_d(m_{i,t})
        \leq q_{i,a}+2\beta_d(m_{i,t}) \quad\text{for every }a.
    \end{equation}

    Taking the maximum over $a\in[d]$ and using $\varepsilon_i=\max_{a\in[d]}(q_{\min,a}-q_{i,a})^+$ gives the first inequality in Eq.~\ref{eq:md-rejection}.

    Applying Lemma~\ref{lem:radius-inversion} with $\delta/d$ in place of $\delta$ and $a=\varepsilon_i$ gives $m_{i,t}\leq M_i^{(d)}$ whenever an unqualified provider is eligible, not just when it is selected. Finally, immediately before its $\ell$-th post-initialization selection, a provider has count $m_{i,t}=\ell$. Since selection requires eligibility, $\ell\leq M_i^{(d)}$. This proves the stated total bound
    without any condition on bids.
\end{proof}

\paragraph{Provider incentives.} Theorem~\ref{thm:robust-dominant} continues to hold under multidimensional quality requirements. Conditional on the history, the eligible set and exploration rule remain independent of the provider's current bid, while the critical-payment characterization and payment rule are unchanged. The original proof therefore applies without modification.

\paragraph{Allocation and quality.} The bound in Theorem~\ref{thm:sample-complexity} changes only through the terms for unqualified providers. The cost-gap thresholds $L_i$ for qualified but more expensive providers remain unchanged.

\begin{theorem}\label{thm:md-allocation}
On $\mathcal E_d$, if all providers bid their empirical cost estimates, then simultaneously for all $T\geq N$,
\begin{equation}
    \sum_{t=1}^T\mathbbm{1}\{I_t\neq i^*\}
    \leq N-1+\sum_{i\notin\mathcal Q}M_i^{(d)}
       +\sum_{i\in\mathcal Q\setminus\{i^*\}}
          \max\{L_i,\lceil g(T)\rceil-1\}.
    \label{eq:md-identification}
\end{equation}
Consequently, the fraction assigned to $i^*$ tends to one. Under arbitrary bids, on the same event,
\begin{equation}
    \sum_{t=1}^T\mathbbm{1}\{I_t\notin\mathcal Q\}
    \leq\sum_{i\notin\mathcal Q}(M_i^{(d)}+1),
    \label{eq:md-violations}
\end{equation}
and, simultaneously for every $a\in[d]$ and $T\geq N$,
\begin{align}
    R_{T,a}^{\mathrm{qual}}
    &\coloneqq\sum_{t=1}^T(q_{\min,a}-q_{I_t,a})^+ \nonumber\\
    &\leq\min\left\{
    N+2\sqrt{2N(T-N)\ln\!\left(\frac{2\pi^2NdT^2}{3\delta}\right)},\
    \sum_{i\notin\mathcal Q}(M_i^{(d)}+1)\varepsilon_{i,a}\right\}.
    \label{eq:md-quality}
\end{align}
For a fixed instance, both $R_{T,a}^{\mathrm{qual}}/T$ and the fraction assigned to a provider violating at least one threshold are $O(T^{-1})$.
\end{theorem}
\begin{proof}
    Work on $\mathcal E_d$ and fix an integer $T\geq N$. Under cost-estimate bidding, the cost confidence bounds are unchanged, and Lemma~\ref{lem:md-confidence} ensures that $i^*$ remains eligible. Therefore, the arguments in Lemmas~\ref{lem:adjusted-bid-bounds}, \ref{lem:ordinary-selection-gap}, \ref{lem:ordinary-selection-threshold}, and~\ref{lem:qualified-suboptimal-selection} apply: each $i\in\mathcal Q\setminus\{i^*\}$ is selected at most $\max\{L_i,\lceil g(T)\rceil-1\}$ times after initialization. By Lemma~\ref{lem:md-confidence}, each unqualified provider $i$ is selected at most $M_i^{(d)}$ times after initialization. Summing these bounds and adding the $N-1$ initialization selections of providers other than $i^*$ proves Eq.~\ref{eq:md-identification}. For a fixed instance, its right-hand side is $O(T^\alpha)$. Since $\alpha<1$, the fraction of queries assigned to $i^*$ tends to one.

    For the quality guarantees, bids need not equal cost estimates. Including initialization, each unqualified provider $i$ is selected at most $M_i^{(d)}+1$ times. Summing over unqualified providers proves Eq.~\ref{eq:md-violations}. For each coordinate $a\in[d]$, weighting these counts by $\varepsilon_{i,a}$ gives
    \begin{equation}
        R_{T,a}^{\mathrm{qual}}
        \leq
        \sum_{i\notin\mathcal Q}
            (M_i^{(d)}+1)\varepsilon_{i,a},
    \end{equation}
    which is the second bound in Eq.~\ref{eq:md-quality}.

    To prove the first bound, fix $a\in[d]$. Initialization contributes at most $N$, since $(q_{\min,a}-q_{i,a})^+\leq1$. For every $t>N$, the selected provider is eligible. On $\mathcal E_d$, this implies
    \begin{equation}
        q_{\min,a}
        \leq \hat q_{I_t,a,t}+\beta_d(m_{I_t,t})
        \leq q_{I_t,a}+2\beta_d(m_{I_t,t}),
    \end{equation}
    and hence $\varepsilon_{I_t,a}\leq2\beta_d(m_{I_t,t})$. Applying Lemma~\ref{lem:radius-sum} with $\kappa=2$ and $\delta/d$ in place of $\delta$, and adding the initialization contribution, proves the first bound in Eq.~\ref{eq:md-quality}.

    All arguments hold for every $T\geq N$ and every coordinate on the same event $\mathcal E_d$. For a fixed instance, the right-hand sides of Eq.~\ref{eq:md-violations} and the second bound in Eq.~\ref{eq:md-quality} are independent of $T$. Dividing by $T$ therefore gives the $O(T^{-1})$ rates.
\end{proof}

\paragraph{Payments.} The payment guarantee in Theorem~\ref{thm:regret} continues to hold on $\mathcal E_d$, with $T_0$ replaced by
\begin{equation}
    T_{0,d}
    =N\max\{1,(\bar M_d/k)^{1/\alpha}\},
    \qquad
    \bar M_d
    =\max\bigl(\{1\}\cup\{M_i^{(d)}:i\notin\mathcal Q\}\bigr).
\end{equation}
At every non-exploration step $t>T_{0,d}$, the eligible set equals $\mathcal Q$, after which the original payment analysis applies unchanged. The payment benchmark $\bar c_{(2)}$ and convergence
rate therefore remain unchanged.

\clearpage

\section{Soft quality constraints}
\label{app:soft-objectives}

The platform can also account for preferences over observable performance attributes, such as latency or energy consumption, while retaining every quality threshold in Appendix~\ref{app:multidimensional-quality}. Let $\boldsymbol\mu_i\in\mathbb R^p$ be the vector of provider $i$'s mean attributes, and let $F$ be a known function that assigns a finite penalty in monetary units. Write
\begin{equation}
    f_i=F(\boldsymbol\mu_i),
    \qquad G_i=\bar c_i+f_i.
    \label{eq:soft-objective}
\end{equation}
The objective is to minimize $G_i$ among providers in $\mathcal Q$. The function $F$ may be a weighted sum or include interactions between attributes. The penalty expresses the user's preferences; provider utility remains Eq.~\ref{eq:utilities}.

Retain the setting of Appendix~\ref{app:multidimensional-quality}, including $|\mathcal Q|\geq2$ and $g(t)=k(t/N)^\alpha$. Replace the assumption of a unique cheapest qualified provider by a unique minimizer
\begin{equation}
    \begin{aligned}
        i_F&=\argmin_{i\in\mathcal Q}G_i,
        \qquad
        G_{(2)}=\min_{i\in\mathcal Q\setminus\{i_F\}}G_i,\\
        \Delta_i^F&=G_i-G_{i_F}>0
        \quad(i\in\mathcal Q\setminus\{i_F\}).
    \end{aligned}
    \label{eq:soft-optimum}
\end{equation}

\paragraph{Penalty estimates.} Before each step $t>N$, the platform computes a finite estimate $\ell_{i,t}$ of each penalty from past observations, which are included in the history. Conditional on that history, these estimates do not depend on the current bids. Suppose that, on an event $\mathcal A$ of probability at least $1-\eta$, where $0\leq\eta<1-\delta$, simultaneously for every $i$ and $t>N$,
\begin{equation}
    0\leq f_i-\ell_{i,t}\leq w(m_{i,t}),
    \label{eq:soft-confidence}
\end{equation}
where $w:[1,\infty)\to[0,\infty)$ is finite, deterministic, nonincreasing, and tends to zero. Thus the estimates are optimistic, with an error that shrinks as observations accumulate.

The allocation and payment guarantees below hold on $\mathcal E=\mathcal E_d\cap\mathcal A$, whose probability is at least $1-\delta-\eta$. No independence between these events is required. Equation~\ref{eq:soft-confidence} is the only additional estimation assumption; a sufficient construction is given below.

\paragraph{Selection and payments.} Keep the eligible set, initialization, exploration, and bid submissions unchanged. At a non-exploration step with a nonempty eligible set, use
\begin{equation}
    J_{i,t}=b_{i,t}-c_{\max}\beta(m_{i,t})+\ell_{i,t},
    \qquad
    I_t\in\argmin_{i\in\hat{\mathcal Q}_t}J_{i,t}.
    \label{eq:soft-allocation}
\end{equation}
The corresponding critical monetary payment and actual payment are
\begin{equation}
    \begin{aligned}
    P^F_{i,t}
    &=c_{\max}\beta(m_{i,t})-\ell_{i,t}
      +\min_{j\in\hat{\mathcal Q}_t\setminus\{i\}}J_{j,t},\\
    \pi_t&=\min\{P^F_{I_t,t},c_{\max}\}.
    \end{aligned}
    \label{eq:soft-critical}
\end{equation}
As before, the minimum over an empty competitor set is $+\infty$. Initialization and exploration continue to pay $c_{\max}$, as when no provider is eligible.

\paragraph{Provider incentives.} Holding the history and opponents' bids fixed, rearranging Eq.~\ref{eq:soft-allocation} shows that provider $i$ wins when $b_{i,t}<P^F_{i,t}$ and loses when $b_{i,t}>P^F_{i,t}$. Whenever $i$ is selected, $P^F_{i,t}\geq b_{i,t}\geq0$; since $b_{i,t}\leq c_{\max}$, its payment covers its bid. The one-step comparison in Lemma~\ref{lem:one-step} and the proof of Theorem~\ref{thm:robust-dominant} therefore apply with $P^F_{i,t}$ in place of $P_{i,t}$. In particular, cost-estimate bidding maximizes the original worst-case estimate-based utility. This conclusion does not require Eq.~\ref{eq:soft-confidence}.

\paragraph{Allocation and quality.} Define, for $x\geq1$,
\begin{equation}
    \begin{aligned}
    h(x)&=2c_{\max}\beta(x)+w(x),\\
    L_i^F&=\max\bigl(\{0\}\cup
        \{m\in\mathbb N:m\geq1,\ h(m)\geq\Delta_i^F\}\bigr),
        \qquad i\in\mathcal Q\setminus\{i_F\}.
    \end{aligned}
    \label{eq:soft-counts}
\end{equation}
Each $L_i^F$ is finite because $h(m)\to0$.

\begin{theorem} \label{thm:soft-allocation}
    On $\mathcal E$, if all providers bid their empirical cost estimates, then simultaneously for every integer $T\geq N$,
    \begin{equation}
        \sum_{t=1}^T\mathbbm{1}\{I_t\neq i_F\}
        \leq N-1+\sum_{i\notin\mathcal Q}M_i^{(d)}
           +\sum_{i\in\mathcal Q\setminus\{i_F\}}
              \max\{L_i^F,\lceil g(T)\rceil-1\}.
        \label{eq:soft-identification}
    \end{equation}
    Consequently, the fraction assigned to $i_F$ tends to one. Under arbitrary bids, the hard-quality bounds in Eqs.~\ref{eq:md-violations} and~\ref{eq:md-quality} remain valid on $\mathcal E_d$.
\end{theorem}
\begin{proof}
    On $\mathcal E$, cost concentration and Eq.~\ref{eq:soft-confidence} give
    \begin{equation}
        G_i-h(m_{i,t})\leq J_{i,t}\leq G_i.
        \label{eq:soft-index-bounds}
    \end{equation}
    If a qualified provider $i\neq i_F$ is selected on a non-exploration step, $i_F$ is eligible and hence
    \begin{equation}
        G_i-h(m_{i,t})\leq J_{i,t}\leq J_{i_F,t}\leq G_{i_F}.
    \end{equation}

     Therefore $\Delta_i^F\leq h(m_{i,t})$, so $m_{i,t}\leq L_i^F$. The counting argument of Theorem~\ref{thm:md-allocation} now applies with $i_F$ and $L_i^F$ in place of $i^*$ and $L_i$, proving Eq.~\ref{eq:soft-identification}. Its right-hand side is $O(T^\alpha)$ for a fixed instance.

     The eligibility rule is unchanged, and on $\mathcal E_d$ every post-initialization selection still comes from the eligible set. The quality part of the proof of Theorem~\ref{thm:md-allocation} therefore applies under arbitrary bids.
\end{proof}

\paragraph{Payments.} The limiting monetary payment is
\begin{equation}
    p_F=\min\{c_{\max},G_{(2)}-f_{i_F}\}.
    \label{eq:soft-benchmark}
\end{equation}
Here $G_{(2)}-f_{i_F}=\bar c_{i_F}+G_{(2)}-G_{i_F}\geq0$. Define
\begin{equation}
    \bar L_F=\max\left\{\bar M_d,
        \max_{i\in\mathcal Q\setminus\{i_F\}}L_i^F\right\},
    \qquad
    T_F=N\max\{1,(\bar L_F/k)^{1/\alpha}\}.
    \label{eq:soft-cutoff}
\end{equation}
This cutoff accounts for both rejecting unqualified providers and identifying $i_F$ on non-exploration steps.

\begin{theorem}\label{thm:soft-payments}
    On $\mathcal E$, if all providers bid their empirical cost estimates, every non-exploration step $t>T_F$ selects $i_F$ and satisfies
    \begin{equation}
        |\pi_t-p_F|\leq h(g(t)).
        \label{eq:soft-payment-bound}
    \end{equation}
    For every $T>T_F$,
    \begin{equation}
        \begin{aligned}
            \frac1T\sum_{t=1}^T|\pi_t-p_F|
            \leq{}&\frac{T_Fc_{\max}}{T}
               +k\left(\frac NT\right)^{1-\alpha}(c_{\max}-p_F)\\
               &+\frac1T\sum_{\substack{t\in\mathbb N\\T_F<t\leq T}}h(g(t)).
        \end{aligned}
        \label{eq:soft-average-payment}
    \end{equation}
    Thus payments converge to $p_F$ along non-exploration steps, and the average absolute difference between payments and $p_F$ tends to zero. Exploration steps continue to pay $c_{\max}$.
\end{theorem}
\begin{proof}
    At a non-exploration step $t>T_F$, every eligible provider has $m_{i,t}\geq g(t)>\bar L_F$. Lemma~\ref{lem:md-confidence} therefore gives $\hat{\mathcal Q}_t=\mathcal Q$, and the selection condition in the preceding proof implies $I_t=i_F$. By Eq.~\ref{eq:soft-index-bounds} and monotonicity of $h$,
    \begin{equation}
        G_{(2)}-h(g(t))
        \leq\min_{j\in\mathcal Q\setminus\{i_F\}}J_{j,t}
        \leq G_{(2)}.
    \end{equation}
    Substituting into Eq.~\ref{eq:soft-critical} and using $0\leq f_{i_F}-\ell_{i_F,t}\leq w(g(t))$ gives
    \begin{equation}
        -h(g(t))
        \leq P^F_{i_F,t}-(G_{(2)}-f_{i_F})
        \leq c_{\max}\beta(g(t))+w(g(t))
        \leq h(g(t)).
    \end{equation}
    Capping both quantities at $c_{\max}$ cannot increase their absolute difference, proving Eq.~\ref{eq:soft-payment-bound}.

    The averaging argument of Theorem~\ref{thm:payment-convergence} applies with cutoff $T_F$, benchmark $p_F$, and error $h(g(t))$: the steps $t\leq T_F$ contribute at most $T_Fc_{\max}$, and later exploration steps number at most $Ng(T)$ by Lemma~\ref{lem:forced-exploration-bound}. This proves Eq.~\ref{eq:soft-average-payment}. Since $h(g(t))\to0$ and $g(T)/T\to0$, its right-hand side tends to zero.
\end{proof}

\paragraph{A sufficient construction.} Suppose the $m$-th assignment to provider $i$ reveals $\mathbf Z_{i,m}\in[0,1]^p$, with each coordinate sequence i.i.d.\ and $\mathbb E[\mathbf Z_{i,m}]=\boldsymbol\mu_i$. Let $\hat{\boldsymbol\mu}_{i,t}$ be the empirical means. Assume that $F$ has a known constant $L\geq0$ such that $|F(u)-F(v)|\leq L\|u-v\|_\infty$ for all $u,v\in[0,1]^p$. For $0<\eta<1-\delta$, set
\begin{equation}
    \gamma_p(m)=\sqrt{\frac1{2m} \ln\!\left(\frac{2\pi^2Np m^2}{3\eta}\right)},
    \qquad
    \ell_{i,t}=F(\hat{\boldsymbol\mu}_{i,t}) -L\gamma_p(m_{i,t}).
    \label{eq:soft-lipschitz-estimator}
\end{equation}
The Hoeffding and union-bound argument of Lemma~\ref{lem:md-confidence} gives, with probability at least $1-\eta$, simultaneously for all providers and sample counts, $\|\hat{\boldsymbol\mu}_{i,t}-\boldsymbol\mu_i\|_\infty \leq\gamma_p(m_{i,t})$. Hence Eq.~\ref{eq:soft-confidence} holds with $w(x)=2L\gamma_p(x)$. For example, $F(u)=\sum_{a=1}^p\lambda_a u_a$ satisfies this condition with $L=\sum_{a=1}^p|\lambda_a|$.

For this construction, with $N,p,\delta,\eta,c_{\max}$ and $L$ fixed, $h(x)=2c_{\max}\beta(x)+2L\gamma_p(x) =O(\sqrt{\log x/x})$ as $x\to\infty$. Applying the summation bound from the proof of Theorem~\ref{thm:payment-convergence} to Eq.~\ref{eq:soft-average-payment} therefore gives
\begin{equation}
    \frac1T\sum_{t=1}^T|\pi_t-p_F|
    =O\!\left(T^{-(1-\alpha)}
        +T^{-\alpha/2}\sqrt{\log T}\right).
\end{equation}
The original rate is thus retained, with $\alpha=2/3$ balancing the two terms.

\clearpage

\section{Bidding Deviations with a Known Horizon}
\label{app:dynamic-deviations}

Theorem~\ref{thm:robust-dominant} establishes optimality under uncertainty about the horizon using a worst-case criterion. Here, we fix the horizon and examine unilateral bidding deviations against providers that bid their cost estimates. We first give an example of a profitable deviation. We then show that, with high probability, the gain per query in estimated and realized profit has a vanishing upper bound even when the cost-estimate bidding run and the deviating run are drawn independently.

\paragraph{A profitable deviation.} Consider two providers with deterministic costs $c_1 = 0.1$ and $c_2 = 0.2$, deterministic qualities $q_1 = q_2 = 1$, and parameters
\begin{equation*}
    c_{\max} = 1, \qquad q_{\min} = 0.5, \qquad \delta = 0.1, \qquad k = 0.5, \qquad \alpha = 3/4
\end{equation*}

After initialization, both selection counts equal one. Since $g(4) < 1$, rounds $3$ and $4$ are non-exploration under either strategy considered below. Write
\begin{equation*}
    d = \beta(1) - \beta(2) \approx 0.31
\end{equation*}

If both providers bid their costs, provider $1$ wins round $3$ at payment $0.2$. Provider $2$ then wins round $4$, since $0.2 - \beta(1) < 0.1 - \beta(2)$. Instead, suppose provider $1$ bids $0.3$ in rounds $3$ and $4$ and returns to cost bidding thereafter. Provider $2$ wins round $3$. Provider $1$ then wins round $4$, since $0.2 + d > 0.3$, and receives $0.2 + d \approx 0.51$.

Provider $1$ serves one query in these two post-initialization rounds under either strategy, so the deviation increases its profit by $d$. After round $4$, the counts, empirical estimates, and bids coincide. With a fixed tie-breaking rule and the same subsequent exploration draws, all later allocations and payments therefore coincide.

The deviation is profitable for every $T \geq 4$ against the specified competitor bidding its cost, so exact knowledge of the horizon is unnecessary. This example establishes that a profitable deviation can exist. It does not show that knowledge of the horizon bound alone guarantees a gain without further information about the competitor.

At provider $1$'s win, its own confidence term is $\beta(1)$ under either strategy. The gain comes from reducing the competitor's radius first. More generally, each selection increases the selected provider's count. Hence, if provider $i$ serves $n$ post-initialization queries, its own radii sum to $\sum^n_{m=1}\beta(m)$, independently of their timing. 

\paragraph{Bounding the gain.} Consider two independent runs $A$ and $B$ of the experiment. Each run has its own query and generation draws, initialization draws and platform observations. In each run, the $m$-th assignment to a provider reveals the $m$-th cost--score pair in that provider's i.i.d\ sequence. We evaluate the cost-estimate bidding profile $\boldsymbol{\sigma}^*$ in $A$ and all unilateral deviations $(\sigma_i,\boldsymbol{\sigma}_{-i}^*)$ in $B$.

Let $\mathcal E_A$ and $\mathcal E_B$ be the events in Eq.~\ref{eq:good-event}, and write $\mathcal{G} = \mathcal E_A \cap \mathcal E_B$. Since $N \geq 2$ and $\delta < 1$, $\beta(1) > 1$, so the cost and score bounds at $m = 1$ hold deterministically. For each $r\in\{A,B\}$, starting the union bound in Lemma~\ref{lem:good-event} at $m=2$ gives
\begin{equation*}
    \Pr(\mathcal{E}^c_r) \leq \frac{6\delta} {\pi^2}\sum_{m=2}^{\infty}\frac{1}{m^2} = \left(1-\frac{6}{\pi^2}\right)\delta < \delta / 2
\end{equation*}
Hence $\Pr(\mathcal{G}) \geq 1 - \delta$.
%Strategies may depend on the known horizon and the provider's observed history. We compare strategies using the same realizations of these sequences and the same initialization. The event $\mathcal{E}$ applies to every such run and has probability at least $1 - \delta$. 

For a strategy profile $\boldsymbol{\sigma}$, let $I_t^{\boldsymbol\sigma}$, $\pi_t^{\boldsymbol\sigma}$, and $\hat c_{i,t}^{\boldsymbol\sigma}$ denote the selections, payments, and cost estimates generated under that profile. Define

\begin{equation}
    V_i^T(\boldsymbol\sigma)
    =\sum_{t=N+1}^{T}
    \bigl(\pi_t^{\boldsymbol\sigma}
    -\hat c_{i,t}^{\boldsymbol\sigma}\bigr)
    \mathbbm{1}\{I_t^{\boldsymbol\sigma}=i\}.
    \label{eq:dynamic-surplus}
\end{equation}
This is provider $i$'s cumulative estimated profit after initialization, given by the sum inside the expectation in Eq.~\ref{eq:utilities}. We write $V^T_{i, r}$ when specifying the run $r \in \{A, B\}$, and omit the profile superscripts when the strategy profile is clear.

\begin{proposition}\label{prop:dynamic-deviation-gain}
    Under the assumptions above, retain $B_{\mathrm{id}}$ and $T_0$ from Theorems~\ref{thm:sample-complexity} and~\ref{thm:regret}. On $\mathcal G$, for every provider $i$, every admissible strategy $\sigma_i$, and every $T\geq\lceil T_0\rceil$,
    \begin{equation}
        V_{i,B}^T(\sigma_i,\boldsymbol\sigma_{-i}^*)
        -V_{i,A}^T(\boldsymbol\sigma^*)\leq R_T,
        \label{eq:dynamic-gain-bound}
    \end{equation}
    where
    \begin{align}
        R_T=c_{\max}\Biggl[&B_{\mathrm{id}}(T)+T_0+(N+1)g(T)
        \nonumber\\
        &+3\sqrt{2(T-N)\ln\!\left(\frac{2\pi^2NT^2}{3\delta}\right)}
        +\frac{6T}{2-\alpha}\,\beta(g(T))\Biggr].
        \label{eq:dynamic-gain-budget}
    \end{align}
    In particular, for fixed problem parameters and confidence level,
        \begin{equation}
        R_T=O\!\left(T^\alpha+T^{1-\alpha/2}\sqrt{\log T}\right)=o(T).
        \label{eq:dynamic-gain-rate}
    \end{equation}
\end{proposition}
\begin{proof}
    Work on $\mathcal{G}$. We upper bound each provider's estimated profit in the deviating run using $\mathcal E_B$ and lower bound its estimated profit in the cost-estimate bidding run using $\mathcal E_A$. 
    
    Write $C=c_{\max}$, $H = T - N$, and
    \begin{equation}
        S_T = \sqrt{2H\ln\!\left(\frac{2\pi^2NT^2}{3\delta}\right)}
    \end{equation}
    The sum of provider $i$'s radii over its post-initialization selections is at most $\sum_{m=1}^H\beta(m)\leq S_T$, by the counting argument of Lemma~\ref{lem:radius-sum}. Its number of exploration selections is at most $\lceil g(T)\rceil-1\leq g(T)$ by Lemma~\ref{lem:forced-exploration-bound}.

    First suppose $i\in\mathcal Q$, and let $c_{-i}=\min_{j\in\mathcal Q\setminus\{i\}}\bar c_j$ be the lowest mean cost among its qualified competitors. In run $B$, on $\mathcal E_B$ a competitor attaining this minimum remains eligible and, under cost-estimate bidding, has adjusted bid at most $c_{-i}$. Thus, whenever $i$ is selected at a non-exploration step, the payment rule and cost concentration give
    \begin{equation}
        \pi_t\leq c_{-i}+C\beta(m_{i,t}),
        \qquad
        \hat c_{i,t}\geq\bar c_i-C\beta(m_{i,t}).
    \end{equation}
    Subtracting yields
    \begin{equation}
        \pi_t-\hat c_{i,t}
        \leq c_{-i}-\bar c_i+2C\beta(m_{i,t}).
        \label{eq:dynamic-one-step-bound}
    \end{equation}

    For $i=i^*$, the cost difference is $\rho=\bar c_{(2)}-\bar c_{i^*}>0$. There are at most $H$ non-exploration selections, and their confidence radii sum to at most $S_T$. Exploration contributes at most $Cg(T)$, since there are at most $g(T)$ exploration selections of this provider, each giving estimated profit at most $C$. Therefore,
    \begin{equation}
        V_{i^*, B}^T(\sigma_{i^*},\boldsymbol\sigma_{-i^*}^*)
        \leq \rho H+2CS_T+Cg(T).
        \label{eq:dynamic-deviating-upper}
    \end{equation}

    Let $A_T$ denote the right-hand side of Eq.~\ref{eq:average-absolute-payment-user} multiplied by $T$. In run $A$, under $\boldsymbol\sigma^*$, each selection of $i^*$ contributes
    \begin{equation*}
        \pi_t-\hat c_{i^*,t}
        =\rho+(\pi_t-\bar c_{(2)})
        -(\hat c_{i^*,t}-\bar c_{i^*}).
    \end{equation*}
    On $\mathcal E_A$, the allocation bound gives at least $H-B_{\mathrm{id}}(T)$ post-initialization selections of $i^*$. Over these selections, the cumulative absolute payment error is at most $A_T$, and the cumulative absolute cost-estimation error is at most $CS_T$. Consequently,
    \begin{equation}
        V_{i^*, A}^T(\boldsymbol\sigma^*)
        \geq \rho\bigl(H-B_{\mathrm{id}}(T)\bigr)-A_T-CS_T.
        \label{eq:dynamic-truthful-lower}
    \end{equation}
    Subtracting this from Eq.~\ref{eq:dynamic-deviating-upper}, using $\rho\leq C$ and
    \begin{equation*}
        A_T\leq C\left[T_0+Ng(T)+\frac{6T}{2-\alpha}\beta(g(T))\right],
    \end{equation*}
    proves Eq.~\ref{eq:dynamic-gain-bound} for $i^*$.

    For $i\in\mathcal Q\setminus\{i^*\}$, we have $c_{-i}=\bar c_{i^*}<\bar c_i$, so Eq.~\ref{eq:dynamic-one-step-bound} gives
    \begin{equation*}
        V_{i,B}^T(\sigma_i,\boldsymbol\sigma_{-i}^*) \leq 2CS_T+Cg(T) \leq R_T.
    \end{equation*}
    Its estimated profit under $\boldsymbol\sigma^*$ is nonnegative at every selection, proving the desired bound.

    Finally, on $\mathcal E_B$, an unqualified provider has at most $M(\varepsilon_i)$ post-initialization selections by Lemma~\ref{lem:unqualified-elimination}, which does not require truthful bids. Each selection gives estimated profit at most $C$. Since its estimated profit under $\boldsymbol\sigma^*$ is nonnegative, its deviation gain is at most $CM(\varepsilon_i)\leq R_T$.

    For fixed parameters, $T_0$ is constant, $B_{\mathrm{id}}(T) = O(T^\alpha)$, $g(T) = O(T^\alpha)$, $S_T = O(\sqrt{T\log T})$, and $\beta(g(T)) = O(T^{-\alpha/2}\sqrt{\log T})$. Substituting these bounds into Eq.~\ref{eq:dynamic-gain-budget} gives the stated rate.
\end{proof}

\paragraph{Estimated utility.} For the fixed horizon and provider distributions, let $U_i^T(\boldsymbol{\sigma}) = \mathbb E[V_i^T(\boldsymbol{\sigma})]$ denote provider $i$'s estimated utility over the post-initialization rounds. Outside $\mathcal{G}$, the deviation gain is at most $c_{\max}(T-N)$, since payments are bounded by $c_{\max}$ and cost-estimate bidding gives nonnegative estimated profit on every run. Hence, for every unilateral deviation $\sigma_i$ and every $T \geq T_0$,
\begin{equation}
    U_i^T(\sigma_i,\boldsymbol\sigma_{-i}^*)
    -U_i^T(\boldsymbol\sigma^*)
    \leq R_T+\delta c_{\max}(T-N).
    \label{eq:dynamic-expected-gain}
\end{equation}
Thus $\boldsymbol{\sigma}^*$ is an $\epsilon_T$-Nash equilibrium for the estimated utilities $U_i^T$, where $\epsilon_T=R_T+\delta c_{\max}(T-N)$. For utilities normalized by $T$, the approximation error is $\epsilon_T/T$. For any fixed $\delta$, the bound $R_T / T$ on the gain in estimated profit per query vanishes on $\mathcal G$. The bound on the gain in estimated utility per query, $\epsilon_T/T$, retains a term of order $\delta$. If the platform knows $T$ in advance and chooses $\delta=T^{-1}$, Eq.~\ref{eq:dynamic-gain-budget}, with this choice used throughout $B_{\mathrm{id}}(T)$, $T_0$, and $\beta$, gives $\epsilon_T/T\to0$ for fixed remaining parameters.

\paragraph{Realized cumulative profit}
Let $W_i^T(\boldsymbol{\sigma})$ denote provider $i$'s realized cumulative profit in Eq.~\ref{eq:dynamic-surplus} with realized generation costs in place of empirical cost estimates.

Work in one run $r\in\{A,B\}$ under a strategy profile $\boldsymbol{\sigma}$. Let $\hat c_{i,m}$ denote the empirical mean of the first $m$ costs observed by provider $i$ in this run. Let $n$ be the total number of queries served by provider $i$ up to step $T$, including initialization. Its realized costs after initialization sum to $n\hat{c}_{i, n} - \hat{c}_{i, 1}$, while its cost estimates on those same selections sum to $\sum_{m=1}^{n-1} \hat{c}_{i,m}$. Subtracting $(n - 1)\bar c_i$ from each sum, on $\mathcal E_r$,
\begin{align*}
    \bigl|W_i^T(\boldsymbol{\sigma})-V_i^T(\boldsymbol{\sigma})\bigr| &= 
    \left|\sum_{m=1}^{n-1}(\hat c_{i,m}-\bar c_i) - n(\hat c_{i,n}-\bar c_i) + (\hat c_{i,1}-\bar c_i)\right| \\
    &\leq \sum_{m=1}^{n-1}|\hat c_{i,m}-\bar c_i| + n|\hat c_{i,n}-\bar c_i| + |\hat c_{i,1}-\bar c_i| \\
    &\leq c_{\max}\left(\sum_{m=1}^{n-1}\beta(m) + n\beta(n) + 1\right) \\
    &\leq c_{\max}\left[\sqrt{\frac{1}{2}\ln\!\left(\frac{2\pi^2NT^2}{3\delta}\right)}
    \left(\sum_{m=1}^{n-1}\frac{1}{\sqrt m}+\sqrt n\right) + 1\right] \\
    &\leq c_{\max}\left(1+3\sqrt{\frac{T}{2}\ln\!\left(\frac{2\pi^2NT^2}{3\delta}\right)}\right) \eqqcolon D_T
\end{align*}
where the second inequality uses Eq.~\ref{eq:good-event} and $|\hat c_{i,1}-\bar c_i|\leq c_{\max}$, and the last two use $m\leq n\leq T$ and $\sum_{m=1}^{n-1}m^{-1/2}\leq 2\sqrt n$.

Fix a unilateral deviation $\sigma_i$ and write $\boldsymbol{\sigma}'=(\sigma_i,\boldsymbol{\sigma}_{-i}^*)$. On $\mathcal G$, the preceding bound applies in each run with its own selection count, empirical means, and realized costs. Thus, for every $T\geq\lceil T_0\rceil$,
% \begin{align*}
%     W_i^T(\boldsymbol{\sigma}') &\leq V_i^T(\boldsymbol{\sigma}') + D_T, \\
%     W_i^T(\boldsymbol{\sigma}^*) &\geq V_i^T(\boldsymbol{\sigma}^*) - D_T.
% \end{align*}
% Using the first inequality for the deviating provider's profit and the second for its profit under cost-estimate bidding gives, for every $T \geq \lceil T_0\rceil$,
\begin{align*}
    W_i^T(\boldsymbol{\sigma}') - W_i^T(\boldsymbol{\sigma}^*)
    &\leq \bigl(V_i^T(\boldsymbol{\sigma}') + D_T\bigr) - \bigl(V_i^T(\boldsymbol{\sigma}^*) - D_T\bigr) \\
    &\leq R_T + 2D_T,
\end{align*}
where the last inequality follows from Proposition~\ref{prop:dynamic-deviation-gain}. For fixed problem parameters and confidence level, $D_T=O(\sqrt{T\log T})$ and $R_T=o(T)$. Hence, with probability at least $1-\delta$, the supremum of the realized gain per query over all deviations evaluated in $B$, relative to cost-estimate bidding in $A$, is upper-bounded by a quantity that vanishes as $T$ grows. The corresponding upper bound for estimated profit is $R_T/T$.

The same realized-profit bound holds if profits include initialization. Adding the initialization profit $c_{\max}-\hat c_{i,1}$ to $W_i^T$ and its mean $c_{\max}-\bar c_i$ to $V_i^T$ cancels the term $\hat c_{i,1}-\bar c_i$ in the preceding cost calculation. The common mean then cancels when comparing the two profiles.

If each deviation is evaluated on a separately drawn run, both profit bounds hold with probability at least $1-\delta$ for any given deviation. When all deviations are evaluated using the same provider outcome sequences in $B$, both bounds hold simultaneously for all deviations with probability at least $1-\delta$.

\paragraph{True utility.}
Replacing the empirical cost estimates in provider $i$'s estimated utility with the generation costs gives its true utility. For a fixed strategy profile $\boldsymbol{\sigma}$, this quantity is $\mathbb E[W_i^T(\boldsymbol{\sigma})]$. 
Under the i.i.d. assumption, the allocation depends only on past observations and the current bids. Therefore, the mean generation cost conditional on provider $i$ being selected is $\bar c_i$. Hence,
\begin{equation*}
    \mathbb E[W_i^T(\boldsymbol{\sigma})] = \mathbb E\!\left[\sum_{t=N+1}^{T}\bigl(\pi_t^{\boldsymbol{\sigma}}-\bar c_i\bigr)\mathbbm{1}\{I_t^{\boldsymbol{\sigma}}=i\}\right].
\end{equation*}

On $\mathcal G$, the gain in realized cumulative profit from any unilateral deviation is at most $R_T+2D_T$. Outside $\mathcal G$, the gain is at most $2c_{\max}(T-N)$, since each profile's realized profit lies in $[-c_{\max}(T-N),c_{\max}(T-N)]$. Taking expectations and using $\Pr(\mathcal G^c)\leq\delta$, we obtain, for every unilateral deviation $\sigma_i$ and every $T\geq\lceil T_0\rceil$,
\begin{equation*}
    \mathbb E\!\left[W_i^T(\sigma_i,\boldsymbol{\sigma}_{-i}^*)\right] - \mathbb E\!\left[W_i^T(\boldsymbol{\sigma}^*)\right] \leq R_T+2D_T+2\delta c_{\max}(T-N) \eqqcolon \epsilon_T^{\mathrm{true}}.
\end{equation*}
Thus, for providers maximizing expected cumulative profit, $\boldsymbol{\sigma}^*$ is an $\epsilon_T^{\mathrm{true}}$-Nash equilibrium. 
In particular, for any admissible deviation chosen with knowledge of $T$, the realized profit forgone per query by bidding one's cost estimate is at most $(R_T+2D_T)/T$ with probability at least $1-\delta$ when the two runs are drawn independently. For fixed $\delta$, this bound vanishes as $T\to\infty$. The corresponding bound on expected profit forgone per query, $\epsilon_T^{\mathrm{true}}/T$, retains a term of order $\delta$. If the platform also knows $T$ and sets $\delta=1/T$, then $\epsilon_T^{\mathrm{true}}/T\to 0$ as well.

\clearpage
\newpage

\section{Experimental details}
\label{app:experimental-details}

In this section, we describe the experimental setup we use to evaluate our platform (Algorithm~\ref{alg:platform}). Here, our goal is twofold: first, to instantiate the model of Section~\ref{sec:model} using recorded generations from real LLMs, so that provider qualities and generation costs are grounded in observed data; and second, to compare the platform against practical routing policies that a platform could deploy today. In what follows, we describe the data, the construction of providers, the experimental settings, the routing policies we compare against, and the metrics we report. 

\paragraph{Recorded generation data.}
We evaluate the platform using a publicly available dataset of recorded LLM generations.\footnote{\url{https://huggingface.co/datasets/Human-Centric-Machine-Learning/strategic-ttc-data}} The data cover \texttt{GSM8K}~\citep{cobbe2021trainingverifierssolvemath} and \texttt{GPQA}~\citep{rein2023gpqagraduatelevelgoogleproofqa}, spanning a wide range of difficulty. For every model--question pair, the data contain 128 recorded generations, each annotated with a binary correctness label, an output-token count, and, when needed for Best-of-$n_i$ service, a score from the reward model \texttt{ArmoRM-Llama3-8B-v0.1}~\citep{wang2024interpretablepreferencesmultiobjectivereward}. Our simulator retains only these quantities and a question identifier; i.e., neither the prompts nor the generated text are used. The data contain 1,319 \texttt{GSM8K} questions and 546 \texttt{GPQA} questions. Since Best-of-$n_i$ service requires the correctness, token, and reward records to be aligned for every included model, we remove one \texttt{GPQA} question that fails this alignment check jointly from all models whenever Best-of-$n_i$ variants are in use, leaving 545 \texttt{GPQA} questions in those experiments; no questions are removed from \texttt{GSM8K}. As a result, every retained model has the same question set and 128 recorded generations per question.

\paragraph{Base models and listed prices.}
We use the nine base models listed in Table~\ref{tab:experimental-list-prices}, together with their publicly listed prices. To determine the per-output-token price of each model, we follow the procedure of \citet{velasco2026ttcgames}: we refer to the Hugging Face list of inference providers\footnote{\url{https://huggingface.co/docs/inference-providers/index}, consulted on December 30, 2025.} and, for each model, compute the average output-token price across the listed providers that offer access to that model. More formally, we denote by $\lambda_m$ the resulting listed price of model $m$ in USD per $10^6$ output tokens. Throughout, all monetary quantities are reported in units of $10^{-6}$ USD, so that the numerical value of $\lambda_m$ coincides with the price per output token in these units.

\begin{table}[th!]
    \centering
    \small
    \begin{tabular}{lr@{\qquad}lr}
        \hline
        Model & $\lambda_m$ &
        Model & $\lambda_m$ \\
        \hline
        \texttt{Llama-3-8B}   & 0.1455 &
        \texttt{Qwen2-0.5B}   & 0.1000 \\
        \texttt{Llama-3.1-8B} & 0.1245 &
        \texttt{Qwen2-1.5B}   & 0.1000 \\
        \texttt{Llama-3.2-1B} & 0.1000 &
        \texttt{Qwen2-7B}     & 0.2000 \\
        \texttt{Llama-3.2-3B} & 0.0800 &
        \texttt{Qwen2.5-3B}   & 0.0650 \\
                              &        &
        \texttt{Qwen2.5-7B}   & 0.1465 \\
        \hline
    \end{tabular}
    \caption{Listed output-token prices used in the experiments, in USD per $10^6$ output tokens.}
    \label{tab:experimental-list-prices}
\end{table}

\paragraph{Provider construction.}
To instantiate providers with a cost--quality trade-off, we build each provider from a base model $m$ and a service multiplicity $n_i$, and write $i=(m,n_i)$ to identify the provider by this pair. A base provider, with $n_i=1$, serves a query with a single generation of model $m$, whereas a Best-of-$n_i$ provider generates $n_i$ candidate responses, incurs the generation cost of all $n_i$ candidates, and returns the candidate with the highest recorded reward score. More concretely, for each question, the simulator samples $n_i$ distinct generations uniformly without replacement from the 128 recorded generations of model $m$ and, if $n_i>1$, returns the sampled generation with the highest reward score, resolving ties uniformly through the random ordering of the sampled generations. The correctness of the response, which serves as the score $r(x_t, y_t)$ of Section~\ref{sec:model}, is the binary correctness of the returned generation, while the token count used for pricing and cost includes all $n_i$ candidates. We consider three rosters. The \emph{full} roster contains all nine base models with $n_i=1$. The \emph{ladder} roster contains
\begin{equation*}
    m\in\{
        \texttt{Llama-3-8B},
        \texttt{Llama-3.2-1B},
        \texttt{Qwen2-1.5B},
        \texttt{Qwen2.5-7B}
    \}
\end{equation*}
at each multiplicity $n_i\in\{1,2,4\}$, for a total of 12 service configurations, each of which we model as being offered by a different provider. The ladder roster therefore introduces a direct cost--quality trade-off within each base model, in addition to the variation across model families. The \emph{strong-competition} roster is used on \texttt{GPQA} only and contains the newest large model and the smallest model of each family, $\texttt{Llama-3.1-8B}$, $\texttt{Llama-3.2-1B}$, $\texttt{Qwen2.5-7B}$ and $\texttt{Qwen2-0.5B}$, all with $n_i = 1$.

\paragraph{Empirical quality, price, and cost.}
We now specify how the quantities of Section~\ref{sec:model} are instantiated from the recorded data. To this end, it is important to distinguish the listed price, the query price, and the generation cost. The \emph{listed price} $\lambda_m$ is the public per-token rate in Table~\ref{tab:experimental-list-prices}; it is known before routing and does not depend on the realized output length. We define the realized \emph{query price} $V_{i,t}$ of provider $i=(m,n_i)$ in step $t$ as the listed price $\lambda_m$ times the total number of output tokens across the $n_i$ recorded generations sampled in that step, i.e., the amount a platform using the model at its listed rate would be invoiced once the output lengths are known. Thus, a Best-of-$n_i$ query is billed for every generated candidate, not only for the returned response. Since the data do not contain the providers' true generation costs, we construct the private realized \emph{generation cost} as
\begin{equation}
    c_{i,t}
    \coloneqq
    \frac{V_{i,t}}{1+\mu},
    \qquad \mu=0.25,
    \label{eq:experimental-generation-cost}
\end{equation}
where $\mu$ is a (fixed) margin. It is worth noting that the margin $\mu$ is used only to construct the latent cost observations, and it is neither supplied to nor used by the platform. Building on the above, the empirical quality $q_i$ of provider $i$ is their expected correctness under a uniformly sampled benchmark question and the generation-sampling procedure above. For a base provider, it is the mean correctness over all recorded question--generation pairs, while for a Best-of-$n_i$ provider, it is computed exactly from the empirical joint distribution of reward scores and correctness, including uniform tie-breaking. Analogously, $\bar c_i$ is provider $i$'s expected generation cost under the same sampling procedure. Since query price and generation cost are proportional by construction, the mean query price of provider $i$ is $\bar{v}_i = (1+\mu)\bar{c}_i$. These quantities are used only for evaluation and for constructing perfect-information reference quantities; they are not given to the platform. Finally, in our primary experiments, providers bid their empirical-mean cost estimate $\hat{c}_{i,t}$, \ie, they follow the equilibrium strategy $\sigma_i^*$ of Eq.~\eqref{eq:truthful-strategy}, and payments follow the rule of Eq.~\eqref{eq:payment}.

\paragraph{Settings and platform parameters.}
For each roster and benchmark, we place the quality threshold $q_{\min}$ at the midpoint between two adjacent empirical provider qualities. Table~\ref{tab:experimental-environments} reports, for each of the resulting five settings, the exact threshold, the benchmark size $K$, the number of empirically qualified providers, and the lowest-cost qualified provider $i^*$ of Eq.~\eqref{eq:optimal-provider}.

\begin{table}[t]
    \centering
    \small
    \begin{tabular}{llrrrrl}
        \hline
        Benchmark & Roster & $N$ & $K$ & $q_{\min}$ & $|\Qcal|$ & $i^*$ \\
        \hline
        \texttt{GSM8K} & full   &  9 & 1319 & 0.7640 & 3 &
        \texttt{Qwen2.5-3B} \\
        \texttt{GPQA}  & full   &  9 &  546 & 0.0790 & 8 &
        \texttt{Qwen2-1.5B} \\
        \texttt{GSM8K} & ladder & 12 & 1319 & 0.8050 & 3 &
        \texttt{Qwen2.5-7B} \\
        \texttt{GPQA}  & ladder & 12 &  545 & 0.2396 & 4 &
        \texttt{Llama-3-8B} \\
        \texttt{GPQA}  & strong-competition &  4 &  546 & 0.1451 & 2 &
        \texttt{Llama-3.1-8B} \\
        \hline
    \end{tabular}
    \caption{Experimental settings. The benchmark size $K$ equals the number of aligned benchmark questions available to the corresponding roster.}
    \label{tab:experimental-environments}
\end{table}

We set the confidence level to $\delta=0.05$. The output-token cap entering the cost bound is fixed at $L_{\max}=512$, and the cost bound $c_{\max}$ of Section~\ref{sec:model}, which also caps every payment, is set to
\begin{equation}
    c_{\max} = L_{\max}\max_{i=(m,n_i)} n_i\lambda_m,
    \label{eq:experimental-cmax}
\end{equation}
which gives $c_{\max}\approx102$ for the full roster, $c_{\max}\approx300$ for the ladder roster and $c_{\max}\approx75$ for the strong-competition roster. Since private costs are obtained by deflating query prices by $1+\mu$, this ceiling also bounds every realized provider cost, as required by our model. Interestingly, the two full-roster settings provide qualitatively different tests of routing by listed price. The lowest listed rate belongs to \texttt{Qwen2.5-3B}, which is $i^*$ on \texttt{GSM8K} and is qualified but not $i^*$ on \texttt{GPQA}. Listed-price routing is therefore, respectively, aligned with and only partially informative about the platform's target decision.

\paragraph{Episode construction and randomization.}
At the beginning of every episode, the benchmark questions are placed in a seeded random permutation. For every step and every provider, the simulator preassigns a potential outcome using the recorded-generation sampling rule above; only the selected provider's outcome is then revealed, while the potential outcomes of nonselected providers remain unavailable to the platform. This preserves the platform's information design of Section~\ref{sec:model} while allowing different policies to be evaluated on matched potential outcomes. The first $N$ steps initialize the platform as prescribed in Section~\ref{sec:model}, selecting every provider exactly once in a seeded random order, and all subsequent steps follow the platform's selection and payment rules using the pre-step state.

\paragraph{Exploration steps.}
In every experiment of this appendix the platform runs with the exploration rule of Eq.~\eqref{eq:allocation}. At the beginning of step $t > N$, if some estimated-qualified provider $i \in \hat{\Qcal}_t$ (Eq.~\eqref{eq:qualified-set}), which we also refer to as an \emph{eligible} provider, has been selected fewer than
\begin{equation*}
    g(t) = k\,(t/N)^{\alpha}, \qquad k = 2,\; \alpha = 3/4,
\end{equation*}
times, \ie, if $\Scal_t \neq \emptyset$, the step is an exploration step: one of the under-sampled eligible providers in $\Scal_t$ is drawn uniformly at random, serves the query and is paid $c_{\max}$; otherwise it is a non-exploration step, selected and paid according to Eqs.~\eqref{eq:allocation} and~\eqref{eq:payment}. Exploration steps count toward every selection statistic, regret, generation cost, query price and amount paid; the winning bid and the payment on non-exploration steps are reported over those steps only. None of the baselines uses forced exploration.

\paragraph{Routing policies.}
Our primary comparison uses three routing policies, none of which observes private generation costs: (i) \emph{uniform routing among eligible providers}, (ii) \emph{lowest-listed-price routing among eligible providers}, and (iii) \emph{lowest-listed-price routing without quality screening}. Policy (i) constructs the same online estimated-qualified set $\hat{\Qcal}_t$ as the platform, i.e., it may use the observed scores, the selection counts, the threshold $q_{\min}$, and the confidence radius $\beta(m_{i,t})$ of Eq.~\eqref{eq:qualified-set}, and selects uniformly at random from $\hat{\Qcal}_t$; it uses neither bids nor listed prices. Policy (ii) constructs the same $\hat{\Qcal}_t$ and then minimizes the public effective listed rate
\begin{equation*}
    \bar\lambda_i\coloneqq n_i\lambda_m \qquad\text{for } i=(m,n_i),
\end{equation*}
breaking ties uniformly at random; it observes listed rates but neither provider bids nor generation costs, and the realized response length is not known at selection time. Policy (iii) is static: it minimizes $\bar\lambda_i$ over the entire roster, irrespective of observed quality, and requires only the public listed rates. It therefore provides a direct measure of the reliability lost when routing ignores the quality requirement. In addition, we consider a \emph{no-quality-filter baseline} that omits forced exploration. After initialization, it selects a provider with the lowest optimistic bidding index $b_{i,t} - c_{\max}\,\beta(m_{i,t})$ over the full roster at every step and pays the critical payment of Eq.~\eqref{eq:critical-payment}. The platform pays $\pi_t$ as in Eq.~\eqref{eq:payment}, the routing policies pay the query price of the selected provider, and both are reported as the amount paid.

\paragraph{Horizons and long-run evaluation.}
To examine longer-run behavior, we additionally run every setting for $T = 70{,}000$ queries and 30 independent runs, paired across rules, corresponding to 54 passes over the \texttt{GSM8K} benchmark and 129 passes over \texttt{GPQA}, with the final pass only partially used. At the beginning of each additional pass, the simulator draws a fresh seeded permutation of all benchmark questions and fresh recorded generations for every provider--question pair, so that the same generation may reappear in different passes; extending the horizon leaves the complete first pass unchanged. This resampling procedure approximates repeated draws from each provider's empirical outcome distribution. It does not create new questions or new model generations, and the resulting long-horizon evidence remains conditional on the observed benchmark and generation pools.

\paragraph{Metrics.}
Our main metrics are the quality regret $R_T^{\mathrm{qual}} = \sum_{t=1}^{T} (q_{\min} - q_{I_t})^+$ of Theorem~\ref{thm:quality-guarantee} and the generation-cost regret $R_T^{\mathrm{gen}} = \sum_{t=1}^{T} (\bar{c}_{I_t} - \bar{c}_{i^*})^+$ of Theorem~\ref{thm:generation-regret}, reported both cumulatively and per query. Note that both regrets are computed using the empirical provider parameters $q_i$ and $\bar{c}_i$ rather than the realized correctness labels and costs, which are subject to sampling noise. We report realized accuracy separately as
\begin{equation*}
    \widehat{\mathrm{Acc}}_T = \frac{1}{T}\sum_{t=1}^T r(x_t, y_t),
\end{equation*}
together with the share of unqualified selections, 
\begin{equation*}
    \frac{1}{T}\sum_{t=1}^T \mathbbm{1}\{I_t\notin \Qcal\},
\end{equation*}
the selection share of $i^*$, provider-level selection counts, and whether $i^*$ is the unique most selected provider. The economic outcomes comprise realized generation cost per query, query price per query, and the amount paid per query under the policy-specific convention above. For the platform and its no-quality-filter baseline, we additionally report the winning bid, the total payment on non-exploration steps, the cumulative excess payment $E_T^{\mathrm{pay}} = \sum_{t=1}^{T} (\pi_t - \bar{c}_{(2)})^+$ bounded in Theorem~\ref{thm:excess-payment-general}, and the provider payoff $\sum_{t:I_t=i}\bigl(\pi_t-c_{i,t}\bigr)$. Where useful, these quantities are reported both over the full episode and after initialization, since initialization pays $c_{\max}$ by construction. Payment diagnostics compare the realized payments on non-exploration steps with the upper bound $\pi_t \leq \min\{c_{\max},\, \bar{c}_{(2)} + c_{\max}\,\beta(m_{I_t,t})\}$ of Lemma~\ref{lem:ordinary-payment} and with the lower bound $\pi_t \geq \max\{0,\, \bar{c}_{I_t} - c_{\max}\,\beta(m_{I_t,t})\}$, which holds on $\mathcal{E}$ because, by Proposition~\ref{prop:selection-threshold}, the selected provider bids $\hat{c}_{I_t,t} \leq P_{I_t,t}$. We further report the realized frequency of the good event $\mathcal{E}$ of Lemma~\ref{lem:good-event} and compare the post-initialization selections of providers $i\neq i^*$ with the corresponding bounds $M_i$ of Lemma~\ref{lem:unqualified-elimination} for unqualified providers and $\max\{L_i,\, \lceil g(T)\rceil - 1\}$ of Lemma~\ref{lem:qualified-suboptimal-selection} for qualified ones, where $L_i$ is defined in Lemma~\ref{lem:ordinary-selection-threshold}. When reporting $B_{\mathrm{id}}(T)$ at $T=70{,}000$, we use the tighter exploration cap $\lceil g(T)\rceil-1$ from Lemma~\ref{lem:qualified-suboptimal-selection}. To compute $T_{\mathrm{id}}$, we use $B_{\mathrm{id}}(T)$ as stated in Theorem~\ref{thm:sample-complexity}. For the long-horizon experiment, we additionally report the size of the estimated-qualified set $\hat{\Qcal}_t$, the probability that $i^*$ is the unique most selected provider, and moving-average selection and payment statistics with window size $K$.

\paragraph{Aggregation and uncertainty.}
In the one-pass tables, $\pm$ denotes the half-width of the 95\% confidence interval for the mean across runs. For one-pass figures, points and selection shares are shown with normal-approximation 95\% intervals, $\overline{x}\pm 1.96\,s_x/\sqrt{R}$, where $R$ is the number of runs and $s_x$ is their sample standard deviation. Long-horizon trajectory figures use the pointwise median across runs as the central trajectory and the empirical 10th--90th percentile range as a band. Long-horizon checkpoint tables report means across runs, while identification probabilities are the corresponding empirical fractions of runs.

\clearpage

\renewcommand{\resultsdir}{analysis/variants/count_scale2_alpha0.75/gsm8k}
\section{GSM8K weak-competition experiments}\label{app:gsm8k}
This section reports the complete \texttt{GSM8K} results for the full and ladder rosters; Appendix~\ref{app:experimental-details} provides the experimental details. The platform runs with the exploration rule described there, $k = 2$ and $\alpha = 3/4$. One-pass results are averaged over $30$ independent runs, paired across rules and reported with $95\%$ confidence intervals, while long-horizon trajectories report the median and $10$--$90\%$ range over $30$ independent runs. Monetary quantities are per query unless stated otherwise. A moving average has window size $K = 1{,}319$, one pass over the benchmark, so that it reflects a rule's current behaviour rather than its whole history. For the platform and its no-quality-filter baseline the amount paid is the payment $\pi_t$; for the routing baselines it is the query price of the selected provider. On the full roster $\bar c_{(2)} = 34.4$ lies above the mean query price of $i^*$, $\bar{v}_{i^*} = 20.1$, so this is the weak-competition case in which a payment that approaches $\bar c_{(2)}$ exceeds the listed price of $i^*$.

\subsection{Providers and environment}\label{app:gsm8k-env}

\begin{figure}[!htbp]
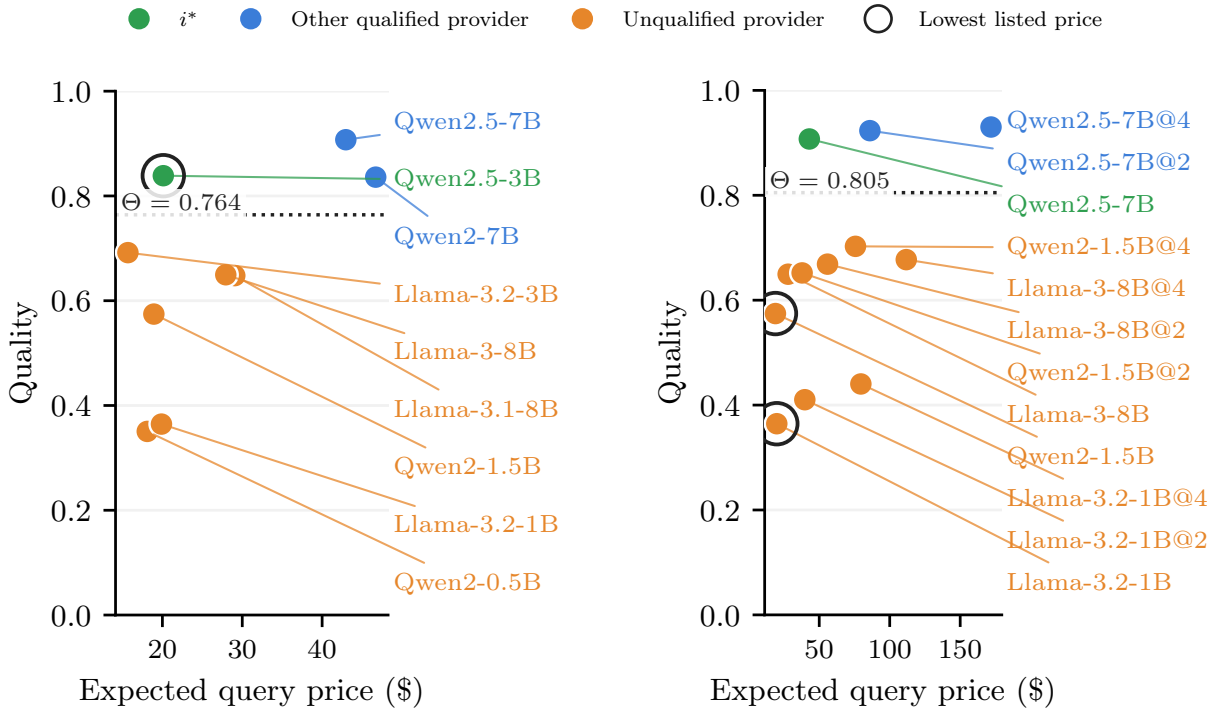

  \centering
  \sharedlegend{landscape}\\[4pt]
  \rosterpair{quality_vs_query_price}{Quality against mean query price}
  \caption{\texttt{GSM8K} providers. Each point is one provider at its quality $q_i$ and mean query price; the dotted line is the quality threshold $q_{\min}$, and the ring marks the provider with the lowest listed price. Monetary quantities are in units of $10^{-6}$ USD. On the full roster the lowest listed price belongs to $i^*$; on the ladder roster it is shared by two unqualified providers.}
  \label{fig:gsm8k-landscape}
\end{figure}

\begin{table}[!htbp]
  \centering\small
  \caption{\texttt{GSM8K} setting constants for the two rosters. Costs and prices are expected values
  per query.}
  \label{tab:gsm8k-env}
  \begin{tabular}{lrr}
    \toprule
    Quantity & Full roster & Ladder roster \\
    \midrule
    Providers $N$ & 9 & 12 \\
    Benchmark size $K$ (queries in one pass) & 1,319 & 1,319 \\
    Quality threshold $q_{\min}$ & 0.764 & 0.805 \\
    Qualified providers $|\Qcal|$ & 3 & 3 \\
    $i^*$ & Qwen2.5-3B & Qwen2.5-7B \\
    $\bar c_{i^*}$ & 16.1 & 34.4 \\
    $\bar{v}_{i^*}$ & 20.1 & 43.0 \\
    $\bar c_{(2)}$ & 34.4 & 68.7 \\
    $\bar{c}_{(2)} - \bar{c}_{i^*}$ & 18.3 & 34.4 \\
    Lowest listed price & Qwen2.5-3B & Llama-3.2-1B, Qwen2-1.5B \\
    \quad qualified & yes & no \\
    \quad equals $i^*$ & yes & no \\
    $c_{\max}$ & 102.4 & 300.0 \\
    $\sum_{i \notin \Qcal} M_i$ & 28,296 & 23,132 \\
    $\sum_{i \in \Qcal \setminus \{i^*\}} L_i$ & 3,464 & 6,012 \\
    $\lceil g(T) \rceil - 1$ & 1,656 & 1,334 \\
    $B_{\mathrm{id}}(T)$ & 31,977 & 30,010 \\
    $T_{\mathrm{id}}$ & 63,740 & 59,735 \\
    \bottomrule
  \end{tabular}
\end{table}

\begin{table}[!htbp]
  \centering\footnotesize
  \setlength{\tabcolsep}{4pt}
  \caption{\texttt{GSM8K} providers, sorted by mean generation cost. Listed price is in USD per
  $10^6$ output tokens; tokens are the mean number of generated tokens per query (summed over
  the $n_i$ generations of a Best-of-$n_i$ provider); query price and $\bar c_i$ are expected values
  per query. The gap is $\varepsilon_i = q_{\min} - q_i$ for an unqualified provider and
  $\Delta_i = \bar c_i - \bar c_{i^*}$ for a qualified one, with the corresponding selection bound
  $M_i$ from~\eqref{eq:unqualified-selection-threshold} or $L_i$ from~\eqref{eq:ordinary-selection-threshold}; where the exploration-step cap $\lceil g(T) \rceil - 1$ at $T = 70{,}000$ exceeds $L_i$, the applicable cap $\max\{L_i, \lceil g(T) \rceil - 1\}$ of Lemma~\ref{lem:qualified-suboptimal-selection} follows it.}
  \label{tab:gsm8k-providers}
  \begin{tabular}{lrrrrrcll}
    \toprule
    Provider & Listed price & Tokens & Query price & $\bar c_i$ & $q_i$ & Qualified & Gap & Bound \\
    \midrule
    \multicolumn{9}{l}{\textit{Full roster}} \\
    Llama-3.2-3B & 0.08 & 196 & 15.7 & 12.5 & 0.692 & no & $\varepsilon_i = 0.072$ & $M = 15{,}141$ \\
    Qwen2-0.5B & 0.1 & 181 & 18.1 & 14.5 & 0.351 & no & $\varepsilon_i = 0.413$ & $M = 299$ \\
    Qwen2-1.5B & 0.1 & 189 & 18.9 & 15.1 & 0.574 & no & $\varepsilon_i = 0.190$ & $M = 1{,}770$ \\
    Llama-3.2-1B & 0.1 & 199 & 19.9 & 15.9 & 0.364 & no & $\varepsilon_i = 0.400$ & $M = 324$ \\
    Qwen2.5-3B ($i^*$) & 0.065 & 309 & 20.1 & 16.1 & 0.839 & yes & -- & -- \\
    Llama-3-8B & 0.1455 & 192 & 27.9 & 22.3 & 0.650 & no & $\varepsilon_i = 0.114$ & $M = 5{,}466$ \\
    Llama-3.1-8B & 0.1245 & 233 & 29.1 & 23.3 & 0.648 & no & $\varepsilon_i = 0.116$ & $M = 5{,}296$ \\
    Qwen2.5-7B & 0.1465 & 293 & 43.0 & 34.4 & 0.907 & yes & $\Delta_i = 18.3$ & $L = 2{,}025$ \\
    Qwen2-7B & 0.2 & 234 & 46.7 & 37.4 & 0.836 & yes & $\Delta_i = 21.3$ & $L = 1{,}439$; cap 1{,}656 \\
    \midrule
    \multicolumn{9}{l}{\textit{Ladder roster}} \\
    Qwen2-1.5B & 0.1 & 189 & 18.9 & 15.1 & 0.574 & no & $\varepsilon_i = 0.231$ & $M = 1{,}159$ \\
    Llama-3.2-1B & 0.1 & 199 & 19.9 & 15.9 & 0.364 & no & $\varepsilon_i = 0.441$ & $M = 264$ \\
    Llama-3-8B & 0.1455 & 192 & 27.9 & 22.3 & 0.650 & no & $\varepsilon_i = 0.155$ & $M = 2{,}809$ \\
    Qwen2-1.5B@2 & 0.1 & 379 & 37.9 & 30.3 & 0.652 & no & $\varepsilon_i = 0.153$ & $M = 2{,}910$ \\
    Llama-3.2-1B@2 & 0.1 & 398 & 39.8 & 31.8 & 0.410 & no & $\varepsilon_i = 0.395$ & $M = 340$ \\
    Qwen2.5-7B ($i^*$) & 0.1465 & 293 & 43.0 & 34.4 & 0.907 & yes & -- & -- \\
    Llama-3-8B@2 & 0.1455 & 384 & 55.9 & 44.7 & 0.669 & no & $\varepsilon_i = 0.136$ & $M = 3{,}768$ \\
    Qwen2-1.5B@4 & 0.1 & 757 & 75.7 & 60.6 & 0.703 & no & $\varepsilon_i = 0.102$ & $M = 7{,}118$ \\
    Llama-3.2-1B@4 & 0.1 & 795 & 79.5 & 63.6 & 0.440 & no & $\varepsilon_i = 0.365$ & $M = 408$ \\
    Qwen2.5-7B@2 & 0.1465 & 587 & 85.9 & 68.7 & 0.923 & yes & $\Delta_i = 34.4$ & $L = 5{,}544$ \\
    Llama-3-8B@4 & 0.1455 & 768 & 111.7 & 89.4 & 0.677 & no & $\varepsilon_i = 0.128$ & $M = 4{,}356$ \\
    Qwen2.5-7B@4 & 0.1465 & 1,173 & 171.9 & 137.5 & 0.930 & yes & $\Delta_i = 103.1$ & $L = 468$; cap 1{,}334 \\
    \bottomrule
  \end{tabular}
\end{table}
\clearpage
\subsection{One-pass behavior}\label{app:gsm8k-one-pass}

\begin{table}[!htbp]
  \centering\scriptsize
  \setlength{\tabcolsep}{3pt}
  \caption{\texttt{GSM8K} one-pass outcomes, mean $\pm$ $95\%$ half-width over $30$ independent runs, paired across rules.
  Generation cost and amount paid are per query; the amount paid is the payment $\pi_t$ for the platform and its no-quality-filter baseline, and the query price for the other rules.}
  \label{tab:gsm8k-one-pass}
  \begin{tabular}{Zccccccc}
    \toprule
    Policy & Unqualified share & $R_T^{\mathrm{qual}}/T$ & $R_T^{\mathrm{gen}}/T$ & Accuracy & Share of $i^*$ & Generation cost & Amount paid \\
    \midrule
    \multicolumn{8}{l}{\textit{Full roster}} \\
    Platform & $0.692 \pm 0.003$ & $0.099 \pm 0.002$ & $3.81 \pm 0.02$ & $0.692 \pm 0.004$ & $0.179 \pm 0.003$ & $18.7 \pm 0.0$ & $37.0 \pm 0.3$ \\
    Uniform among eligible & $0.598 \pm 0.007$ & $0.094 \pm 0.002$ & $7.05 \pm 0.12$ & $0.707 \pm 0.003$ & $0.136 \pm 0.004$ & $22.5 \pm 0.1$ & $28.1 \pm 0.2$ \\
    Lowest listed price among eligible & $0.005 \pm 0.000$ & $0.001 \pm 0.000$ & $0.04 \pm 0.00$ & $0.837 \pm 0.002$ & $0.994 \pm 0.000$ & $16.1 \pm 0.0$ & $20.1 \pm 0.0$ \\
    Lowest listed price, no quality filter & $0.005 \pm 0.000$ & $0.001 \pm 0.000$ & $0.04 \pm 0.00$ & $0.837 \pm 0.002$ & $0.994 \pm 0.000$ & $16.1 \pm 0.0$ & $20.1 \pm 0.0$ \\
    No quality filter & $0.799 \pm 0.003$ & $0.182 \pm 0.001$ & $2.30 \pm 0.03$ & $0.600 \pm 0.004$ & $0.135 \pm 0.002$ & $17.2 \pm 0.0$ & $17.9 \pm 0.1$ \\
    \midrule
    \multicolumn{8}{l}{\textit{Ladder roster}} \\
    Platform & $0.771 \pm 0.005$ & $0.148 \pm 0.003$ & $13.69 \pm 0.15$ & $0.684 \pm 0.003$ & $0.125 \pm 0.005$ & $42.1 \pm 0.4$ & $118.1 \pm 1.5$ \\
    Uniform among eligible & $0.690 \pm 0.005$ & $0.126 \pm 0.003$ & $24.81 \pm 0.37$ & $0.711 \pm 0.003$ & $0.103 \pm 0.003$ & $55.3 \pm 0.5$ & $69.2 \pm 0.7$ \\
    Lowest listed price among eligible & $0.446 \pm 0.031$ & $0.087 \pm 0.005$ & $0.20 \pm 0.00$ & $0.775 \pm 0.007$ & $0.552 \pm 0.031$ & $28.2 \pm 0.4$ & $35.2 \pm 0.5$ \\
    Lowest listed price, no quality filter & $0.998 \pm 0.000$ & $0.333 \pm 0.001$ & $0.20 \pm 0.00$ & $0.477 \pm 0.004$ & $0.001 \pm 0.000$ & $15.8 \pm 0.1$ & $19.8 \pm 0.1$ \\
    No quality filter & $0.855 \pm 0.001$ & $0.220 \pm 0.000$ & $7.57 \pm 0.09$ & $0.600 \pm 0.003$ & $0.091 \pm 0.001$ & $33.5 \pm 0.1$ & $36.0 \pm 0.1$ \\
    \bottomrule
  \end{tabular}
\end{table}

\begin{figure}[!htbp]
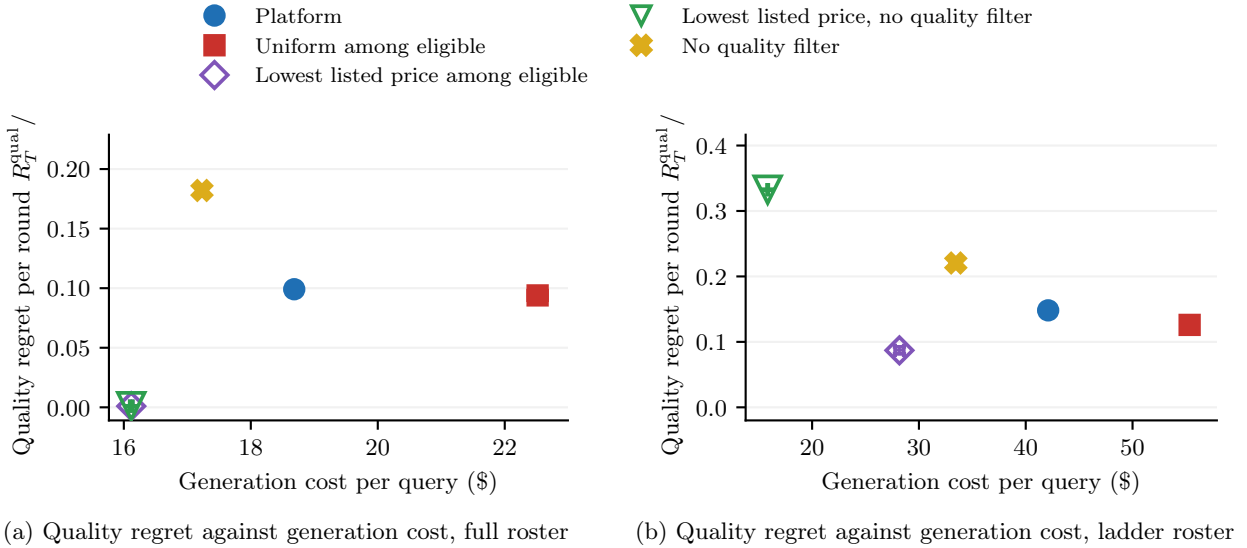

  \centering
  \sharedlegend{policies_markers}\\[4pt]
  \rosterpair{one_pass_quality_regret_vs_generation_cost}{Quality regret against generation cost}
  \caption{\texttt{GSM8K} one pass: quality regret per step against generation cost per query for every routing rule, mean with $95\%$ confidence intervals across runs. Monetary quantities are in units of $10^{-6}$ USD. The two listed-price rules coincide on the full roster, where the lowest listed price is $i^*$.}
  \label{fig:gsm8k-one-pass}
\end{figure}

\begin{figure}[!htbp]
  \centering
  \rosterpairb{one_pass_mechanism_selection_shares}{Platform selection shares}
  \caption{\texttt{GSM8K} one pass: share of steps the platform assigned to each provider, exploration steps included, mean with $95\%$ confidence intervals across runs. Green marks $i^*$, blue the other qualified providers and orange the unqualified providers; the platform never observes these labels.}
  \label{fig:gsm8k-selection-shares}
\end{figure}
\clearpage
\subsection{Long-horizon learning and identification}\label{app:gsm8k-long}

\begin{figure}[!htbp]
  \centering
  \sharedlegend{policies_lines}\\[4pt]
  \rosterpairs{long_horizon_rolling_i_star_share}{Share of $i^*$}\\[-2pt]
  \rosterpairs{long_horizon_unqualified_share}{Share of unqualified selections}\\[-2pt]
  \rosterpairs{long_horizon_i_star_identification_probability}{$i^*$ is the unique most selected provider}
  \caption{\texttt{GSM8K} long horizon: identification of $i^*$ under every routing rule. Shares are moving averages with window size $K$, exploration steps included, median over $30$ independent runs with a $10$--$90\%$ band; the bottom row is the fraction of runs in which $i^*$ has strictly more selections than every other provider. Vertical lines mark one pass and the identification horizon $T_{\mathrm{id}}$. On the full roster the two listed-price rules coincide.}
  \label{fig:gsm8k-identification}
\end{figure}

\begin{figure}[!htbp]
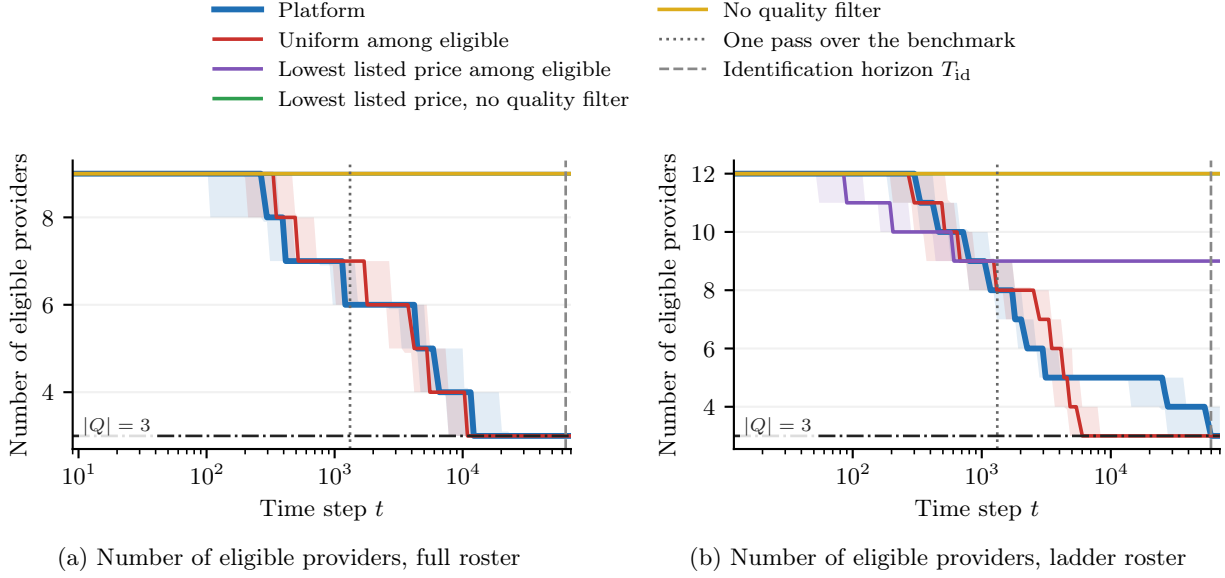

  \centering
  \sharedlegend{policies_lines}\\[4pt]
  \rosterpair{long_horizon_eligible_providers}{Number of eligible providers}
  \caption{\texttt{GSM8K} long horizon: the number of providers not yet ruled out by the quality filter under every rule, median over $30$ independent runs with a $10$--$90\%$ band. The two rules without a filter stay at $N$; the listed-price rule inside a quality set learned with the same filter rules out only the providers it samples before settling on $i^*$; the dash-dotted line marks $|\Qcal|$.}
  \label{fig:gsm8k-eligible}
\end{figure}

\begin{figure}[!htbp]
  \centering
  \sharedlegend{policies_lines}\\[4pt]
  \rosterpair{long_horizon_generation_cost_per_query}{Generation cost per query}\\[2pt]
  \rosterpair{long_horizon_amount_paid_per_query}{Amount paid per query}
  \caption{\texttt{GSM8K} long horizon: generation cost per query as a moving average with window size $K$ (top) and amount paid per query as a moving average with window size $K$ over the steps after initialization (bottom), under every rule, median over $30$ independent runs with a $10$--$90\%$ band. Monetary quantities are in units of $10^{-6}$ USD. The reference lines mark $\bar c_{i^*}$, the mean query price $\bar{v}_{i^*}$ of $i^*$ and $\bar c_{(2)}$. The platform's payment $\pi_t$ includes its exploration steps at $c_{\max}$, which are most of the steps early in the run; the vertical axis of the bottom panels is cut, so the platform's curve enters from above.}
  \label{fig:gsm8k-routing}
\end{figure}

\begin{table}[!htbp]
  \centering\scriptsize
  \setlength{\tabcolsep}{3pt}
  \caption{\texttt{GSM8K} long horizon: outcomes at $T = 70{,}000$, means over $30$ independent runs. Shares of $i^*$ and of unqualified selections are cumulative unless marked as moving averages and include the exploration steps; ``$i^*$ most selected'' is the fraction of runs in which $i^*$ is the unique most selected provider; $|\hat{\Qcal}_t|$ is the number of eligible providers at $T$. Generation cost and amount paid are cumulative per query; for the platform the amount paid includes the exploration steps.}
  \label{tab:gsm8k-long}
  \begin{tabular}{Yrrrrrrrrr}
    \toprule
    Policy & \shortstack{Share\\of $i^*$} & \shortstack{Moving\\avg.} & \shortstack{Unqualified\\share} & $R_t^{\mathrm{qual}}/t$ & $R_t^{\mathrm{gen}}/t$ & \shortstack{Generation\\cost} & \shortstack{Amount\\paid} & \shortstack{$i^*$ most\\selected} & $|\hat{\Qcal}_t|$ \\
    \midrule
    \multicolumn{10}{l}{\textit{Full roster}} \\
    Platform & 0.899 & 0.964 & 0.054 & 0.0056 & 1.077 & 17.1 & 28.2 & 1.00 & 3.0 \\
    Uniform among eligible & 0.314 & 0.333 & 0.057 & 0.0059 & 12.580 & 28.6 & 35.7 & 0.30 & 3.0 \\
    Lowest listed price among eligible & 1.000 & 1.000 & 0.000 & 0.0000 & 0.001 & 16.1 & 20.1 & 1.00 & 9.0 \\
    Lowest listed price, no quality filter & 1.000 & 1.000 & 0.000 & 0.0000 & 0.001 & 16.1 & 20.1 & 1.00 & 9.0 \\
    No quality filter & 0.061 & 0.029 & 0.935 & 0.1466 & 0.219 & 13.8 & 13.8 & 0.00 & 9.0 \\
    \midrule
    \multicolumn{10}{l}{\textit{Ladder roster}} \\
    Platform & 0.914 & 0.970 & 0.048 & 0.0070 & 3.643 & 37.9 & 50.5 & 1.00 & 3.2 \\
    Uniform among eligible & 0.319 & 0.334 & 0.044 & 0.0064 & 44.670 & 78.9 & 98.6 & 0.33 & 3.0 \\
    Lowest listed price among eligible & 0.992 & 1.000 & 0.008 & 0.0016 & 0.004 & 34.3 & 42.8 & 1.00 & 9.0 \\
    Lowest listed price, no quality filter & 0.000 & 0.000 & 1.000 & 0.3355 & 0.004 & 15.5 & 19.4 & 0.00 & 12.0 \\
    No quality filter & 0.022 & 0.007 & 0.974 & 0.2934 & 0.609 & 18.5 & 18.5 & 0.00 & 12.0 \\
    \bottomrule
  \end{tabular}
\end{table}

\begin{table}[!htbp]
  \centering\small
  \caption{\texttt{GSM8K} exploration steps of the platform, means over $30$ independent runs, over one pass and over the long run of $T = 70{,}000$ queries. Shares are of the steps after initialization; payments and costs are per step of the corresponding horizon, the long-run values as moving averages with window size $K$ at $T$. The last two rows give the platform's moving-average payment $\pi_t$ on non-exploration steps and over all steps at $T$, per query.}
  \label{tab:gsm8k-exploration}
  \begin{tabular}{lrrrr}
    \toprule
    Quantity & \multicolumn{2}{c}{Full roster} & \multicolumn{2}{c}{Ladder roster} \\
     & one pass & long run & one pass & long run \\
    \midrule
    Exploration steps (mean) & 314 & 3,458 & 422 & 3,348 \\
    Share of exploration steps after initialization & 0.240 & 0.049 & 0.323 & 0.048 \\
    Share of exploration steps at $T$ (moving avg.) & -- & 0.036 & -- & 0.030 \\
    Exploration-step payment per step & 24.40 & 3.73 & 95.92 & 9.14 \\
    Exploration-step generation cost per step & 6.58 & 1.27 & 22.26 & 3.01 \\
    Share of exploration steps given to $i^*$ & 0.057 & 0.005 & 0.047 & 0.006 \\
    Share of exploration steps given to unqualified providers & 0.408 & 0.037 & 0.641 & 0.307 \\
    Payment $\pi_t$ at $T$, non-exploration steps (moving avg.) & -- & 27.52 & -- & 44.27 \\
    Payment $\pi_t$ at $T$, all steps (moving avg.) & -- & 30.25 & -- & 52.09 \\
    \bottomrule
  \end{tabular}
\end{table}
\clearpage
\subsection{Regret and theoretical bounds}\label{app:gsm8k-bounds}

\begin{figure}[!htbp]
  \centering
  \sharedlegend{policies_lines}\\[4pt]
  \rosterpair{long_horizon_quality_regret_per_round}{Average quality regret}\\[2pt]
  \rosterpair{long_horizon_generation_regret_per_round}{Average generation regret}
  \caption{\texttt{GSM8K} long horizon: average quality regret $R_t^{\mathrm{qual}}/t$ and average generation regret $R_t^{\mathrm{gen}}/t$ of every routing rule, median over $30$ independent runs with a $10$--$90\%$ band, on logarithmic axes. Monetary quantities are in units of $10^{-6}$ USD. The platform's regrets include its exploration steps.}
  \label{fig:gsm8k-regret}
\end{figure}

\begin{figure}[!htbp]
  \centering
  \sharedlegend{bounds}\\[4pt]
  \rosterpair{long_horizon_quality_regret_vs_bound}{Quality regret}\\[2pt]
  \rosterpair{long_horizon_generation_regret_vs_bound}{Generation regret}\\[2pt]
  \rosterpair{long_horizon_excess_payment_vs_bound}{Excess payment}
  \caption{\texttt{GSM8K} long horizon: the platform's cumulative quality regret, generation regret and excess payment against their bounds, evaluated at every time step. Monetary quantities are in units of $10^{-6}$ USD. The bounds are those for the platform with exploration steps, each with an anytime branch and a gap-dependent branch; the excess-payment bound \eqref{eq:excess-payment-general} without its exploration term has a single branch. The regrets include the exploration steps; the excess payment excludes them.}
  \label{fig:gsm8k-bounds}
\end{figure}

\begin{table}[!htbp]
  \centering\small
  \caption{\texttt{GSM8K} theoretical diagnostics of the platform over the long run ($T = 70{,}000$, $30$ independent runs). Realized quantities are means over runs; each regret bound is the smaller of its two branches at $T$, for the platform with exploration steps. Selections and regrets include the exploration steps and are compared with $B_{\mathrm{id}}(T)$ and the bounds that account for them; the excess payment is given without the exploration steps, against \eqref{eq:excess-payment-general} without its exploration term, and with them, against \eqref{eq:excess-payment-general}. The good event is checked at every selection count of every provider; the payment bounds are those of Appendix~\ref{app:experimental-details}, checked on non-exploration steps.}
  \label{tab:gsm8k-diagnostics}
  \begin{tabular}{lrr}
    \toprule
    Quantity & Full roster & Ladder roster \\
    \midrule
    $B_{\mathrm{id}}(T)$ & 31,977 & 30,010 \\
    $T_{\mathrm{id}}$ & 63,740 & 59,735 \\
    Horizon $T$ of the long run & 70,000 & 70,000 \\
    Selections of providers other than $i^*$ after initialization (mean) & 7,094 & 6,034 \\
    \quad ratio to $B_{\mathrm{id}}(T)$ & 0.222 & 0.201 \\
    Runs in which $i^*$ is the most selected provider at $t = 5{,}000$ & 0.20 & 1.00 \\
    \quad at $t = 10{,}000$ & 1.00 & 1.00 \\
    \quad at $t = T_{\mathrm{id}}$ & 1.00 & 1.00 \\
    \quad at $t = T$ & 1.00 & 1.00 \\
    $R_T^{\mathrm{qual}}$, realized (mean) & 392 & 487 \\
    \quad bound \eqref{eq:quality-regret-bound} & 2,924 & 3,349 \\
    \quad ratio & 0.134 & 0.145 \\
    $R_T^{\mathrm{gen}}$, realized (mean) & 75,384 & 254,988 \\
    \quad bound \eqref{eq:generation-regret-bound} & 144,428 & 805,387 \\
    \quad ratio & 0.522 & 0.317 \\
    $E_T^{\mathrm{pay}}$ without the exploration steps, realized (mean) & 612 & 2,956 \\
    \quad bound \eqref{eq:excess-payment-general} without its exploration term & 623,706 & 2,121,117 \\
    \quad ratio & 0.001 & 0.001 \\
    $E_T^{\mathrm{pay}}$ incl. exploration steps, realized (mean) & 235,875 & 777,199 \\
    \quad bound \eqref{eq:excess-payment-general} & 1,637,592 & 5,823,569 \\
    \quad ratio & 0.144 & 0.133 \\
    Good event held, one pass (runs) & 30 of 30 & 30 of 30 \\
    Good event held, long run, platform & 30 of 30 & 30 of 30 \\
    Good event held, long run, no quality filter & 30 of 30 & 30 of 30 \\
    Payment-bound violations, one pass (steps) & 0 & 0 \\
    \bottomrule
  \end{tabular}
\end{table}

\begin{table}[!htbp]
  \centering\footnotesize
  \setlength{\tabcolsep}{4pt}
  \caption{\texttt{GSM8K} long horizon: selections of every provider other than $i^*$ by the platform after initialization, over $30$ independent runs of $T = 70{,}000$ queries, against its cap, $M_i$ for an unqualified provider and $\max\{L_i, \lceil g(T) \rceil - 1\}$ for a qualified one (Lemma~\ref{lem:qualified-suboptimal-selection}), shown as $\lceil g(T) \rceil - 1$ where the exploration-step cap exceeds $L_i$; the last row of each roster compares the total with $B_{\mathrm{id}}(T)$. Providers are sorted by mean generation cost, as in Table~\ref{tab:gsm8k-providers}. Selections are the mean and the maximum over runs of the number of post-initialization selections, exploration steps included; the exploration column is the mean number of exploration steps among them; the last selection is the mean over runs of the step of the provider's last selection; the ratio is the mean number of selections divided by the cap.}
  \label{tab:gsm8k-selections}
  \begin{tabular}{llrrrrlr}
    \toprule
    Provider & Gap & \multicolumn{2}{c}{Selections} & Exploration & Last selection & Bound & Ratio \\
    & & mean & max & (mean) & (mean step) & & \\
    \midrule
    \multicolumn{8}{l}{\textit{Full roster}} \\
    Llama-3.2-3B & $\varepsilon_i = 0.072$ & 2,008 & 3,240 & 16 & 4,182 & $M = 15{,}141$ & 0.133 \\
    Qwen2-0.5B & $\varepsilon_i = 0.413$ & 35 & 59 & 16 & 280 & $M = 299$ & 0.118 \\
    Qwen2-1.5B & $\varepsilon_i = 0.190$ & 221 & 361 & 17 & 1,184 & $M = 1{,}770$ & 0.125 \\
    Llama-3.2-1B & $\varepsilon_i = 0.400$ & 51 & 142 & 17 & 391 & $M = 324$ & 0.158 \\
    Llama-3-8B & $\varepsilon_i = 0.114$ & 745 & 1,271 & 30 & 8,845 & $M = 5{,}466$ & 0.136 \\
    Llama-3.1-8B & $\varepsilon_i = 0.116$ & 721 & 1,156 & 32 & 11,280 & $M = 5{,}296$ & 0.136 \\
    Qwen2.5-7B & $\Delta_i = 18.3$ & 1,656 & 1,656 & 1,656 & 69,977 & $L = 2{,}025$ & 0.818 \\
    Qwen2-7B & $\Delta_i = 21.3$ & 1,656 & 1,656 & 1,656 & 69,978 & $\lceil g(T) \rceil - 1 = 1{,}656$ & 1.000 \\
    All providers other than $i^*$ & & 7,094 & & 3,440 &  & $B_{\mathrm{id}}(T) = 31{,}977$ & 0.222 \\
    \midrule
    \multicolumn{8}{l}{\textit{Ladder roster}} \\
    Qwen2-1.5B & $\varepsilon_i = 0.231$ & 181 & 340 & 16 & 963 & $M = 1{,}159$ & 0.157 \\
    Llama-3.2-1B & $\varepsilon_i = 0.441$ & 37 & 76 & 15 & 337 & $M = 264$ & 0.140 \\
    Llama-3-8B & $\varepsilon_i = 0.155$ & 401 & 820 & 17 & 1,746 & $M = 2{,}809$ & 0.143 \\
    Qwen2-1.5B@2 & $\varepsilon_i = 0.153$ & 426 & 599 & 18 & 2,144 & $M = 2{,}910$ & 0.146 \\
    Llama-3.2-1B@2 & $\varepsilon_i = 0.395$ & 45 & 103 & 18 & 482 & $M = 340$ & 0.132 \\
    Llama-3-8B@2 & $\varepsilon_i = 0.136$ & 490 & 1,015 & 25 & 3,231 & $M = 3{,}768$ & 0.130 \\
    Qwen2-1.5B@4 & $\varepsilon_i = 0.102$ & 1,059 & 1,334 & 215 & 49,140 & $M = 7{,}118$ & 0.149 \\
    Llama-3.2-1B@4 & $\varepsilon_i = 0.365$ & 58 & 106 & 46 & 1,013 & $M = 408$ & 0.141 \\
    Qwen2.5-7B@2 & $\Delta_i = 34.4$ & 1,334 & 1,334 & 954 & 69,935 & $L = 5{,}544$ & 0.241 \\
    Llama-3-8B@4 & $\varepsilon_i = 0.128$ & 669 & 1,227 & 669 & 28,411 & $M = 4{,}356$ & 0.154 \\
    Qwen2.5-7B@4 & $\Delta_i = 103.1$ & 1,334 & 1,334 & 1,334 & 69,934 & $\lceil g(T) \rceil - 1 = 1{,}334$ & 1.000 \\
    All providers other than $i^*$ & & 6,034 & & 3,328 &  & $B_{\mathrm{id}}(T) = 30{,}010$ & 0.201 \\
    \bottomrule
  \end{tabular}
\end{table}
\clearpage
\subsection{Payments and provider outcomes}\label{app:gsm8k-payments}
% \paragraph{What is evaluated.}
% Figure~\ref{fig:gsm8k-economic} follows the mechanism's realized generation cost, winning bid, critical payment, query price and amount paid over the long run, Figure~\ref{fig:gsm8k-payment-bounds} shows the critical payment of one long-horizon repetition inside the payment bounds of Appendix~\ref{app:experimental-details}, and Tables~\ref{tab:gsm8k-payments} and~\ref{tab:gsm8k-payoffs} report the one-pass payment quantities and every provider's payoff.

\begin{figure}[!htbp]
  \centering
  \sharedlegend{economic}\\[4pt]
  \rosterpair{long_horizon_mechanism_economic_quantities}{Cost, bid, payment and query price}
  \caption{\texttt{GSM8K} long horizon: the platform's realized generation cost, the query price of the same generations and the payment $\pi_t$, as a moving average with window size $K$ over the steps after initialization, and its winning bid and payment on non-exploration steps in the same window, median over $30$ independent runs with a $10$--$90\%$ band. Monetary quantities are in units of $10^{-6}$ USD. The dotted, dash-dotted and dashed lines mark $\bar c_{i^*}$, $\bar c_{(2)}$ and the mean query price $\bar{v}_{i^*}$ of $i^*$. The vertical axis is cut; $\pi_t$ enters from above, since early in the run most steps are exploration steps and paid $c_{\max}$.}
  \label{fig:gsm8k-economic}
\end{figure}

\begin{figure}[!htbp]
  \centering
  \sharedlegend{payment_bounds}\\[4pt]
  \rosterpair{payment_bounds_long_horizon}{Payment on non-exploration steps and its bounds}
  \caption{\texttt{GSM8K} long horizon: the payment $\pi_t$ of the first run on non-exploration steps, sampled at logarithmically spaced time steps, inside the payment bounds of Appendix~\ref{app:experimental-details} for the provider selected in that step; sampled steps that were exploration steps are left blank. Monetary quantities are in units of $10^{-6}$ USD. The bounds depend on the selected provider's cost and selection count, and close in as the counts grow.}
  \label{fig:gsm8k-payment-bounds}
\end{figure}

\begin{table}[!htbp]
  \centering\small
  \caption{\texttt{GSM8K} one pass: the platform's economic quantities, mean $\pm$ $95\%$ half-width over $30$ independent runs. The payment $\pi_t$ is over all steps; the winning bid, the payment on non-exploration steps and the second excess-payment row are over non-exploration steps only. A provider's payoff is its total payment minus its total generation cost over the pass.}
  \label{tab:gsm8k-payments}
  \begin{tabular}{lrr}
    \toprule
    Quantity & Full roster & Ladder roster \\
    \midrule
    Generation cost per query, whole pass & $18.68 \pm 0.05$ & $42.10 \pm 0.36$ \\
    Query price per query, whole pass & $23.35 \pm 0.06$ & $52.62 \pm 0.44$ \\
    Payment $\pi_t$ per query, whole pass & $37.03 \pm 0.29$ & $118.10 \pm 1.49$ \\
    Generation cost per query, after initialization & $18.67 \pm 0.05$ & $42.01 \pm 0.36$ \\
    Query price per query, after initialization & $23.34 \pm 0.06$ & $52.51 \pm 0.45$ \\
    Payment $\pi_t$ per query, after initialization & $36.58 \pm 0.29$ & $116.43 \pm 1.51$ \\
    Winning bid, non-exploration steps & $15.75 \pm 0.07$ & $28.77 \pm 0.75$ \\
    Payment $\pi_t$ per query, non-exploration steps & $15.80 \pm 0.07$ & $28.93 \pm 0.75$ \\
    Post-initialization exploration payments per query, whole pass & $24.40 \pm 0.36$ & $95.92 \pm 1.98$ \\
    Share of exploration steps after initialization & $0.240 \pm 0.004$ & $0.323 \pm 0.007$ \\
    Excess payment per step $E_T^{\mathrm{pay}}/T$, whole pass & $16.676 \pm 0.238$ & $76.051 \pm 1.524$ \\
    Excess payment per query, non-exploration steps & $0.000 \pm 0.000$ & $0.002 \pm 0.004$ \\
    Mean provider payoff over the pass & $2688.6 \pm 43.6$ & $8353.3 \pm 170.9$ \\
    Lowest provider payoff over providers and runs & 903.8 & 2814.4 \\
    \bottomrule
  \end{tabular}
\end{table}

\begin{table}[!htbp]
  \centering\footnotesize
  \caption{\texttt{GSM8K} one pass: each provider's selections and payoff under the platform over $30$ independent runs, exploration steps included, with providers sorted by mean generation cost as in Table~\ref{tab:gsm8k-providers}. The last column counts the runs in which the provider's payoff over the pass was negative.}
  \label{tab:gsm8k-payoffs}
  \begin{tabular}{lrrrrr}
    \toprule
    Provider & Selections (mean) & Payoff (mean) & Payoff (min) & Payoff (max) & Negative runs \\
    \midrule
    \multicolumn{6}{l}{\textit{Full roster}} \\
    Llama-3.2-3B & 366.6 & 1506.2 & 1264.8 & 1675.0 & 0 \\
    Qwen2-0.5B & 36.3 & 1516.9 & 903.8 & 1872.4 & 0 \\
    Qwen2-1.5B & 210.2 & 1576.2 & 1257.8 & 1877.1 & 0 \\
    Llama-3.2-1B & 52.3 & 1584.3 & 1166.0 & 1943.0 & 0 \\
    Qwen2.5-3B ($i^*$) & 236.0 & 1618.9 & 1243.1 & 1848.1 & 0 \\
    Llama-3-8B & 128.6 & 2464.5 & 1764.6 & 3205.6 & 0 \\
    Llama-3.1-8B & 118.9 & 2613.9 & 1708.6 & 3915.3 & 0 \\
    Qwen2.5-7B & 85.0 & 5795.4 & 5587.7 & 5932.6 & 0 \\
    Qwen2-7B & 85.0 & 5520.5 & 5261.4 & 5733.9 & 0 \\
    \midrule
    \multicolumn{6}{l}{\textit{Ladder roster}} \\
    Qwen2-1.5B & 177.8 & 4814.9 & 4303.1 & 4955.2 & 0 \\
    Llama-3.2-1B & 37.8 & 4671.6 & 2814.4 & 4881.2 & 0 \\
    Llama-3-8B & 241.5 & 4908.6 & 4503.1 & 5163.9 & 0 \\
    Qwen2-1.5B@2 & 189.9 & 5156.8 & 4516.7 & 5882.3 & 0 \\
    Llama-3.2-1B@2 & 45.8 & 5225.4 & 4016.8 & 5863.3 & 0 \\
    Qwen2.5-7B ($i^*$) & 165.2 & 5567.6 & 4594.4 & 6821.7 & 0 \\
    Llama-3-8B@2 & 121.9 & 6622.6 & 5220.5 & 7875.4 & 0 \\
    Qwen2-1.5B@4 & 80.2 & 11988.5 & 6944.9 & 16058.1 & 0 \\
    Llama-3.2-1B@4 & 53.6 & 10912.0 & 3421.6 & 16011.6 & 0 \\
    Qwen2.5-7B@2 & 69.2 & 15030.6 & 10862.7 & 15919.4 & 0 \\
    Llama-3-8B@4 & 68.0 & 14278.0 & 13922.0 & 14696.6 & 0 \\
    Qwen2.5-7B@4 & 68.0 & 11062.7 & 10464.0 & 11551.1 & 0 \\
    \bottomrule
  \end{tabular}
\end{table}

\clearpage
\renewcommand{\resultsdir}{analysis/variants/count_scale2_alpha0.75/gpqa}
\section{GPQA weak-competition experiments}\label{app:gpqa}
This section reports the complete \texttt{GPQA} results for the full and ladder rosters, in the same form as the \texttt{GSM8K} section; Appendix~\ref{app:experimental-details} provides the experimental details, and the platform runs with the exploration rule described there. One-pass results are averaged over $30$ independent runs, paired across rules and reported with $95\%$ confidence intervals, while long-horizon trajectories report the median and $10$--$90\%$ range over $30$ independent runs. Monetary quantities are per query unless stated otherwise. A moving average has window size $K$, one pass over the benchmark, \ie, $546$ steps on the full roster and $545$ on the ladder roster, where the reward alignment drops one question. The amount paid follows the convention of Appendix~\ref{app:gsm8k}. On both rosters the identification horizon $T_{\mathrm{id}}$ lies beyond the long run of $T = 70{,}000$ queries, so the long-horizon results describe learning before that horizon, and no vertical line marks it in the figures. On both rosters $\bar c_{(2)}$ lies above the mean query price $\bar{v}_{i^*}$ of $i^*$, as on \texttt{GSM8K}, but here unqualified providers remain in $\hat{\Qcal}_t$ throughout the run, so the payment on non-exploration steps does not approach $\bar c_{(2)}$; the strong-competition roster of Appendix~\ref{app:gpqa-strong} is the \texttt{GPQA} setting in which it does.

\subsection{Providers and environment}\label{app:gpqa-env}

\begin{figure}[!htbp]
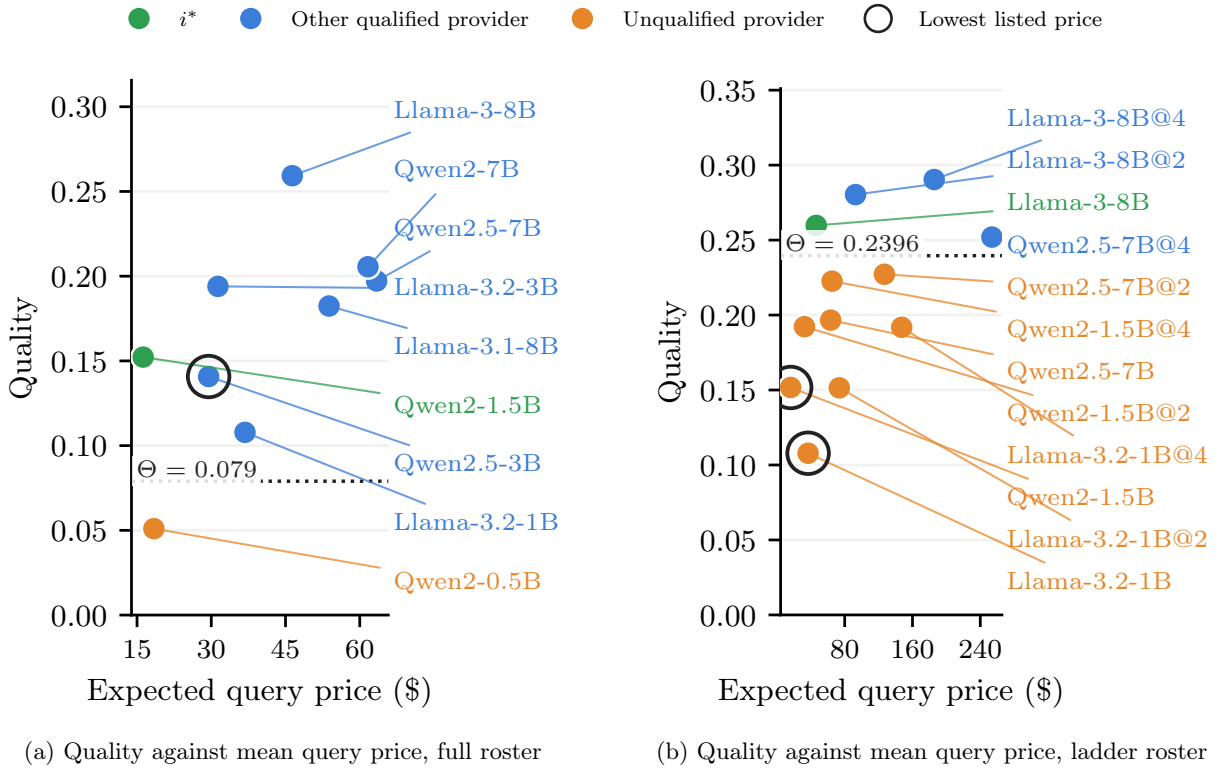

  \centering
  \sharedlegend{landscape}\\[4pt]
  \rosterpair{quality_vs_query_price}{Quality against mean query price}
  \caption{\texttt{GPQA} providers. Each point is one provider at its quality $q_i$ and mean query price; the dotted line is the quality threshold $q_{\min}$, and the ring marks the provider with the lowest listed price. Monetary quantities are in units of $10^{-6}$ USD. On the full roster the lowest listed price belongs to a qualified provider that is not $i^*$; on the ladder roster it is shared by two unqualified providers.}
  \label{fig:gpqa-landscape}
\end{figure}

\begin{table}[!htbp]
  \centering\small
  \caption{\texttt{GPQA} setting constants for the two rosters. Costs and prices are mean values per query.}
  \label{tab:gpqa-env}
  \begin{tabular}{lrr}
    \toprule
    Quantity & Full roster & Ladder roster \\
    \midrule
    Providers $N$ & 9 & 12 \\
    Benchmark size $K$ (queries in one pass) & 546 & 545 \\
    Quality threshold $q_{\min}$ & 0.079 & 0.2396 \\
    Qualified providers $|\Qcal|$ & 8 & 4 \\
    $i^*$ & Qwen2-1.5B & Llama-3-8B \\
    $\bar c_{i^*}$ & 13.0 & 37.1 \\
    $\bar{v}_{i^*}$ & 16.2 & 46.4 \\
    $\bar c_{(2)}$ & 23.6 & 74.2 \\
    $\bar{c}_{(2)} - \bar{c}_{i^*}$ & 10.6 & 37.1 \\
    Lowest listed price & Qwen2.5-3B & Llama-3.2-1B, Qwen2-1.5B \\
    \quad qualified & yes & no \\
    \quad equals $i^*$ & no & no \\
    $c_{\max}$ & 102.4 & 300.0 \\
    $\sum_{i \notin \Qcal} M_i$ & 120,055 & 1,211,234 \\
    $\sum_{i \in \Qcal \setminus \{i^*\}} L_i$ & 17,090 & 5,222 \\
    $\lceil g(T) \rceil - 1$ & 1,656 & 1,334 \\
    $B_{\mathrm{id}}(T)$ & 141,210 & 1,218,576 \\
    $T_{\mathrm{id}}$ & 315,298 & 2,540,932 \\
    \bottomrule
  \end{tabular}
\end{table}

\begin{table}[!htbp]
  \centering\footnotesize
  \setlength{\tabcolsep}{4pt}
  \caption{\texttt{GPQA} providers, sorted by mean generation cost. Listed price is in USD per $10^6$ output tokens; tokens are the mean number of generated tokens per query (summed over the $n_i$ generations of a Best-of-$n_i$ provider); query price and $\bar c_i$ are mean values per query. The gap is $\varepsilon_i = q_{\min} - q_i$ for an unqualified provider and $\Delta_i = \bar c_i - \bar c_{i^*}$ for a qualified one, with the corresponding selection bound $M_i$ from~\eqref{eq:unqualified-selection-threshold} or $L_i$ from~\eqref{eq:ordinary-selection-threshold}; where the exploration-step cap $\lceil g(T) \rceil - 1$ at $T = 70{,}000$ exceeds $L_i$, the applicable cap $\max\{L_i, \lceil g(T) \rceil - 1\}$ of Lemma~\ref{lem:qualified-suboptimal-selection} follows it.}
  \label{tab:gpqa-providers}
  \begin{tabular}{lrrrrrcll}
    \toprule
    Provider & Listed price & Tokens & Query price & $\bar c_i$ & $q_i$ & Qualified & Gap & Bound \\
    \midrule
    \multicolumn{9}{l}{\textit{Full roster}} \\
    Qwen2-1.5B ($i^*$) & 0.1 & 162 & 16.2 & 13.0 & 0.152 & yes & -- & -- \\
    Qwen2-0.5B & 0.1 & 184 & 18.4 & 14.7 & 0.051 & no & $\varepsilon_i = 0.028$ & $M = 120{,}055$ \\
    Qwen2.5-3B & 0.065 & 453 & 29.5 & 23.6 & 0.141 & yes & $\Delta_i = 10.6$ & $L = 6{,}861$ \\
    Llama-3.2-3B & 0.08 & 392 & 31.3 & 25.1 & 0.194 & yes & $\Delta_i = 12.1$ & $L = 5{,}106$ \\
    Llama-3.2-1B & 0.1 & 368 & 36.8 & 29.4 & 0.108 & yes & $\Delta_i = 16.5$ & $L = 2{,}564$ \\
    Llama-3-8B & 0.1455 & 319 & 46.4 & 37.1 & 0.259 & yes & $\Delta_i = 24.1$ & $L = 1{,}085$; cap 1{,}656 \\
    Llama-3.1-8B & 0.1245 & 432 & 53.8 & 43.0 & 0.182 & yes & $\Delta_i = 30.0$ & $L = 658$; cap 1{,}656 \\
    Qwen2-7B & 0.2 & 308 & 61.7 & 49.3 & 0.206 & yes & $\Delta_i = 36.3$ & $L = 426$; cap 1{,}656 \\
    Qwen2.5-7B & 0.1465 & 433 & 63.4 & 50.7 & 0.197 & yes & $\Delta_i = 37.7$ & $L = 390$; cap 1{,}656 \\
    \midrule
    \multicolumn{9}{l}{\textit{Ladder roster}} \\
    Qwen2-1.5B & 0.1 & 163 & 16.3 & 13.0 & 0.152 & no & $\varepsilon_i = 0.088$ & $M = 9{,}964$ \\
    Qwen2-1.5B@2 & 0.1 & 325 & 32.5 & 26.0 & 0.192 & no & $\varepsilon_i = 0.047$ & $M = 38{,}751$ \\
    Llama-3.2-1B & 0.1 & 368 & 36.8 & 29.5 & 0.108 & no & $\varepsilon_i = 0.132$ & $M = 4{,}063$ \\
    Llama-3-8B ($i^*$) & 0.1455 & 319 & 46.4 & 37.1 & 0.260 & yes & -- & -- \\
    Qwen2.5-7B & 0.1465 & 433 & 63.4 & 50.7 & 0.197 & no & $\varepsilon_i = 0.043$ & $M = 47{,}908$ \\
    Qwen2-1.5B@4 & 0.1 & 651 & 65.1 & 52.1 & 0.223 & no & $\varepsilon_i = 0.017$ & $M = 357{,}226$ \\
    Llama-3.2-1B@2 & 0.1 & 736 & 73.6 & 58.9 & 0.151 & no & $\varepsilon_i = 0.088$ & $M = 9{,}906$ \\
    Llama-3-8B@2 & 0.1455 & 638 & 92.8 & 74.2 & 0.280 & yes & $\Delta_i = 37.1$ & $L = 4{,}674$ \\
    Qwen2.5-7B@2 & 0.1465 & 866 & 126.9 & 101.5 & 0.227 & no & $\varepsilon_i = 0.012$ & $M = 705{,}391$ \\
    Llama-3.2-1B@4 & 0.1 & 1,473 & 147.3 & 117.8 & 0.192 & no & $\varepsilon_i = 0.048$ & $M = 38{,}025$ \\
    Llama-3-8B@4 & 0.1455 & 1,276 & 185.6 & 148.5 & 0.290 & yes & $\Delta_i = 111.4$ & $L = 392$; cap 1{,}334 \\
    Qwen2.5-7B@4 & 0.1465 & 1,732 & 253.7 & 203.0 & 0.252 & yes & $\Delta_i = 165.8$ & $L = 156$; cap 1{,}334 \\
    \bottomrule
  \end{tabular}
\end{table}
\clearpage
\subsection{One-pass behavior}\label{app:gpqa-one-pass}

\begin{table}[!htbp]
  \centering\scriptsize
  \setlength{\tabcolsep}{3pt}
  \caption{\texttt{GPQA} one-pass outcomes, mean $\pm$ $95\%$ half-width over $30$ independent runs, paired across rules. Generation cost and amount paid are per query; the amount paid is the payment $\pi_t$ for the platform and its no-quality-filter baseline, and the query price for the other rules. Within one pass the quality filter rules out no provider on either roster, so the two listed-price rules coincide.}
  \label{tab:gpqa-one-pass}
  \begin{tabular}{Zccccccc}
    \toprule
    Policy & Unqualified share & $R_T^{\mathrm{qual}}/T$ & $R_T^{\mathrm{gen}}/T$ & Accuracy & Share of $i^*$ & Generation cost & Amount paid \\
    \midrule
    \multicolumn{8}{l}{\textit{Full roster}} \\
    Platform & $0.185 \pm 0.006$ & $0.005 \pm 0.000$ & $14.22 \pm 0.04$ & $0.154 \pm 0.004$ & $0.214 \pm 0.006$ & $27.4 \pm 0.2$ & $64.4 \pm 0.2$ \\
    Uniform among eligible & $0.112 \pm 0.005$ & $0.003 \pm 0.000$ & $18.84 \pm 0.22$ & $0.161 \pm 0.005$ & $0.109 \pm 0.005$ & $31.9 \pm 0.3$ & $39.9 \pm 0.4$ \\
    Lowest listed price among eligible & $0.002 \pm 0.000$ & $0.000 \pm 0.000$ & $10.72 \pm 0.00$ & $0.143 \pm 0.004$ & $0.002 \pm 0.000$ & $23.7 \pm 0.0$ & $29.6 \pm 0.1$ \\
    Lowest listed price, no quality filter & $0.002 \pm 0.000$ & $0.000 \pm 0.000$ & $10.72 \pm 0.00$ & $0.143 \pm 0.004$ & $0.002 \pm 0.000$ & $23.7 \pm 0.0$ & $29.6 \pm 0.1$ \\
    No quality filter & $0.233 \pm 0.007$ & $0.007 \pm 0.000$ & $9.73 \pm 0.14$ & $0.139 \pm 0.004$ & $0.279 \pm 0.008$ & $22.6 \pm 0.2$ & $24.0 \pm 0.3$ \\
    \midrule
    \multicolumn{8}{l}{\textit{Ladder roster}} \\
    Platform & $0.715 \pm 0.001$ & $0.047 \pm 0.000$ & $32.96 \pm 0.04$ & $0.201 \pm 0.006$ & $0.092 \pm 0.001$ & $64.4 \pm 0.2$ & $196.5 \pm 0.8$ \\
    Uniform among eligible & $0.664 \pm 0.008$ & $0.039 \pm 0.001$ & $42.11 \pm 0.85$ & $0.211 \pm 0.005$ & $0.087 \pm 0.005$ & $75.9 \pm 1.0$ & $94.9 \pm 1.2$ \\
    Lowest listed price among eligible & $0.993 \pm 0.000$ & $0.108 \pm 0.000$ & $0.94 \pm 0.00$ & $0.130 \pm 0.003$ & $0.002 \pm 0.000$ & $22.3 \pm 0.2$ & $27.9 \pm 0.3$ \\
    Lowest listed price, no quality filter & $0.993 \pm 0.000$ & $0.108 \pm 0.000$ & $0.94 \pm 0.00$ & $0.130 \pm 0.003$ & $0.002 \pm 0.000$ & $22.3 \pm 0.2$ & $27.9 \pm 0.3$ \\
    No quality filter & $0.794 \pm 0.002$ & $0.056 \pm 0.000$ & $15.94 \pm 0.20$ & $0.193 \pm 0.005$ & $0.113 \pm 0.001$ & $45.4 \pm 0.3$ & $51.0 \pm 0.5$ \\
    \bottomrule
  \end{tabular}
\end{table}

\begin{figure}[!htbp]
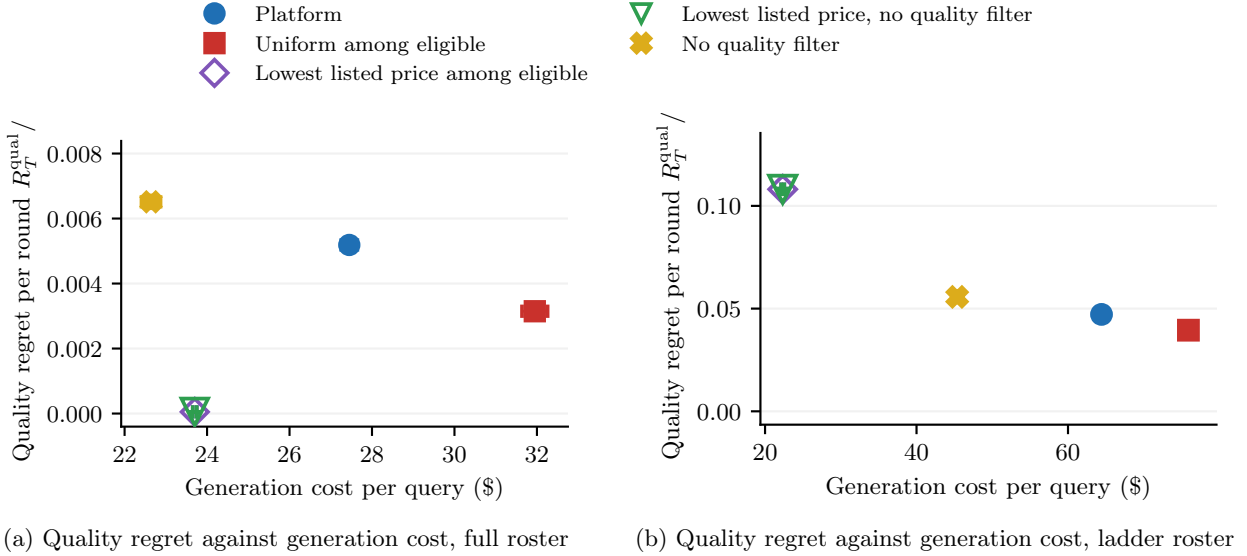

  \centering
  \sharedlegend{policies_markers}\\[4pt]
  \rosterpair{one_pass_quality_regret_vs_generation_cost}{Quality regret against generation cost}
  \caption{\texttt{GPQA} one pass: quality regret per step against generation cost per query for every routing rule, mean with $95\%$ confidence intervals across runs. Monetary quantities are in units of $10^{-6}$ USD. On both rosters the two listed-price rules coincide, since the quality filter rules out no provider within one pass.}
  \label{fig:gpqa-one-pass}
\end{figure}

\begin{figure}[!htbp]
  \centering
  \rosterpairb{one_pass_mechanism_selection_shares}{Platform selection shares}
  \caption{\texttt{GPQA} one pass: share of steps the platform assigned to each provider, exploration steps included, mean with $95\%$ confidence intervals across runs. Green marks $i^*$, blue the other qualified providers and orange the unqualified providers; the platform never observes these labels.}
  \label{fig:gpqa-selection-shares}
\end{figure}
\clearpage
\subsection{Long-horizon learning and identification}\label{app:gpqa-long}

\begin{figure}[!htbp]
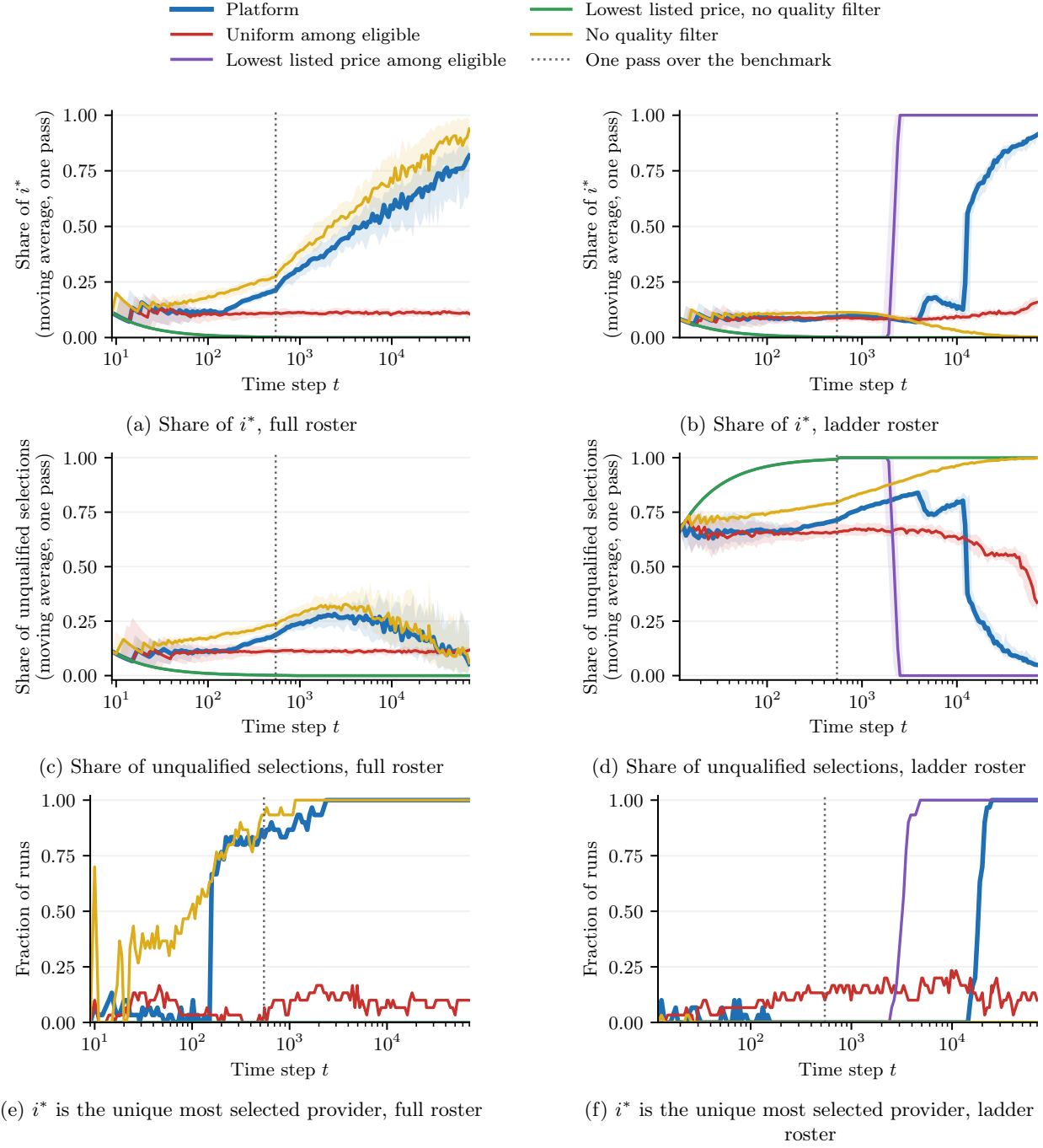

  \centering
  \sharedlegend{policies_lines}\\[4pt]
  \rosterpairs{long_horizon_rolling_i_star_share}{Share of $i^*$}\\[-2pt]
  \rosterpairs{long_horizon_unqualified_share}{Share of unqualified selections}\\[-2pt]
  \rosterpairs{long_horizon_i_star_identification_probability}{$i^*$ is the unique most selected provider}
  \caption{\texttt{GPQA} long horizon: identification of $i^*$ under every routing rule. Shares are moving averages with window size $K$, exploration steps included, median over $30$ independent runs with a $10$--$90\%$ band; the bottom row is the fraction of runs in which $i^*$ has strictly more selections than every other provider. The vertical line marks one pass over the benchmark; the identification horizon lies beyond $T$. On the full roster the two listed-price rules coincide.}
  \label{fig:gpqa-identification}
\end{figure}

\begin{figure}[!htbp]
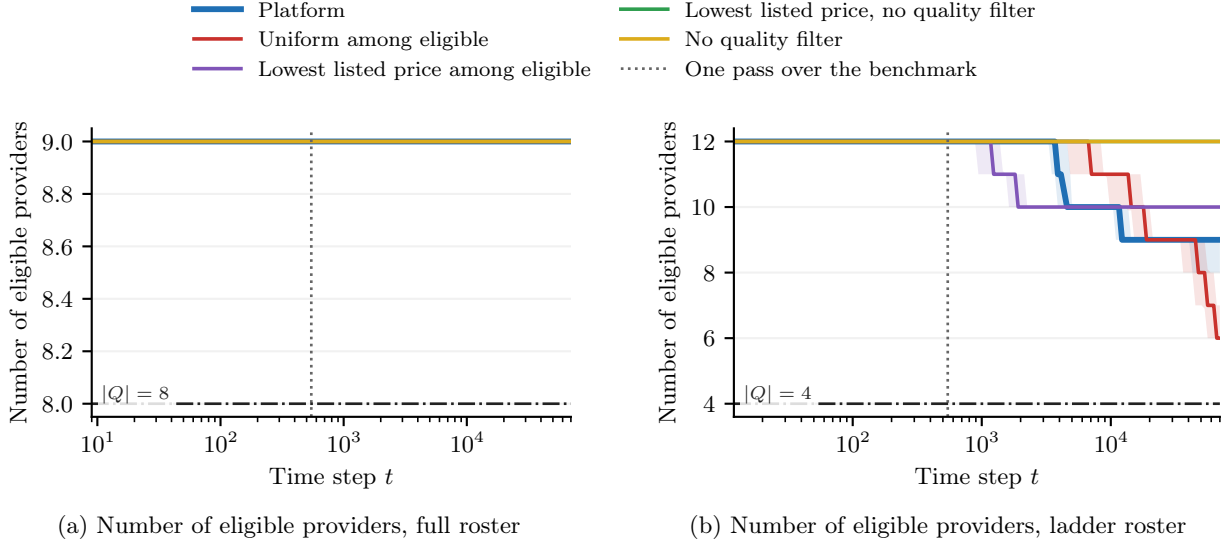

  \centering
  \sharedlegend{policies_lines}\\[4pt]
  \rosterpair{long_horizon_eligible_providers}{Number of eligible providers}
  \caption{\texttt{GPQA} long horizon: the number of providers not yet ruled out by the quality filter under every rule, median over $30$ independent runs with a $10$--$90\%$ band. On the full roster the single unqualified provider is never ruled out, so every curve stays at $N$; on the ladder roster the two rules without a filter stay at $N$, and the dash-dotted line marks $|\Qcal|$.}
  \label{fig:gpqa-eligible}
\end{figure}

\begin{figure}[!htbp]
  \centering
  \sharedlegend{policies_lines}\\[4pt]
  \rosterpair{long_horizon_generation_cost_per_query}{Generation cost per query}\\[2pt]
  \rosterpair{long_horizon_amount_paid_per_query}{Amount paid per query}
  \caption{\texttt{GPQA} long horizon: generation cost per query as a moving average with window size $K$ (top) and amount paid per query as a moving average with window size $K$ over the steps after initialization (bottom), under every rule, median over $30$ independent runs with a $10$--$90\%$ band. Monetary quantities are in units of $10^{-6}$ USD. The reference lines mark $\bar c_{i^*}$, the mean query price $\bar{v}_{i^*}$ of $i^*$, $\bar c_{(2)}$ and, on the full roster, the cost of the qualified provider with the lowest listed price. The platform's payment $\pi_t$ includes its exploration steps at $c_{\max}$; the vertical axis of the bottom panels is cut, so the platform's curve enters from above.}
  \label{fig:gpqa-routing}
\end{figure}

\begin{table}[!htbp]
  \centering\scriptsize
  \setlength{\tabcolsep}{3pt}
  \caption{\texttt{GPQA} long horizon: outcomes at $T = 70{,}000$, means over $30$ independent runs. Shares of $i^*$ and of unqualified selections are cumulative unless marked as moving averages and include the exploration steps; ``$i^*$ most selected'' is the fraction of runs in which $i^*$ is the unique most selected provider; $|\hat{\Qcal}_t|$ is the number of eligible providers at $T$. Generation cost and amount paid are cumulative per query; for the platform the amount paid includes the exploration steps.}
  \label{tab:gpqa-long}
  \begin{tabular}{Yrrrrrrrrr}
    \toprule
    Policy & \shortstack{Share\\of $i^*$} & \shortstack{Moving\\avg.} & \shortstack{Unqualified\\share} & $R_t^{\mathrm{qual}}/t$ & $R_t^{\mathrm{gen}}/t$ & \shortstack{Generation\\cost} & \shortstack{Amount\\paid} & \shortstack{$i^*$ most\\selected} & $|\hat{\Qcal}_t|$ \\
    \midrule
    \multicolumn{10}{l}{\textit{Full roster}} \\
    Platform & 0.676 & 0.779 & 0.158 & 0.0044 & 4.234 & 17.2 & 27.4 & 1.00 & 9.0 \\
    Uniform among eligible & 0.111 & 0.111 & 0.111 & 0.0031 & 18.785 & 31.8 & 39.7 & 0.10 & 9.0 \\
    Lowest listed price among eligible & 0.000 & 0.000 & 0.000 & 0.0000 & 10.582 & 23.6 & 29.5 & 0.00 & 9.0 \\
    Lowest listed price, no quality filter & 0.000 & 0.000 & 0.000 & 0.0000 & 10.582 & 23.6 & 29.5 & 0.00 & 9.0 \\
    No quality filter & 0.801 & 0.913 & 0.172 & 0.0048 & 0.709 & 13.7 & 13.7 & 1.00 & 9.0 \\
    \midrule
    \multicolumn{10}{l}{\textit{Ladder roster}} \\
    Platform & 0.701 & 0.909 & 0.242 & 0.0117 & 10.288 & 46.0 & 63.2 & 1.00 & 8.5 \\
    Uniform among eligible & 0.117 & 0.161 & 0.533 & 0.0196 & 54.062 & 89.7 & 112.2 & 0.10 & 6.3 \\
    Lowest listed price among eligible & 0.972 & 1.000 & 0.027 & 0.0028 & 0.007 & 36.6 & 45.7 & 1.00 & 10.0 \\
    Lowest listed price, no quality filter & 0.000 & 0.000 & 1.000 & 0.1098 & 0.007 & 21.2 & 26.6 & 0.00 & 12.0 \\
    No quality filter & 0.017 & 0.003 & 0.980 & 0.0847 & 0.635 & 15.8 & 15.9 & 0.00 & 12.0 \\
    \bottomrule
  \end{tabular}
\end{table}

\begin{table}[!htbp]
    \centering\small
    \caption{\texttt{GPQA} exploration steps of the platform, means over $30$ independent runs, over one pass and over the long run of $T = 70{,}000$ queries. Shares are of the steps after initialization; payments and costs are per step of the corresponding horizon, the long-run values as moving averages with window size $K$ at $T$. The last two rows give the platform's moving-average payment $\pi_t$ on non-exploration steps and over all steps at $T$, per query.}
    \label{tab:gpqa-exploration}
    \begin{tabular}{lrrrr}
    \toprule
    Quantity & \multicolumn{2}{c}{Full roster} & \multicolumn{2}{c}{Ladder roster} \\
     & one pass & long run & one pass & long run \\
    \midrule
    Exploration steps (mean) & 295 & 10,967 & 326 & 6,916 \\
    Share of exploration steps after initialization & 0.550 & 0.157 & 0.612 & 0.099 \\
    Share of exploration steps at $T$ (moving avg.) & -- & 0.128 & -- & 0.066 \\
    Exploration-step payment per step & 55.39 & 13.13 & 179.69 & 19.78 \\
    Exploration-step generation cost per step & 19.62 & 4.68 & 52.58 & 8.38 \\
    Share of exploration steps given to $i^*$ & 0.055 & 0.001 & 0.059 & 0.003 \\
    Share of exploration steps given to unqualified providers & 0.057 & 0.002 & 0.629 & 0.419 \\
    Payment $\pi_t$ at $T$, non-exploration steps (moving avg.) & -- & 13.19 & -- & 37.49 \\
    Payment $\pi_t$ at $T$, all steps (moving avg.) & -- & 24.63 & -- & 54.80 \\
    \bottomrule
    \end{tabular}
\end{table}
\clearpage
\subsection{Regret and theoretical bounds}\label{app:gpqa-bounds}

\begin{figure}[!htbp]
  \centering
  \sharedlegend{policies_lines}\\[4pt]
  \rosterpair{long_horizon_quality_regret_per_round}{Average quality regret}\\[2pt]
  \rosterpair{long_horizon_generation_regret_per_round}{Average generation regret}
  \caption{\texttt{GPQA} long horizon: average quality regret $R_t^{\mathrm{qual}}/t$ and average generation regret $R_t^{\mathrm{gen}}/t$ of every routing rule, median over $30$ independent runs with a $10$--$90\%$ band, on logarithmic axes. Monetary quantities are in units of $10^{-6}$ USD. The platform's regrets include its exploration steps.}
  \label{fig:gpqa-regret}
\end{figure}

\begin{figure}[!htbp]
  \centering
  \sharedlegend{bounds}\\[4pt]
  \rosterpair{long_horizon_quality_regret_vs_bound}{Quality regret}\\[2pt]
  \rosterpair{long_horizon_generation_regret_vs_bound}{Generation regret}\\[2pt]
  \rosterpair{long_horizon_excess_payment_vs_bound}{Excess payment}
  \caption{\texttt{GPQA} long horizon: the platform's cumulative quality regret, generation regret and excess payment against their bounds, evaluated at every time step. Monetary quantities are in units of $10^{-6}$ USD. The bounds are those for the platform with exploration steps, Theorems~\ref{thm:quality-guarantee} and~\ref{thm:generation-regret}, each with an anytime branch and a gap-dependent branch; the excess-payment bound \eqref{eq:excess-payment-general} without its exploration term has a single branch. The regrets include the exploration steps; the excess payment excludes them.}
  \label{fig:gpqa-bounds}
\end{figure}

\begin{table}[!htbp]
  \centering\small
  \caption{\texttt{GPQA} theoretical diagnostics of the platform over the long run ($T = 70{,}000$, $30$ independent runs). Realized quantities are means over runs; each regret bound is the smaller of its two branches at $T$, for the platform with exploration steps. Selections and regrets include the exploration steps and are compared with $B_{\mathrm{id}}(T)$ and the bounds that account for them; the excess payment is given without the exploration steps, against \eqref{eq:excess-payment-general} without its exploration term, and with them, against \eqref{eq:excess-payment-general}. The identification horizon $T_{\mathrm{id}}$ lies beyond $T$ on both rosters, so the corresponding row is empty. The good event is checked at every selection count of every provider; the payment bounds are those of Appendix~\ref{app:experimental-details}, checked on non-exploration steps.}
  \label{tab:gpqa-diagnostics}
  \begin{tabular}{lrr}
    \toprule
    Quantity & Full roster & Ladder roster \\
    \midrule
    $B_{\mathrm{id}}(T)$ & 141,210 & 1,218,576 \\
    $T_{\mathrm{id}}$ & 315,298 & 2,540,932 \\
    Horizon $T$ of the long run & 70,000 & 70,000 \\
    Selections of providers other than $i^*$ after initialization (mean) & 22,641 & 20,907 \\
    \quad ratio to $B_{\mathrm{id}}(T)$ & 0.160 & 0.017 \\
    Runs in which $i^*$ is the most selected provider at $t = 5{,}000$ & 1.00 & 0.00 \\
    \quad at $t = 10{,}000$ & 1.00 & 0.00 \\
    \quad at $t = T_{\mathrm{id}}$ & -- & -- \\
    \quad at $t = T$ & 1.00 & 1.00 \\
    $R_T^{\mathrm{qual}}$, realized (mean) & 309 & 818 \\
    \quad bound \eqref{eq:quality-regret-bound} & 3,362 & 14,133 \\
    \quad ratio & 0.092 & 0.058 \\
    $R_T^{\mathrm{gen}}$, realized (mean) & 296,412 & 720,159 \\
    \quad bound \eqref{eq:generation-regret-bound} & 596,889 & 8,642,653 \\
    \quad ratio & 0.497 & 0.083 \\
    $E_T^{\mathrm{pay}}$ without the exploration steps, realized (mean) & 1,149 & 2,710 \\
    \quad bound \eqref{eq:excess-payment-general} without its exploration term & 623,803 & 2,121,051 \\
    \quad ratio & 0.002 & 0.001 \\
    $E_T^{\mathrm{pay}}$ incl. exploration steps, realized (mean) & 865,643 & 1,564,346 \\
    \quad bound \eqref{eq:excess-payment-general} & 1,798,620 & 5,735,545 \\
    \quad ratio & 0.481 & 0.273 \\
    Good event held, one pass (runs) & 30 of 30 & 30 of 30 \\
    Good event held, long run, platform & 30 of 30 & 30 of 30 \\
    Good event held, long run, no quality filter & 30 of 30 & 30 of 30 \\
    Payment-bound violations, one pass (steps) & 0 & 0 \\
    \bottomrule
  \end{tabular}
\end{table}

\begin{table}[!htbp]
  \centering\footnotesize
  \setlength{\tabcolsep}{4pt}
  \caption{\texttt{GPQA} long horizon: selections of every provider other than $i^*$ by the platform after initialization, over $30$ independent runs of $T = 70{,}000$ queries, against its cap, $M_i$ for an unqualified provider and $\max\{L_i, \lceil g(T) \rceil - 1\}$ for a qualified one (Lemma~\ref{lem:qualified-suboptimal-selection}), shown as $\lceil g(T) \rceil - 1$ where the exploration-step cap exceeds $L_i$; the last row of each roster compares the total with $B_{\mathrm{id}}(T)$. Providers are sorted by mean generation cost, as in Table~\ref{tab:gpqa-providers}. Selections are the mean and the maximum over runs of the number of post-initialization selections, exploration steps included; the exploration column is the mean number of exploration steps among them; the last selection is the mean over runs of the step of the provider's last selection; the ratio is the mean number of selections divided by the cap.}
  \label{tab:gpqa-selections}
  \begin{tabular}{llrrrrlr}
    \toprule
    Provider & Gap & \multicolumn{2}{c}{Selections} & Exploration & Last selection & Bound & Ratio \\
    & & mean & max & (mean) & (mean step) & & \\
    \midrule
    \multicolumn{8}{l}{\textit{Full roster}} \\
    Qwen2-0.5B & $\varepsilon_i = 0.028$ & 11,049 & 12,859 & 17 & 69,838 & $M = 120{,}055$ & 0.092 \\
    Qwen2.5-3B & $\Delta_i = 10.6$ & 1,656 & 1,656 & 1,255 & 69,981 & $L = 6{,}861$ & 0.241 \\
    Llama-3.2-3B & $\Delta_i = 12.1$ & 1,656 & 1,656 & 1,406 & 69,980 & $L = 5{,}106$ & 0.324 \\
    Llama-3.2-1B & $\Delta_i = 16.5$ & 1,656 & 1,656 & 1,649 & 69,980 & $L = 2{,}564$ & 0.646 \\
    Llama-3-8B & $\Delta_i = 24.1$ & 1,656 & 1,656 & 1,656 & 69,979 & $\lceil g(T) \rceil - 1 = 1{,}656$ & 1.000 \\
    Llama-3.1-8B & $\Delta_i = 30.0$ & 1,656 & 1,656 & 1,656 & 69,980 & $\lceil g(T) \rceil - 1 = 1{,}656$ & 1.000 \\
    Qwen2-7B & $\Delta_i = 36.3$ & 1,656 & 1,656 & 1,656 & 69,980 & $\lceil g(T) \rceil - 1 = 1{,}656$ & 1.000 \\
    Qwen2.5-7B & $\Delta_i = 37.7$ & 1,656 & 1,656 & 1,656 & 69,980 & $\lceil g(T) \rceil - 1 = 1{,}656$ & 1.000 \\
    All providers other than $i^*$ & & 22,641 & & 10,951 &  & $B_{\mathrm{id}}(T) = 141{,}210$ & 0.160 \\
    \midrule
    \multicolumn{8}{l}{\textit{Ladder roster}} \\
    Qwen2-1.5B & $\varepsilon_i = 0.088$ & 1,379 & 1,979 & 16 & 4,064 & $M = 9{,}964$ & 0.138 \\
    Qwen2-1.5B@2 & $\varepsilon_i = 0.047$ & 5,486 & 7,649 & 17 & 12,345 & $M = 38{,}751$ & 0.142 \\
    Llama-3.2-1B & $\varepsilon_i = 0.132$ & 545 & 716 & 17 & 4,314 & $M = 4{,}063$ & 0.134 \\
    Qwen2.5-7B & $\varepsilon_i = 0.043$ & 2,983 & 3,157 & 27 & 69,918 & $M = 47{,}908$ & 0.062 \\
    Qwen2-1.5B@4 & $\varepsilon_i = 0.017$ & 2,563 & 3,051 & 37 & 69,716 & $M = 357{,}226$ & 0.007 \\
    Llama-3.2-1B@2 & $\varepsilon_i = 0.088$ & 1,281 & 1,456 & 113 & 59,060 & $M = 9{,}906$ & 0.129 \\
    Llama-3-8B@2 & $\Delta_i = 37.1$ & 1,334 & 1,334 & 1,334 & 69,936 & $L = 4{,}674$ & 0.285 \\
    Qwen2.5-7B@2 & $\varepsilon_i = 0.012$ & 1,334 & 1,334 & 1,334 & 69,936 & $M = 705{,}391$ & 0.002 \\
    Llama-3.2-1B@4 & $\varepsilon_i = 0.048$ & 1,334 & 1,334 & 1,334 & 69,936 & $M = 38{,}025$ & 0.035 \\
    Llama-3-8B@4 & $\Delta_i = 111.4$ & 1,334 & 1,334 & 1,334 & 69,936 & $\lceil g(T) \rceil - 1 = 1{,}334$ & 1.000 \\
    Qwen2.5-7B@4 & $\Delta_i = 165.8$ & 1,334 & 1,334 & 1,334 & 69,936 & $\lceil g(T) \rceil - 1 = 1{,}334$ & 1.000 \\
    All providers other than $i^*$ & & 20,907 & & 6,897 &  & $B_{\mathrm{id}}(T) = 1{,}218{,}576$ & 0.017 \\
    \bottomrule
  \end{tabular}
\end{table}
\clearpage
\subsection{Payments and provider outcomes}\label{app:gpqa-payments}

\begin{figure}[!htbp]
  \centering
  \sharedlegend{economic}\\[4pt]
  \rosterpair{long_horizon_mechanism_economic_quantities}{Cost, bid, payment and query price}
  \caption{\texttt{GPQA} long horizon: the platform's realized generation cost, the query price of the same generations and the payment $\pi_t$, as a moving average with window size $K$ over the steps after initialization, and its winning bid and payment on non-exploration steps in the same window, median over $30$ independent runs with a $10$--$90\%$ band. Monetary quantities are in units of $10^{-6}$ USD. The dotted, dash-dotted and dashed lines mark $\bar c_{i^*}$, $\bar c_{(2)}$ and the mean query price $\bar{v}_{i^*}$ of $i^*$. The vertical axis is cut; $\pi_t$ enters from above, since early in the run most steps are exploration steps and paid $c_{\max}$.}
  \label{fig:gpqa-economic}
\end{figure}

\begin{figure}[!htbp]
  \centering
  \sharedlegend{payment_bounds}\\[4pt]
  \rosterpair{payment_bounds_long_horizon}{Payment on non-exploration steps and its bounds}
  \caption{\texttt{GPQA} long horizon: the payment $\pi_t$ of the first run on non-exploration steps, sampled at logarithmically spaced time steps, inside the payment bounds of Appendix~\ref{app:experimental-details} for the provider selected in that step; sampled steps that were exploration steps are left blank. Monetary quantities are in units of $10^{-6}$ USD. The bounds depend on the selected provider's cost and selection count, and close in as the counts grow; the vertical axis is cut at $1.6\,\bar c_{(2)}$, above which the upper bound sits at $c_{\max}$ in the early time steps.}
  \label{fig:gpqa-payment-bounds}
\end{figure}

\begin{table}[!htbp]
  \centering\small
  \caption{\texttt{GPQA} one pass: the platform's economic quantities, mean $\pm$ $95\%$ half-width over $30$ independent runs. The payment $\pi_t$ is over all steps; the winning bid, the payment on non-exploration steps and the second excess-payment row are over non-exploration steps only. A provider's payoff is its total payment minus its total generation cost over the pass.}
  \label{tab:gpqa-payments}
  \begin{tabular}{lrr}
    \toprule
    Quantity & Full roster & Ladder roster \\
    \midrule
    Generation cost per query, whole pass & $27.45 \pm 0.20$ & $64.42 \pm 0.21$ \\
    Query price per query, whole pass & $34.31 \pm 0.24$ & $80.53 \pm 0.26$ \\
    Payment $\pi_t$ per query, whole pass & $64.43 \pm 0.25$ & $196.53 \pm 0.84$ \\
    Generation cost per query, after initialization & $27.37 \pm 0.20$ & $64.14 \pm 0.21$ \\
    Query price per query, after initialization & $34.21 \pm 0.25$ & $80.18 \pm 0.26$ \\
    Payment $\pi_t$ per query, after initialization & $63.80 \pm 0.25$ & $194.20 \pm 0.86$ \\
    Winning bid, non-exploration steps & $16.35 \pm 0.45$ & $26.43 \pm 0.59$ \\
    Payment $\pi_t$ per query, non-exploration steps & $16.59 \pm 0.45$ & $26.96 \pm 0.59$ \\
    Post-initialization exploration payments per query, whole pass & $55.39 \pm 0.47$ & $179.69 \pm 1.08$ \\
    Share of exploration steps after initialization & $0.550 \pm 0.005$ & $0.612 \pm 0.004$ \\
    Excess payment per step $E_T^{\mathrm{pay}}/T$, whole pass & $44.018 \pm 0.352$ & $140.199 \pm 0.810$ \\
    Excess payment per query, non-exploration steps & $0.182 \pm 0.042$ & $0.000 \pm 0.000$ \\
    Mean provider payoff over the pass & $2243.8 \pm 19.5$ & $5999.7 \pm 40.6$ \\
    Lowest provider payoff over providers and runs & 1343.3 & 2870.6 \\
    \bottomrule
  \end{tabular}
\end{table}

\begin{table}[!htbp]
  \centering\footnotesize
  \caption{\texttt{GPQA} one pass: each provider's selections and payoff under the platform over $30$ independent runs, exploration steps included, with providers sorted by mean generation cost as in Table~\ref{tab:gpqa-providers}. The last column counts the runs in which the provider's payoff over the pass was negative.}
  \label{tab:gpqa-payoffs}
  \begin{tabular}{lrrrrr}
    \toprule
    Provider & Selections (mean) & Payoff (mean) & Payoff (min) & Payoff (max) & Negative runs \\
    \midrule
    \multicolumn{6}{l}{\textit{Full roster}} \\
    Qwen2-1.5B ($i^*$) & 116.7 & 1555.1 & 1343.3 & 1986.5 & 0 \\
    Qwen2-0.5B & 101.1 & 1594.5 & 1435.4 & 1888.2 & 0 \\
    Qwen2.5-3B & 56.5 & 1883.0 & 1560.1 & 2205.5 & 0 \\
    Llama-3.2-3B & 51.4 & 2100.5 & 1593.1 & 2870.4 & 0 \\
    Llama-3.2-1B & 44.2 & 3063.6 & 1996.9 & 3278.2 & 0 \\
    Llama-3-8B & 44.0 & 2861.2 & 2719.4 & 2994.6 & 0 \\
    Llama-3.1-8B & 44.0 & 2609.1 & 2493.5 & 2728.7 & 0 \\
    Qwen2-7B & 44.0 & 2247.9 & 1844.5 & 2560.2 & 0 \\
    Qwen2.5-7B & 44.0 & 2278.8 & 2122.5 & 2431.9 & 0 \\
    \midrule
    \multicolumn{6}{l}{\textit{Ladder roster}} \\
    Qwen2-1.5B & 84.5 & 4935.4 & 4805.1 & 5068.8 & 0 \\
    Qwen2-1.5B@2 & 62.8 & 4886.1 & 4441.9 & 5491.0 & 0 \\
    Llama-3.2-1B & 58.8 & 4942.7 & 4600.5 & 5403.9 & 0 \\
    Llama-3-8B ($i^*$) & 50.3 & 5308.7 & 4857.2 & 5847.9 & 0 \\
    Qwen2.5-7B & 38.7 & 6965.4 & 5994.6 & 7906.3 & 0 \\
    Qwen2-1.5B@4 & 39.4 & 6897.7 & 4887.2 & 8577.4 & 0 \\
    Llama-3.2-1B@2 & 35.3 & 8251.0 & 7355.8 & 8529.1 & 0 \\
    Llama-3-8B@2 & 35.0 & 7929.8 & 7581.1 & 8159.3 & 0 \\
    Qwen2.5-7B@2 & 35.0 & 6918.9 & 6730.4 & 7165.6 & 0 \\
    Llama-3.2-1B@4 & 35.0 & 6327.6 & 6072.9 & 6828.6 & 0 \\
    Llama-3-8B@4 & 35.0 & 5278.7 & 4828.8 & 5951.7 & 0 \\
    Qwen2.5-7B@4 & 35.0 & 3354.2 & 2870.6 & 3942.7 & 0 \\
    \bottomrule
  \end{tabular}
\end{table}

\clearpage
\renewcommand{\resultsdir}{analysis/variants/count_scale2_alpha0.75/gpqa}
\section{GPQA strong-competition experiment}\label{app:gpqa-strong}
This section reports the results for the strong-competition roster on \texttt{GPQA}, in the same form as the two preceding sections; Appendix~\ref{app:experimental-details} provides the experimental details, and the platform runs with the exploration rule described there. The roster is the newest large model and the smallest model of each family, and the threshold is placed so that both large models qualify and both small ones do not. The setting is chosen for its payment ordering: the second-lowest qualified cost $\bar c_{(2)} = 50.7$ lies below the mean query price $\bar{v}_{i^*}$ of $i^*$, $53.8$, and above $\bar c_{i^*} = 43.0$, so a payment that approaches $\bar c_{(2)}$ pays the providers more than the generation cost of $i^*$ and the platform less than the listed price of $i^*$. Conventions are those of Appendix~\ref{app:gpqa}; a moving average has window size $K = 546$. The exploration-step cap at $T = 70{,}000$, $\lceil g(T) \rceil - 1 = 3{,}043$, lies below $L_i = 6{,}659$ of the single qualified competitor, so the cap is $L_i$ itself, and $B_{\mathrm{id}}(T) = 76{,}846$; the identification horizon $T_{\mathrm{id}} = 153{,}697$ lies beyond the long run, so no vertical line marks it.

\subsection{Providers and environment}\label{app:strong-env}

\begin{figure}[!htbp]
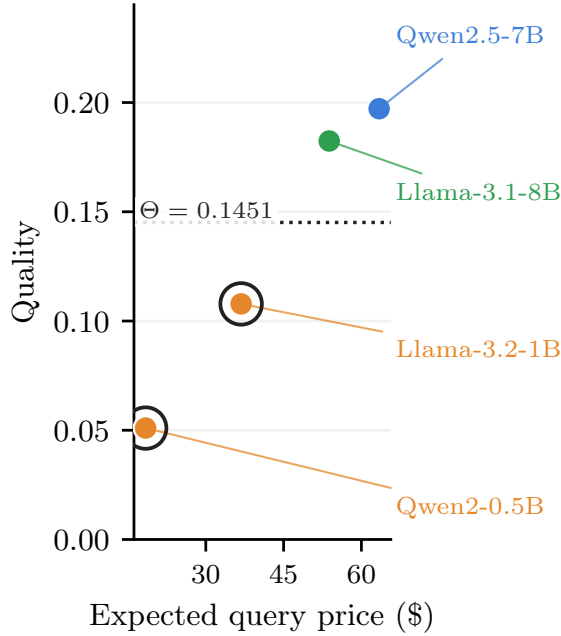

  \centering
  \sharedlegendof{strong}{landscape}\\[4pt]
  \strongpanel{quality_vs_query_price}{Quality against mean query price}
  \caption{\texttt{GPQA} strong-competition roster. Each point is one provider at its quality $q_i$ and mean query price; the dotted line is the quality threshold $q_{\min}$, and the ring marks the providers with the lowest listed price, both unqualified. Monetary quantities are in units of $10^{-6}$ USD.}
  \label{fig:strong-landscape}
\end{figure}

\begin{table}[!htbp]
  \centering\small
  \caption{\texttt{GPQA} strong-competition setting constants. Costs and prices are mean values per query.}
  \label{tab:strong-env}
  \begin{tabular}{lr}
    \toprule
    Quantity & Strong-competition roster \\
    \midrule
    Providers $N$ & 4 \\
    Benchmark size $K$ (queries in one pass) & 546 \\
    Quality threshold $q_{\min}$ & 0.1451107 \\
    Qualified providers $|\Qcal|$ & 2 \\
    $i^*$ & Llama-3.1-8B \\
    $\bar c_{i^*}$ & 43.0 \\
    $\bar{v}_{i^*}$ & 53.8 \\
    $\bar c_{(2)}$ & 50.7 \\
    $\bar{c}_{(2)} - \bar{c}_{i^*}$ & 7.7 \\
    Lowest listed price & Llama-3.2-1B, Qwen2-0.5B \\
    \quad qualified & no \\
    \quad equals $i^*$ & no \\
    $c_{\max}$ & 75.0 \\
    $\sum_{i \notin \Qcal} M_i$ & 70,187 \\
    $\sum_{i \in \Qcal \setminus \{i^*\}} L_i$ & 6,659 \\
    $\lceil g(T) \rceil - 1$ & 3,043 \\
    $B_{\mathrm{id}}(T)$ & 76,846 \\
    $T_{\mathrm{id}}$ & 153,697 \\
    \bottomrule
  \end{tabular}
\end{table}

\begin{table}[!htbp]
  \centering\footnotesize
  \setlength{\tabcolsep}{4pt}
  \caption{\texttt{GPQA} strong-competition providers, sorted by mean generation cost; columns as in Table~\ref{tab:gpqa-providers}.}
  \label{tab:strong-providers}
  \begin{tabular}{lrrrrrcll}
    \toprule
    Provider & Listed price & Tokens & Query price & $\bar c_i$ & $q_i$ & Qualified & Gap & Bound \\
    \midrule
    \multicolumn{9}{l}{\textit{Strong-competition roster}} \\
    Qwen2-0.5B & 0.1 & 184 & 18.4 & 14.7 & 0.051 & no & $\varepsilon_i = 0.094$ & $M = 8{,}074$ \\
    Llama-3.2-1B & 0.1 & 368 & 36.8 & 29.4 & 0.108 & no & $\varepsilon_i = 0.037$ & $M = 62{,}113$ \\
    Llama-3.1-8B ($i^*$) & 0.1245 & 432 & 53.8 & 43.0 & 0.182 & yes & -- & -- \\
    Qwen2.5-7B & 0.1465 & 433 & 63.4 & 50.7 & 0.197 & yes & $\Delta_i = 7.7$ & $L = 6{,}659$ \\
    \bottomrule
  \end{tabular}
\end{table}
\clearpage
\subsection{One-pass behavior}\label{app:strong-one-pass}

\begin{table}[!htbp]
  \centering\scriptsize
  \setlength{\tabcolsep}{3pt}
  \caption{\texttt{GPQA} strong-competition one-pass outcomes, mean $\pm$ $95\%$ half-width over $30$ independent runs, paired across rules; columns and the amount-paid convention as in Table~\ref{tab:gpqa-one-pass}. Within one pass the quality filter rules out no provider, so the two listed-price rules coincide.}
  \label{tab:strong-one-pass}
  \begin{tabular}{Zccccccc}
    \toprule
    Policy & Unqualified share & $R_T^{\mathrm{qual}}/T$ & $R_T^{\mathrm{gen}}/T$ & Accuracy & Share of $i^*$ & Generation cost & Amount paid \\
    \midrule
    \multicolumn{8}{l}{\textit{Strong-competition roster}} \\
    Platform & $0.707 \pm 0.000$ & $0.058 \pm 0.000$ & $1.13 \pm 0.00$ & $0.100 \pm 0.004$ & $0.147 \pm 0.000$ & $26.4 \pm 0.1$ & $46.0 \pm 0.3$ \\
    Uniform among eligible & $0.503 \pm 0.008$ & $0.033 \pm 0.001$ & $1.93 \pm 0.04$ & $0.136 \pm 0.005$ & $0.246 \pm 0.007$ & $34.4 \pm 0.2$ & $42.9 \pm 0.3$ \\
    Lowest listed price among eligible & $0.996 \pm 0.000$ & $0.066 \pm 0.000$ & $0.01 \pm 0.00$ & $0.085 \pm 0.004$ & $0.002 \pm 0.000$ & $22.1 \pm 0.2$ & $27.6 \pm 0.2$ \\
    Lowest listed price, no quality filter & $0.996 \pm 0.000$ & $0.066 \pm 0.000$ & $0.01 \pm 0.00$ & $0.085 \pm 0.004$ & $0.002 \pm 0.000$ & $22.1 \pm 0.2$ & $27.6 \pm 0.2$ \\
    No quality filter & $0.930 \pm 0.002$ & $0.081 \pm 0.000$ & $0.22 \pm 0.01$ & $0.071 \pm 0.004$ & $0.042 \pm 0.002$ & $18.6 \pm 0.2$ & $19.2 \pm 0.2$ \\
    \bottomrule
  \end{tabular}
\end{table}

\begin{figure}[!htbp]
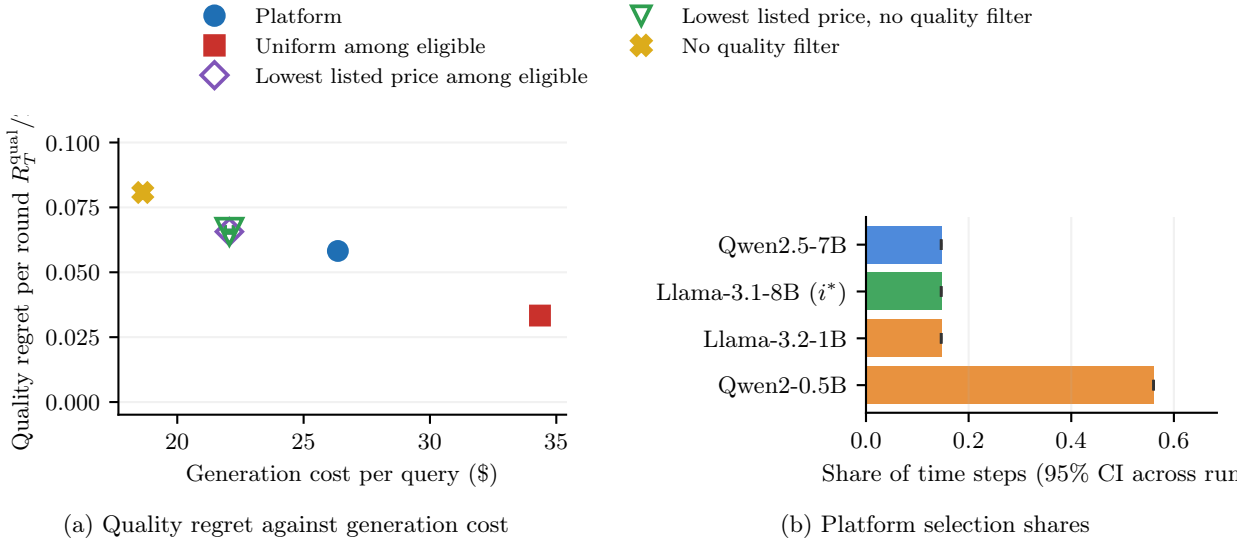

  \centering
  \sharedlegendof{strong}{policies_markers}\\[4pt]
  \strongpair{one_pass_quality_regret_vs_generation_cost}{Quality regret against generation cost}{one_pass_mechanism_selection_shares}{Platform selection shares}
  \caption{\texttt{GPQA} strong-competition one pass: quality regret per step against generation cost per query for every routing rule (left), and the share of steps the platform assigned to each provider, exploration steps included (right), means with $95\%$ confidence intervals across runs. Monetary quantities are in units of $10^{-6}$ USD. In the right panel green marks $i^*$, blue the other qualified provider and orange the unqualified providers.}
  \label{fig:strong-one-pass}
\end{figure}
\clearpage
\subsection{Long-horizon learning and identification}\label{app:strong-long}

\begin{figure}[!htbp]
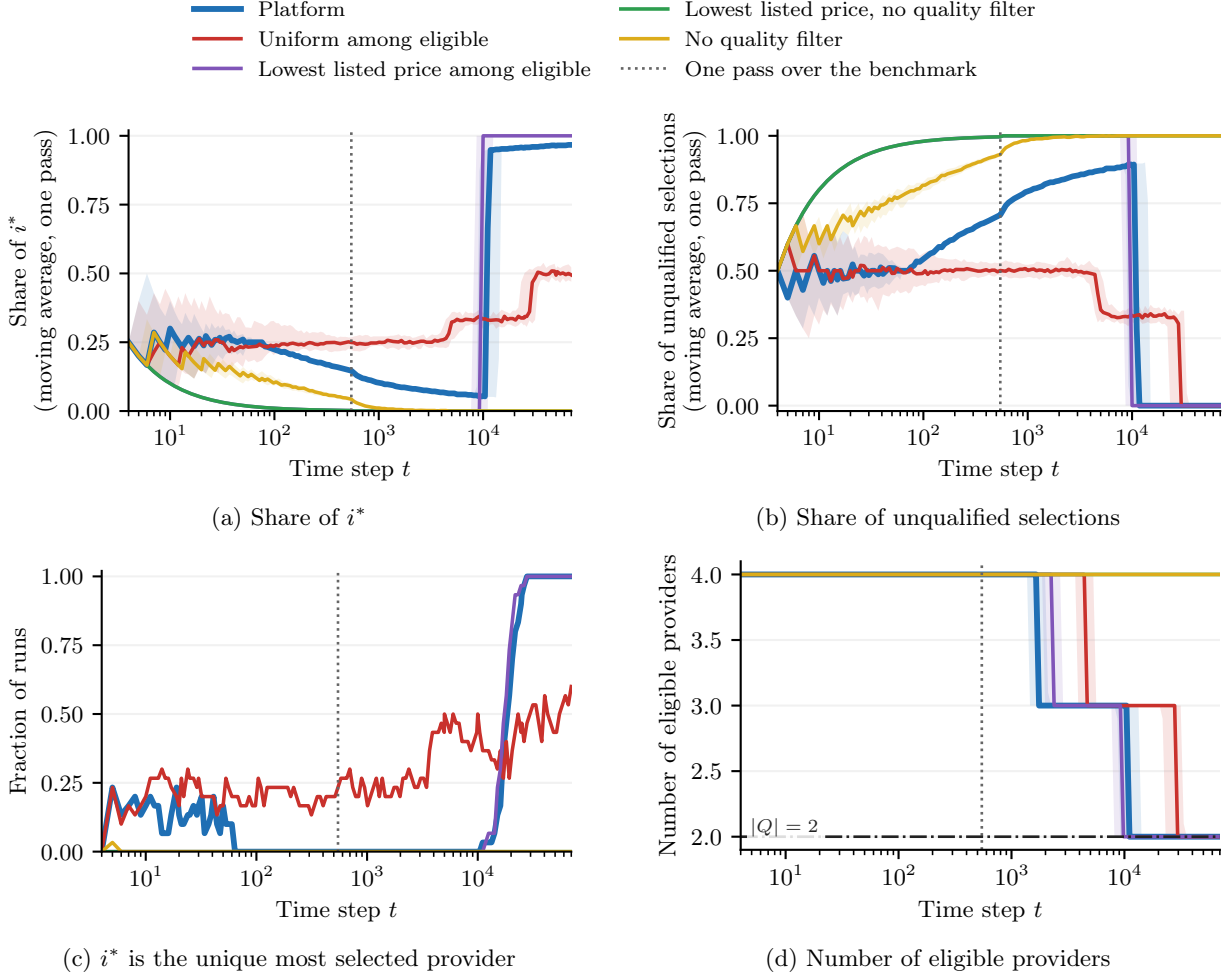

  \centering
  \sharedlegendof{strong}{policies_lines}\\[4pt]
  \strongpair{long_horizon_rolling_i_star_share}{Share of $i^*$}{long_horizon_unqualified_share}{Share of unqualified selections}\\[2pt]
  \strongpair{long_horizon_i_star_identification_probability}{$i^*$ is the unique most selected provider}{long_horizon_eligible_providers}{Number of eligible providers}
  \caption{\texttt{GPQA} strong-competition long horizon: identification of $i^*$ under every routing rule. Shares are moving averages with window size $K$, exploration steps included, median over $30$ independent runs with a $10$--$90\%$ band; the leader panel is the fraction of runs in which $i^*$ has strictly more selections than every other provider; the last panel is the number of providers not yet ruled out, with the dash-dotted line at $|\Qcal|$. The vertical line marks one pass over the benchmark. The two listed-price rules coincide until the filter rules out the two cheap providers.}
  \label{fig:strong-identification}
\end{figure}

\begin{figure}[!htbp]
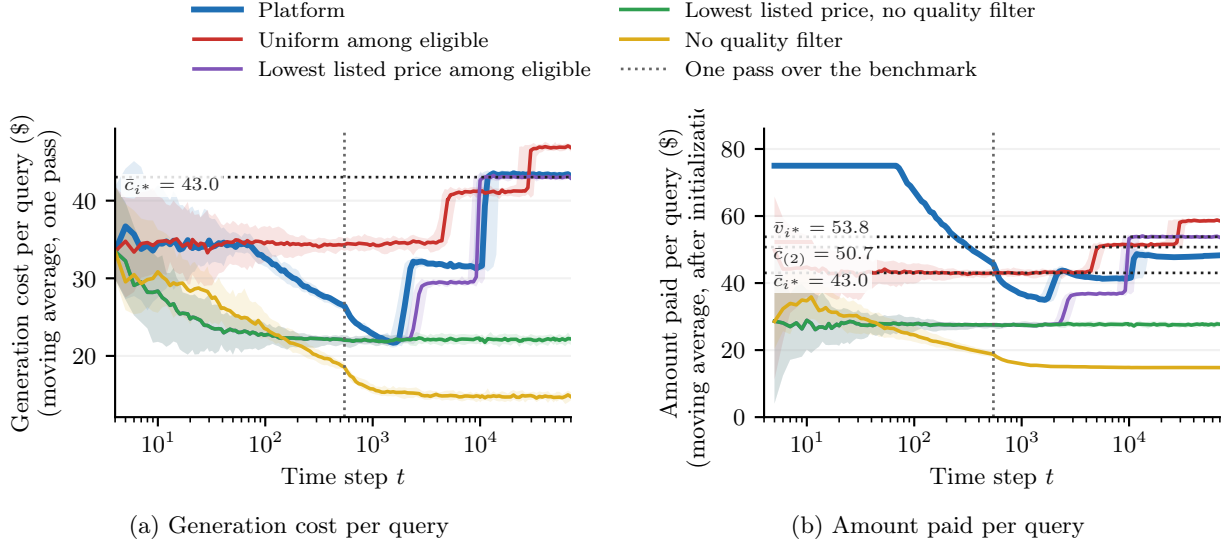

  \centering
  \sharedlegendof{strong}{policies_lines}\\[4pt]
  \strongpair{long_horizon_generation_cost_per_query}{Generation cost per query}{long_horizon_amount_paid_per_query}{Amount paid per query}
  \caption{\texttt{GPQA} strong-competition long horizon: generation cost per query as a moving average with window size $K$ (left) and amount paid per query as a moving average with window size $K$ over the steps after initialization (right), under every rule, median over $30$ independent runs with a $10$--$90\%$ band. Monetary quantities are in units of $10^{-6}$ USD. The reference lines mark $\bar c_{i^*}$, $\bar c_{(2)}$ and the mean query price $\bar{v}_{i^*}$ of $i^*$. The platform's payment $\pi_t$ includes its exploration steps at $c_{\max}$; the vertical axis of the right panel is cut.}
  \label{fig:strong-routing}
\end{figure}

\begin{table}[!htbp]
  \centering\scriptsize
  \setlength{\tabcolsep}{3pt}
  \caption{\texttt{GPQA} strong-competition long horizon: outcomes at $T = 70{,}000$, means over $30$ independent runs; columns as in Table~\ref{tab:gpqa-long}.}
  \label{tab:strong-long}
  \begin{tabular}{Yrrrrrrrrr}
    \toprule
    Policy & \shortstack{Share\\of $i^*$} & \shortstack{Moving\\avg.} & \shortstack{Unqualified\\share} & $R_t^{\mathrm{qual}}/t$ & $R_t^{\mathrm{gen}}/t$ & \shortstack{Generation\\cost} & \shortstack{Amount\\paid} & \shortstack{$i^*$ most\\selected} & $|\hat{\Qcal}_t|$ \\
    \midrule
    \multicolumn{10}{l}{\textit{Strong-competition roster}} \\
    Platform & 0.821 & 0.967 & 0.135 & 0.0060 & 0.335 & 41.3 & 47.0 & 1.00 & 2.0 \\
    Uniform among eligible & 0.428 & 0.498 & 0.144 & 0.0063 & 3.291 & 44.1 & 55.1 & 0.60 & 2.0 \\
    Lowest listed price among eligible & 0.865 & 1.000 & 0.135 & 0.0060 & 0.000 & 41.0 & 51.2 & 1.00 & 2.0 \\
    Lowest listed price, no quality filter & 0.000 & 0.000 & 1.000 & 0.0657 & 0.000 & 22.1 & 27.6 & 0.00 & 4.0 \\
    No quality filter & 0.001 & 0.000 & 0.999 & 0.0939 & 0.003 & 14.8 & 14.8 & 0.00 & 4.0 \\
    \bottomrule
  \end{tabular}
\end{table}

\begin{table}[!htbp]
  \centering\small
  \caption{\texttt{GPQA} strong-competition exploration steps of the platform, means over $30$ independent runs, over one pass and over the long run; rows as in Table~\ref{tab:gpqa-exploration}.}
  \label{tab:strong-exploration}
  \begin{tabular}{lrr}
    \toprule
    Quantity & \multicolumn{2}{c}{Strong-competition roster} \\
     & one pass & long run \\
    \midrule
    Exploration steps (mean) & 253 & 4,005 \\
    Share of exploration steps after initialization & 0.467 & 0.057 \\
    Share of exploration steps at $T$ (moving avg.) & -- & 0.033 \\
    Exploration-step payment per step & 34.76 & 2.47 \\
    Exploration-step generation cost per step & 18.30 & 1.70 \\
    Share of exploration steps given to $i^*$ & 0.312 & 0.188 \\
    Share of exploration steps given to unqualified providers & 0.375 & 0.051 \\
    Payment $\pi_t$ at $T$, non-exploration steps (moving avg.) & -- & 47.37 \\
    Payment $\pi_t$ at $T$, all steps (moving avg.) & -- & 48.28 \\
    \bottomrule
  \end{tabular}
\end{table}
\clearpage
\subsection{Regret and theoretical bounds}\label{app:strong-bounds}

\begin{figure}[!htbp]
  \centering
  \sharedlegendof{strong}{policies_lines}\\[4pt]
  \strongpair{long_horizon_quality_regret_per_round}{Average quality regret}{long_horizon_generation_regret_per_round}{Average generation regret}
  \caption{\texttt{GPQA} strong-competition long horizon: average quality regret $R_t^{\mathrm{qual}}/t$ and average generation regret $R_t^{\mathrm{gen}}/t$ of every routing rule, median over $30$ independent runs with a $10$--$90\%$ band, on logarithmic axes. Monetary quantities are in units of $10^{-6}$ USD. The platform's regrets include its exploration steps.}
  \label{fig:strong-regret}
\end{figure}

\begin{figure}[!htbp]
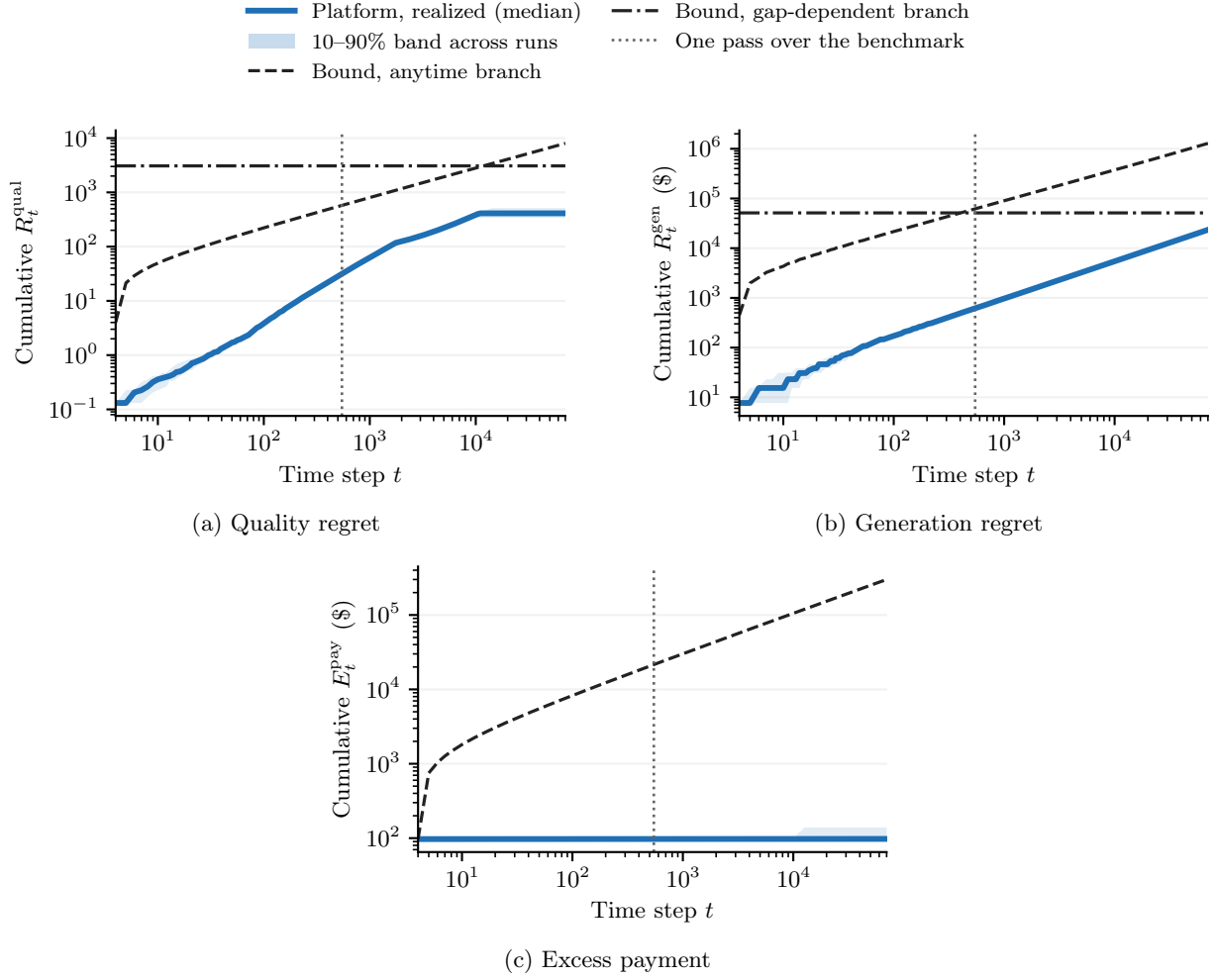

  \centering
  \sharedlegendof{strong}{bounds}\\[4pt]
  \strongpair{long_horizon_quality_regret_vs_bound}{Quality regret}{long_horizon_generation_regret_vs_bound}{Generation regret}\\[2pt]
  \strongpanel{long_horizon_excess_payment_vs_bound}{Excess payment}
  \caption{\texttt{GPQA} strong-competition long horizon: the platform's cumulative quality regret, generation regret and excess payment against their bounds, evaluated at every time step. Monetary quantities are in units of $10^{-6}$ USD. The bounds are those of Theorems~\ref{thm:quality-guarantee} and~\ref{thm:generation-regret} and of \eqref{eq:excess-payment-general} without its exploration term, as in Figure~\ref{fig:gpqa-bounds}. The regrets include the exploration steps; the excess payment excludes them.}
  \label{fig:strong-bounds}
\end{figure}

\begin{table}[!htbp]
  \centering\small
  \caption{\texttt{GPQA} strong-competition theoretical diagnostics of the platform over the long run ($T = 70{,}000$, $30$ independent runs); rows as in Table~\ref{tab:gpqa-diagnostics}.}
  \label{tab:strong-diagnostics}
  \begin{tabular}{lr}
    \toprule
    Quantity & Strong-competition roster \\
    \midrule
    $B_{\mathrm{id}}(T)$ & 76,846 \\
    $T_{\mathrm{id}}$ & 153,697 \\
    Horizon $T$ of the long run & 70,000 \\
    Selections of providers other than $i^*$ after initialization (mean) & 12,508 \\
    \quad ratio to $B_{\mathrm{id}}(T)$ & 0.163 \\
    Runs in which $i^*$ is the most selected provider at $t = 5{,}000$ & 0.00 \\
    \quad at $t = 10{,}000$ & 0.00 \\
    \quad at $t = T_{\mathrm{id}}$ & -- \\
    \quad at $t = T$ & 1.00 \\
    $R_T^{\mathrm{qual}}$, realized (mean) & 420 \\
    \quad bound \eqref{eq:quality-regret-bound} & 3,076 \\
    \quad ratio & 0.136 \\
    $R_T^{\mathrm{gen}}$, realized (mean) & 23,433 \\
    \quad bound \eqref{eq:generation-regret-bound} & 51,268 \\
    \quad ratio & 0.457 \\
    $E_T^{\mathrm{pay}}$ without the exploration steps, realized (mean) & 111 \\
    \quad bound \eqref{eq:excess-payment-general} without its exploration term & 300,158 \\
    \quad ratio & 0.000 \\
    $E_T^{\mathrm{pay}}$ incl. exploration steps, realized (mean) & 97,343 \\
    \quad bound \eqref{eq:excess-payment-general} & 595,681 \\
    \quad ratio & 0.163 \\
    Good event held, one pass (runs) & 30 of 30 \\
    Good event held, long run, platform & 30 of 30 \\
    Good event held, long run, no quality filter & 30 of 30 \\
    Payment-bound violations, one pass (steps) & 0 \\
    \bottomrule
  \end{tabular}
\end{table}

\begin{table}[!htbp]
  \centering\footnotesize
  \setlength{\tabcolsep}{4pt}
  \caption{\texttt{GPQA} strong-competition long horizon: selections of every provider other than $i^*$ by the platform after initialization, over $30$ independent runs of $T = 70{,}000$ queries, against its cap; columns as in Table~\ref{tab:gpqa-selections}.}
  \label{tab:strong-selections}
  \begin{tabular}{llrrrrlr}
    \toprule
    Provider & Gap & \multicolumn{2}{c}{Selections} & Exploration & Last selection & Bound & Ratio \\
    & & mean & max & (mean) & (mean step) & & \\
    \midrule
    \multicolumn{8}{l}{\textit{Strong-competition roster}} \\
    Qwen2-0.5B & $\varepsilon_i = 0.094$ & 1,172 & 1,585 & 16 & 1,746 & $M = 8{,}074$ & 0.145 \\
    Llama-3.2-1B & $\varepsilon_i = 0.037$ & 8,293 & 12,542 & 190 & 10,980 & $M = 62{,}113$ & 0.134 \\
    Qwen2.5-7B & $\Delta_i = 7.7$ & 3,043 & 3,043 & 3,043 & 69,999 & $L = 6{,}659$ & 0.457 \\
    All providers other than $i^*$ & & 12,508 & & 3,249 &  & $B_{\mathrm{id}}(T) = 76{,}846$ & 0.163 \\
    \bottomrule
  \end{tabular}
\end{table}
\clearpage
\subsection{Payments and provider outcomes}\label{app:strong-payments}

\begin{figure}[!htbp]
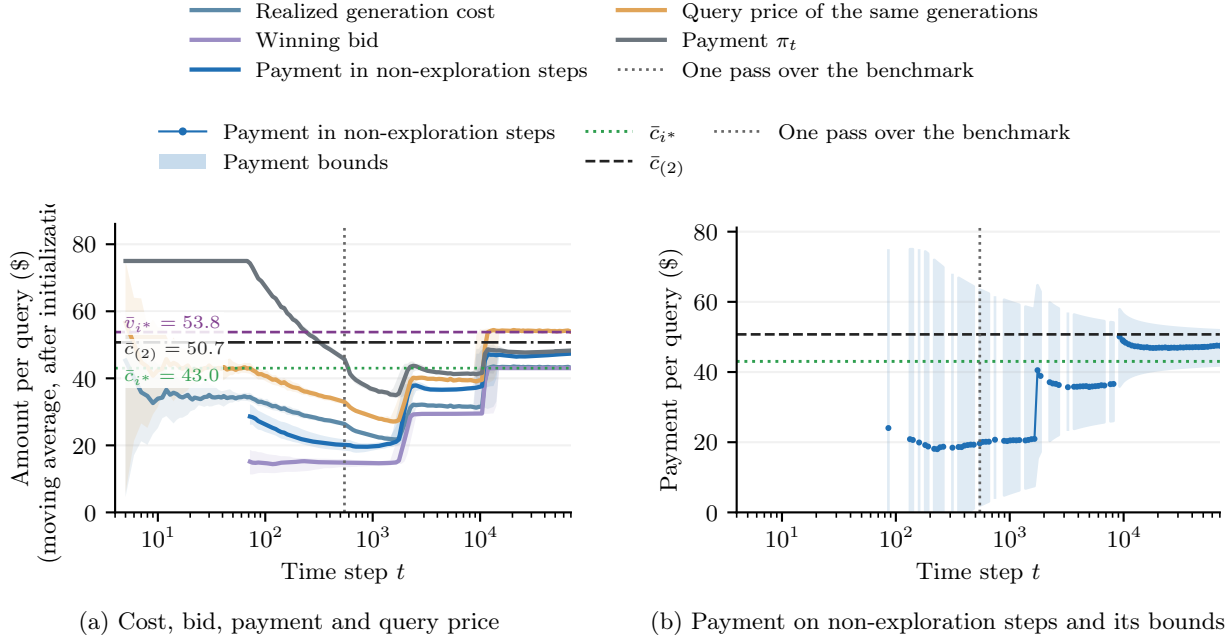

  \centering
  \sharedlegendof{strong}{economic}\\[4pt]
  \sharedlegendof{strong}{payment_bounds}\\[4pt]
  \strongpair{long_horizon_mechanism_economic_quantities}{Cost, bid, payment and query price}{payment_bounds_long_horizon}{Payment on non-exploration steps and its bounds}
  \caption{\texttt{GPQA} strong-competition long horizon. Left: the platform's realized generation cost, the query price of the same generations and the payment $\pi_t$, as a moving average with window size $K$ over the steps after initialization, and its winning bid and payment on non-exploration steps in the same window, median over $30$ independent runs with a $10$--$90\%$ band; the dotted, dash-dotted and dashed lines mark $\bar c_{i^*}$, $\bar c_{(2)}$ and the mean query price $\bar{v}_{i^*}$ of $i^*$, and the vertical axis is cut. Right: the payment $\pi_t$ of the first run on non-exploration steps, sampled at logarithmically spaced time steps, inside the payment bounds of Appendix~\ref{app:experimental-details} for the provider selected in that step; sampled steps that were exploration steps are left blank. Monetary quantities are in units of $10^{-6}$ USD.}
  \label{fig:strong-economic}
\end{figure}

\begin{table}[!htbp]
  \centering\small
  \caption{\texttt{GPQA} strong-competition one pass: the platform's economic quantities, mean $\pm$ $95\%$ half-width over $30$ independent runs; rows as in Table~\ref{tab:gpqa-payments}.}
  \label{tab:strong-payments}
  \begin{tabular}{lr}
    \toprule
    Quantity & Strong-competition roster \\
    \midrule
    Generation cost per query, whole pass & $26.36 \pm 0.13$ \\
    Query price per query, whole pass & $32.94 \pm 0.16$ \\
    Payment $\pi_t$ per query, whole pass & $45.96 \pm 0.27$ \\
    Generation cost per query, after initialization & $26.29 \pm 0.13$ \\
    Query price per query, after initialization & $32.87 \pm 0.16$ \\
    Payment $\pi_t$ per query, after initialization & $45.75 \pm 0.27$ \\
    Winning bid, non-exploration steps & $14.66 \pm 0.32$ \\
    Payment $\pi_t$ per query, non-exploration steps & $20.13 \pm 0.51$ \\
    Post-initialization exploration payments per query, whole pass & $34.76 \pm 0.00$ \\
    Share of exploration steps after initialization & $0.467 \pm 0.000$ \\
    Excess payment per step $E_T^{\mathrm{pay}}/T$, whole pass & $11.428 \pm 0.000$ \\
    Excess payment per query, non-exploration steps & $0.000 \pm 0.000$ \\
    Mean provider payoff over the pass & $2675.9 \pm 35.2$ \\
    Lowest provider payoff over providers and runs & 1726.3 \\
    \bottomrule
  \end{tabular}
\end{table}

\begin{table}[!htbp]
  \centering\footnotesize
  \caption{\texttt{GPQA} strong-competition one pass: each provider's selections and payoff under the platform over $30$ independent runs, exploration steps included, with providers sorted by mean generation cost. The last column counts the runs in which the provider's payoff over the pass was negative.}
  \label{tab:strong-payoffs}
  \begin{tabular}{lrrrrr}
    \toprule
    Provider & Selections (mean) & Payoff (mean) & Payoff (min) & Payoff (max) & Negative runs \\
    \midrule
    \multicolumn{6}{l}{\textit{Strong-competition roster}} \\
    Qwen2-0.5B & 306.0 & 2581.9 & 1726.3 & 3657.6 & 0 \\
    Llama-3.2-1B & 80.0 & 3637.5 & 3389.7 & 3815.3 & 0 \\
    Llama-3.1-8B ($i^*$) & 80.0 & 2564.5 & 2386.3 & 2730.8 & 0 \\
    Qwen2.5-7B & 80.0 & 1919.8 & 1749.9 & 2130.2 & 0 \\
    \bottomrule
  \end{tabular}
\end{table}

\clearpage

\section{Empirical Robustness to Bidding Deviations} \label{app:robustness-exp}
\newcommand{\robustresults}{analysis/bounded_deviations/gpqa_strong_grid}

We complement the theoretical analysis in Appendix~\ref{app:robustness} with experiments on the \texttt{GPQA} strong-competition environment of Appendix~\ref{app:gpqa-strong}, using the same $T=70{,}000$ queries, $30$ runs, empirical cost estimates, and exploration parameters $k=2$ and $\alpha=3/4$. The quality threshold, $c_{\max}$, $\delta=0.05$, $\gamma=0$, and cost construction are unchanged. The qualified providers are \texttt{Llama-3.1-8B} ($i^*$) and \texttt{Qwen2.5-7B}, with mean costs $43.031$ and $50.729$, respectively, so $\Delta=7.698$. All monetary quantities below are in units of $10^{-6}$ USD.

We replace cost-estimate bidding by $b_{i,t}=\min\{c_{\max},\max\{0,\hat c_{i,t}+s_i d\}\}$, with $s_i=1$ for all providers (all overbid), $s_i=-1$ for all providers (all underbid), or $s_{i^*}=1$ and $s_j=-1$ for $j\ne i^*$ (mixed). We use $d/\Delta\in\{0,0.125,0.25,0.375,0.5,0.625,0.75,1\}$. Each profile and deviation level uses the same $30$ seeds, with a shared cost-estimate bidding baseline at $d=0$.

\begin{figure}[!htbp]
  \centering
  \includegraphics[width=\linewidth]{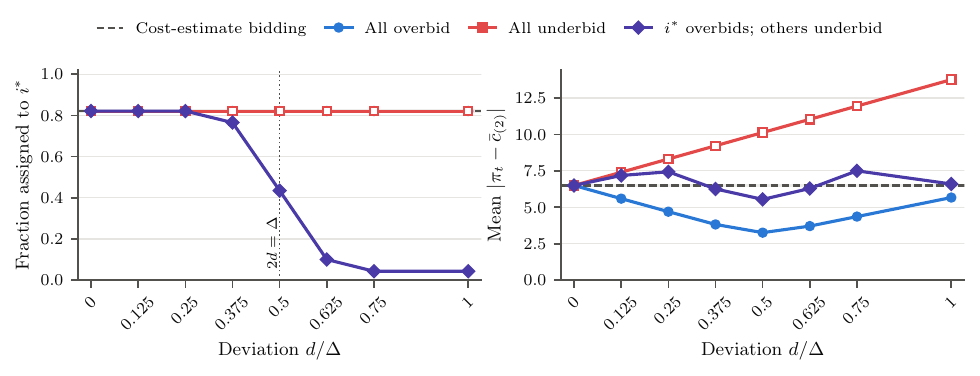}
    \caption{Bounded bidding deviations on \texttt{GPQA}, at $T=70{,}000$: fraction of queries assigned to $i^*$ (left) and average absolute payment error $T^{-1}\sum_{t=1}^T|\pi_t-\bar c_{(2)}|$ (right). Both metrics include initialization and exploration steps. Points are means over $30$ paired runs with $95\%$ confidence intervals. The vertical line marks $2d=\Delta$.}
  \label{fig:empirical-robustness-grid}
\end{figure}

\begin{table}[!htbp]
    \centering\footnotesize
    \setlength{\tabcolsep}{4pt}
    \caption{Outcomes in Figure~\ref{fig:empirical-robustness-grid}, means over $30$ runs. Payment error is the average absolute deviation from $\bar c_{(2)}$. The percentage of queries assigned to unqualified providers is identical across profiles and levels within each seed, averaging $13.52\%$ (a $95\%$ confidence interval of $\pm 0.92$ percentage points).}
    \label{tab:empirical-robustness-grid}
    \begin{tabular}{rcccccc}
        \toprule
        & \multicolumn{2}{c}{All overbid} & \multicolumn{2}{c}{All underbid} & \multicolumn{2}{c}{Mixed} \\
        $d/\Delta$ & Share of $i^*$ & Payment error & Share of $i^*$ & Payment error & Share of $i^*$ & Payment error \\
        \midrule
        0 & 0.821 & 6.50 & 0.821 & 6.50 & 0.821 & 6.50 \\
        0.125 & 0.821 & 5.59 & 0.821 & 7.40 & 0.821 & 7.18 \\
        0.25 & 0.821 & 4.70 & 0.821 & 8.31 & 0.821 & 7.43 \\
        0.375 & 0.821 & 3.82 & 0.821 & 9.22 & 0.765 & 6.26 \\
        0.5 & 0.821 & 3.26 & 0.821 & 10.12 & 0.435 & 5.54 \\
        0.625 & 0.821 & 3.71 & 0.821 & 11.03 & 0.100 & 6.29 \\
        0.75 & 0.821 & 4.36 & 0.821 & 11.94 & 0.043 & 7.50 \\
        1 & 0.821 & 5.67 & 0.821 & 13.75 & 0.043 & 6.60 \\
        \bottomrule
    \end{tabular}
\end{table}

% Figure~\ref{fig:empirical-robustness-grid} shows that mixed deviations shift allocation away from $i^*$. At $d=0.75\Delta$, \texttt{Qwen2.5-7B} receives $82.13\%$ of all queries on average, while $i^*$ receives $4.35\%$. The replacement provider is qualified but has a $17.89\%$ higher mean cost. These results are consistent with Appendix~\ref{app:robustness}: the quality bounds remain unchanged under bounded deviations, even when identification of $i^*$ is no longer guaranteed.

\begin{figure}[!htbp]
    \centering
    \includegraphics[width=\linewidth]{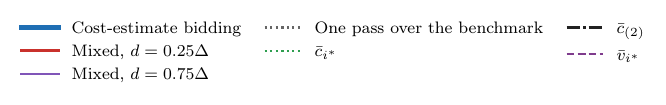}\\[3pt]
    \begin{subfigure}[t]{0.48\textwidth}\centering
        \includegraphics[width=\linewidth]{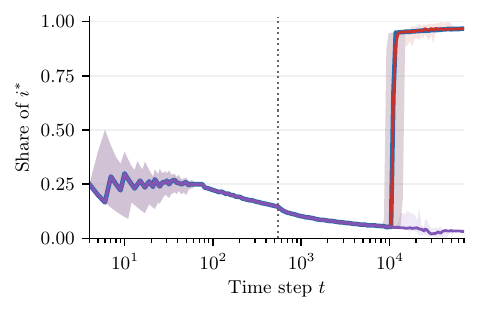}
        \caption{Share of $i^*$}
    \end{subfigure}\hfill
    \begin{subfigure}[t]{0.48\textwidth}\centering
        \includegraphics[width=\linewidth]{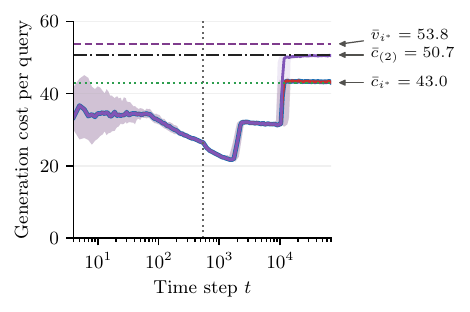}
        \caption{Generation cost per query}
    \end{subfigure}\\[2pt]
    \begin{subfigure}[t]{0.48\textwidth}\centering
        \includegraphics[width=\linewidth]{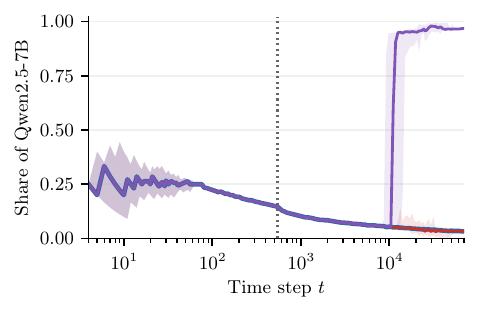}
        \caption{Share of \texttt{Qwen2.5-7B}}
    \end{subfigure}\hfill
    \begin{subfigure}[t]{0.48\textwidth}\centering
        \includegraphics[width=\linewidth]{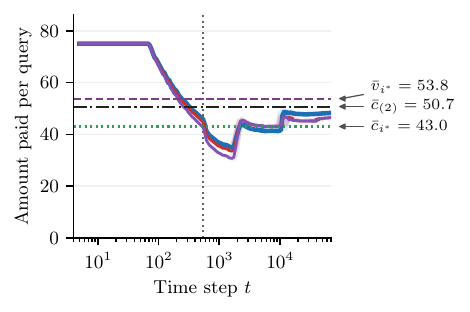}
        \caption{Amount paid per query}
    \end{subfigure}
    \caption{Curves show medians over $30$ runs of moving averages with window $K=546$; bands show the $10$th--$90$th percentiles. The time axis is logarithmic. Selection shares and generation costs include all steps; payments exclude initialization and include exploration payments at $c_{\max}$. Before the window is full, averages use all available observations. The vertical line marks one benchmark pass.}
\end{figure}

\clearpage

\section{Sensitivity to Provider-Specific Margins}\label{app:margin-sensitivity}
\newcommand{\marginresults}{analysis/margin_sensitivity/gpqa}

\paragraph{Setup.} We test whether the pricing pattern in Table~\ref{tab:main-environments} depends on assigning every provider the same per-token margin. We retain the four-provider \texttt{GPQA} roster and recorded generations, but replace Eq.~\eqref{eq:experimental-generation-cost} with $c_{i,t}=V_{i,t}/(1+\mu_i)$, where $\mu_i$ is provider $i$'s margin relative to cost, unknown to the platform. For $q_{\min}$, we use the original threshold, which qualifies the two larger models, and the lower threshold, the midpoint between the qualities of the two smaller models, which additionally qualifies \texttt{Llama-3.2-1B}. The common-margin baselines exhibit strong and weak competition, respectively. For each configuration, we recompute the cheapest qualified provider $i^*$ and the cost ordering.

\begin{table}[!htbp]
    \centering\footnotesize
    \setlength{\tabcolsep}{4pt}
    \caption{Provider-specific per-token margins $\mu_i$; listed query prices remain fixed.}
    \label{tab:margin-configurations}
    \begin{tabular}{lcccc}
        \toprule
        Configuration & Llama-3.1-8B & Qwen2.5-7B & Llama-3.2-1B & Qwen2-0.5B \\
        \midrule
        Common & 25\% & 25\% & 25\% & 25\% \\
        Mild & 20\% & 30\% & 30\% & 20\% \\
        Wide A & 40\% & 10\% & 10\% & 40\% \\
        Wide B & 10\% & 40\% & 40\% & 10\% \\
        \bottomrule
    \end{tabular}
\end{table}

We run Algorithm~\ref{alg:platform} with estimate-cost bidding, using the same remaining parameters, number of runs, and averaging window as in Appendix~\ref{app:gpqa-strong}. Query orders, generation draws, and random seeds are paired across configurations and thresholds. Following Table~\ref{tab:main-environments}, $\pi^{\rm ord}$ denotes the mean payment on non-exploration steps within the final averaging window of each run. The pricing margin and saving are $100(\pi^{\rm ord}/\bar c_{i^*}-1)$ and $100(1-\pi^{\rm ord}/\bar v_{i^*})$, respectively. 

\begin{figure}[!htbp]
\centering
    \includegraphics[width=\linewidth]{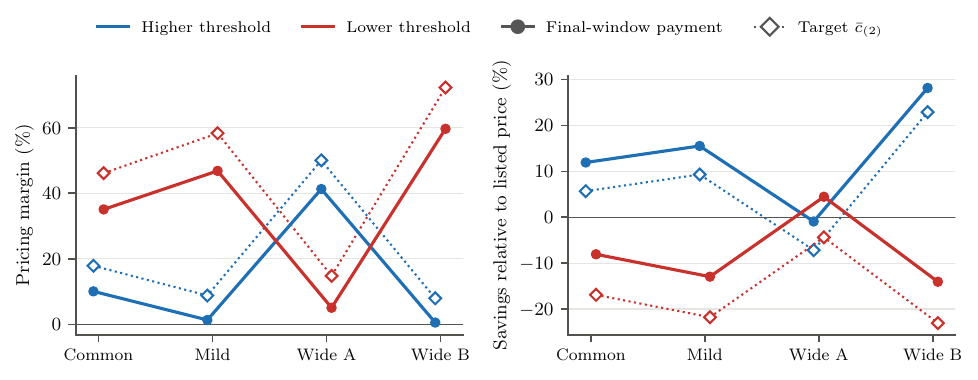}
    \caption{\texttt{GPQA} sensitivity to provider-specific margins for the configurations in Table~\ref{tab:margin-configurations}. Pricing margin (left) and savings relative to the listed price of $i^*$ (right). Filled circles use the final-window non-exploration payment $\pi^{\rm ord}$. Diamonds use $\bar c_{(2)}$ in place of $\pi^{\rm ord}$ as the asymptotic reference. Colors identify the fixed quality thresholds; the competition regime at each threshold can change with the margins.}
    \label{fig:margin-sensitivity}
\end{figure}

\begin{table}[!htbp]
    \centering\footnotesize
    \setlength{\tabcolsep}{4pt}
    \caption{\texttt{GPQA} margin sensitivity: means over $30$ runs, with $95\%$ confidence intervals for payments. Monetary quantities are in units of $10^{-6}$ USD. L3.1, L3.2, and Q2.5 denote \texttt{Llama-3.1-8B}, \texttt{Llama-3.2-1B}, and \texttt{Qwen2.5-7B}, respectively. Negative savings indicate payments above the listed-price reference.}
    \label{tab:margin-sensitivity}
    \begin{tabular}{llcccccc}
        \toprule
        Configuration & $i^*$ & $\bar c_{i^*}$ & $\bar c_{(2)}$ & $\bar v_{i^*}$ & $\pi^{\rm ord}$ & Margin & Saving \\
        \midrule
        \multicolumn{8}{l}{\textit{Higher threshold (strong baseline)}} \\
        Common & L3.1 & 43.03 & 50.73 & 53.79 & $47.37\pm0.07$ & 10.1\% & 11.9\% \\
        Mild & L3.1 & 44.82 & 48.78 & 53.79 & $45.44\pm0.07$ & 1.4\% & 15.5\% \\
        Wide A & L3.1 & 38.42 & 57.65 & 53.79 & $54.28\pm0.07$ & 41.3\% & -0.9\% \\
        Wide B & Q2.5 & 45.29 & 48.90 & 63.41 & $45.55\pm0.07$ & 0.6\% & 28.2\% \\
        \multicolumn{8}{l}{\textit{Lower threshold (weak baseline)}} \\
        Common & L3.2 & 29.45 & 43.03 & 36.81 & $39.78\pm0.07$ & 35.1\% & -8.1\% \\
        Mild & L3.2 & 28.32 & 44.82 & 36.81 & $41.57\pm0.07$ & 46.8\% & -12.9\% \\
        Wide A & L3.2 & 33.46 & 38.42 & 36.81 & $35.17\pm0.06$ & 5.1\% & 4.5\% \\
        Wide B & L3.2 & 26.29 & 45.29 & 36.81 & $41.98\pm0.06$ & 59.7\% & -14.0\% \\
        \bottomrule
    \end{tabular}
\end{table}

Figure~\ref{fig:margin-sensitivity} and Table~\ref{tab:margin-sensitivity} show that differences in margins preserve the pricing pattern observed when all providers have the same margin. Wider differences can change the cheapest qualified provider and the competition regime, altering both the magnitude and sign of savings.
Only in the higher-threshold Wide B configuration does another qualified provider offer a lower mean listed price per query than $i^*$.
There, the final-window payment on non-exploration steps is $28.2\%$ below the listed mean query price of $i^*$ (\texttt{Qwen2.5-7B}) and $15.3\%$ below that of the cheapest qualified alternative (\texttt{Llama-3.1-8B}).
Under provider-specific margins, Theorem~\ref{thm:regret} still guarantees that the average payment converges with high probability to the second-lowest mean cost among qualified providers.
%These results support qualitatively similar conclusions regarding market inefficiency: competition among qualified providers can generate savings relative to listed prices even when margins differ across providers.

\FloatBarrier

\end{document}